\documentclass[lettersize,journal]{IEEEtran}

\usepackage{amssymb}
\usepackage{xcolor}
\usepackage{amsmath}
\usepackage{graphicx}
\usepackage{algorithm}
\usepackage{algorithmic}
\newtheorem{definition}{Definition}
\newtheorem{theorem}{Theorem}

\newtheorem{proposition}{Proposition}

\usepackage{array}
\usepackage{booktabs}
\usepackage{multirow}
\usepackage{hyperref}
\usepackage{xr}
\usepackage[justification=centering]{caption}

\numberwithin{equation}{subsection}
\numberwithin{definition}{subsection}
\numberwithin{theorem}{subsection}
\numberwithin{lemma}{subsection}
\numberwithin{corollary}{subsection}
\numberwithin{proposition}{subsection}
\numberwithin{remark}{subsection}
\numberwithin{assumption}{subsection}
\allowdisplaybreaks[4]

\begin{document}

\title{Compact Multi-level-prior Tensor Representation for Hyperspectral Image Super-resolution}

\author{%
	Yinjian Wang\textsuperscript, Wei Li,~\IEEEmembership{Senior Member,~IEEE}, Yuanyuan Gui, Gemine Vivone,~\IEEEmembership{Senior Member,~IEEE}
    \thanks{%
        This paper was supported by NSFC Projects of International Cooperation and Exchanges [Grant number  W2411055]. (Corresponding author: Wei Li.)
    }
    \thanks{ Y. Wang, W. Li, and Y. Gui are with the School of Information and
    	Electronics, Beijing Institute of Technology, and
    	the National Key Laboratory of Science and Technology on Space-Born
    	Intelligent Information Processing, Beijing
    	100081, China(e-mail: yinjw@bit.edu.cn, liwei089@ieee.org, 953647315@qq.com).}
    \thanks{G. Vivone is with the National Research Council, Institute of Methodologies for Environmental Analysis (CNR-IMAA), 85050 Tito, Italy(e-mail: gemine.vivone@imaa.cnr.it).}
}
\maketitle

\begin{abstract}
	Fusing a hyperspectral image with a multispectral image acquired over the same scene, \textit{i.e.}, hyperspectral image super-resolution, has become a popular computational way to access the latent high-spatial-spectral-resolution image. To date, a variety of fusion methods have been proposed, among which the tensor-based ones have testified that multiple priors, such as multidimensional low-rankness and spatial total variation at multiple levels, effectively drive the fusion process. However, existing tensor-based models can only effectively leverage one or two priors at one or two levels, since simultaneously incorporating multi-level priors inevitably increases model complexity. This introduces challenges in both balancing the weights of different priors and optimizing multi-block structures. Concerning this, we present a novel hyperspectral super-resolution model compactly characterizing these multi-level priors of hyperspectral images within the tensor framework. Firstly, the proposed model decouples the spectral low-rankness and spatial priors by casting the latent high-spatial-spectral-resolution image into spectral subspace and spatial maps via block term decomposition. Secondly, these spatial maps are stacked as the spatial tensor encoding the high-order spatial low-rankness and smoothness priors, which are co-modeled via the proposed non-convex mode-shuffled tensor correlated total variation. Finally, we draw inspiration from the linearized alternating direction method of multipliers to design an efficient algorithm to optimize the resulting model, theoretically proving its Karush-Kuhn-Tucker convergence under mild conditions. Experiments on multiple datasets demonstrate the effectiveness of the proposed algorithm. The code implementation will be available from \url{https://github.com/WongYinJ}.
\end{abstract}



\section{Introduction}
\label{sec: Intro}
Hyperspectral image (HSI) is much appreciated for its nanometer-level fine-grained spectral information which to a large extent facilitates many scene interpretation tasks such as classification {\cite{9954181,9457035}}, object tracking \cite{10128966}, and target detection {\cite{10549817,9478333}}. Unfortunately, acquiring high-spatial-resolution HSI at the hardware level remains a significant challenge due to the fundamental physical trade-offs in existing imaging systems. To be more specific, the desire for adequate photon capture to maintain photoelectric signal-to-noise-ratio necessitates larger detector units per spectral band in hyperspectral sensors, inherently reducing the spatial resolution of the acquired HSI. Such defect severely impairs the delineation of fine-scale features in the application scenarios above \cite{VIVONE2023405}. Consequently, fusing HSI with a corresponding multi-spectral image (MSI) that is easily captured at a much higher spatial resolution due to a much sparser spectral sampling rate \cite{7946218}, \textit{i.e.}, hyperspectral super-resolution (HSR), has become a popular software-level alternative to access to the latent high-spatial-spectral-resolution image (HSSI). 

In this vein, numerous HSR models have been subsequently proposed. Song \textit{et al.} \cite{SONG2014148} developed a dictionary-based fusion method using non-negative matrix factorization to couple spectral and spatial information from HSI and MSI, respectively, achieving higher-fidelity results than benchmarks on satellite data. A convex alternating direction method of multipliers (ADMM)-based fusion method that leverages subspace modeling and edge-preserving vector total variation (TV) regularization to achieve HSR has been presented in \cite{7000523}, which jointly estimates the spatial and spectral responses that degrade the latent HSSI. Wei \textit{et al.} \cite{7163298} modeled HSR as the problem of solving a Sylvester equation maximizing the likelihood criterion. By leveraging the structural properties of the circulant and downsampling matrices inherent to the fusion problem, they derived a closed-form solution for the associated Sylvester equation, eliminating the need for iterative updates. Dian \textit{et al.} \cite{DIAN2019262} framed the fusion task as estimating spectral basis functions and their corresponding coefficients, and integrated non-local spatial similarities, spectral unmixing priors, and a sparsity constraint into the fusion problem. Xue \textit{et al.} \cite{9356457} proposed a structured sparse low-rank representation model that effectively exploits spatial/spectral subspace relationships from higher-level analysis, achieving superior HSR performance.

Despite their promising performance, all those methods fail to outrun the regime of matricizing the 3D HSI data cube. Since matrix analysis relies on the conventional matrix product, which only captures row-column relationships, it struggles to model the intricate structures of high-dimensional data like HSI. Consequently, tensor-based HSR approaches have gained increasing attention in HSI analysis \cite{10035509}. Exploiting a coupled Tucker decomposition \cite{kolda2009tensor} paradigm, Dian \textit{et al.} \cite{8359412} translated the spatial-spectral correlation of HSSI into the sparse prior on the core tensor and delivered a coupled sparse tensor factorization (CSTF) HSR model, which was later generalized for semi-blind HSR \cite{8917657}. Chang \textit{et al.} \cite{9076843} proposed a HSI restoration model that simultaneously captures spectral-spatial nonlocal similarity and spectral correlation through third-order tensors, demonstrating how singular value reweighting significantly enhances tensor modeling capability and flexibility. As research progresses, it is acknowledged that HSI exhibits multi-dimensional structures that cannot be sufficiently characterized by imposing a single constraint on any individual factor. In light of this, Dian \textit{et al.} \cite{8603806} exploited the low tensor-train (TT) \cite{oseledets2011tensor} rank (LTTR) prior, formulated in a multi-rank form, of the grouped 4D non-local patch tensors. They also proposed the generalized tensor nuclear norm {(GTNN)} \cite{10522984} to promote flexibility when characterizing the multi-dimensional low tensor tubal rank \cite{doi:10.1137/110837711}. To avoid the computational burden inherited from patch clustering, Liu \textit{et al.} \cite{9328229} imposed the tensor trace norm \cite{6138863} straightforwardly on the HSSI to cast the HSR task as a low Tucker rank tensor approximation (LRTA) problem. However, these methods formulate the multi-dimensional rank minimization problem as minimizing the nuclear norms of several matrices unfolded from the desired HSSI, incapable of explicitly analyzing the structures of the decomposed factors \cite{10154463}. To overcome this limitation, Li \textit{et al.} \cite{10149108} proposed the anisotropic sparsity (AS) tensor norm to transform the multi-dimensional low-rankness of the original tensor into the multi-dimensional sparsity of the Tucker core tensor. The resulting AS constrained low-rank tensor approximation (ASLA) framework achieves simultaneous multi-dimensional prior exploitation and factor structure pursuit under Tucker decomposition. Xu \textit{et al.} \cite{10770239} proposed the cascade-transform based tensor nuclear norm to alleviate the boundary effects and singleton domain limitation accompanied with the conventional tensor nuclear norm (TNN), which, noteworthily, was then generalized into a full-scale/multi-dimensional version, reducing its directional sensitivity. Another variant of TNN based on a non-linear transform, along with its full-scale/multi-dimensional version, was presented in \cite{10994802}, breaking TNN's limitation in modeling complex data structures inherent in the linear transform.  Additionally, TV \cite{RUDIN1992259}, as an efficient tool modeling the spatial continuity equipped within diverse image modalities, has been successfully absorbed into tensor-based HSR frameworks, \textit{e.g.}, \cite{10149108,9323227,10641088,9556548,9151315}. However, with multi-prior being employed, these models inevitably increase by scale, which leads to the challenges of model selection and multi-block optimization. Thus many of them suffer from hyper-weight balancing and the loss of convergence guarantee. Besides, most of these methods only manage to extract those prior information from a single level, that is, from the original data or its factorized components. The comprehensive exploitation of multi-prior at multi-level remains an open issue.

The advancement of computing devices has enabled the development of deep learning techniques. The HSR field has therefore witnessed a proliferation of data-driven methods. Nguyen \textit{et al.} \cite{9924190} proposed an unsupervised CNN-based fusion framework that incorporates Stein’s unbiased risk estimate (SURE) into a novel loss function, combining backprojection {mean square error (MSE)} with SURE to approximate the projected MSE between the fused and ground-truth HSSIs. Inspired by the smoothed particle hydrodynamics (SPH) theory, \cite{10302422} analogized pixel motion to SPH particle dynamics, implementing multi-scale smooth convolutions to preserve spectral information while using discretized Navier-Stokes equations to guide pixel motion for enhanced spatial edge clarity. Dian {et al.} \cite{10137388} delivered a zero-shot learning (ZSL) HSR network that estimates sensor responses for realistic data simulation, also incorporating dimension reduction and model-based loss, thus overcoming data scarcity/generalization issues while maintaining high accuracy and efficiency. Ma \textit{et al.} \cite{ma2023learning} replaced convolutional neural network (CNN) with transformers to learn HSSI priors, and unfolded a proximal gradient algorithm into a network where self-attention captures global spatial interactions and 3D-CNN layers enhance spatial-spectral correlation modeling. Nevertheless, the reliance on large-scale computational resources and high-quality training data remains a critical limitation hindering the practical application of data-driven methods.

In such context, this paper presents an as compact as possible tensor model for the HSR task, to concisely realize multi-level-prior representation, which further facilitates a convergence-guaranteed iterative algorithm. Specifically, the contributions are three-folds:

(1) We present the non-convex mode-shuffled tensor correlated total variation (NMS-t-CTV), which modifies t-CTV \cite{10078018} with a designed mode-shuffle strategy improving from convex approximation to non-convex one. By imposing it on the stacked spatial tensor acquired from the block-term decomposition (BTD) \cite{doi:10.1137/070690729}, we theoretically justify that the resulting model can cross-represent multi-level-priors equipped within the latent HSSI in a compact form.

(2) Inspired by the linearized alternating direction method of multipliers {(LADMM)} \cite{10.5555/2986459.2986528,yang2013linearized}, we customize an optimization framework for the proposed model that successfully deals with the high dimensionality of the spatial tensor. With a mild assumption that some multipliers are bounded, we manage to theoretically guarantee the convergence of the optimization framework to the Karush-Kuhn-Tucker (KKT) \cite{Karush1939,KuhnTucker1951,BoydVandenberghe2004} point, which also benefits from the compactness of the proposed model.

(3) A broad experimental analysis demonstrates the practical efficacy of the proposed compact multi-level-prior tensor representation (CMlpTR). 

The remainder of this article is organized as follows. Sec. \ref{sec:Preparations} familiarizes the readers with the notation and the foundational concepts related to this paper. Sec. \ref{sec:Methodology} presents the proposed CMlpTR and the derivation of the optimization framework. Sec. \ref{sec:Experiments} shows the experimental results. Finally, Sec. \ref{sec:Conclusion} draws the conclusions.
\section{Preparation}
\label{sec:Preparations}
\subsection{Notations}
Throughout this article, we denote a scalar as $a$, a vector as
$\boldsymbol{a}$, a matrix as $\boldsymbol{A}$, and a tensor as $\mathcal{A}$. $\left\Vert\boldsymbol{A}\right\Vert_F,\left\Vert\boldsymbol{A}\right\Vert,\left\Vert\boldsymbol{A}\right\Vert_1$, and $\left\Vert\boldsymbol{A}\right\Vert_\ast$ denote the Frobenius, the $l_1$, and the nuclear norms of a matrix, respectively. Particularly, we use $\boldsymbol{I}_N$ to denote the $N\times N$ identity matrix, $\boldsymbol{F}_N$ to denote the $N\times N$ discrete Fourier transformation matrix, and $\boldsymbol{D}_N$ to denote the $(N-1)\times N$ first-order difference matrix, \textit{i.e.}, 
$$\boldsymbol{D}_N=\begin{bmatrix}
	1&-1&&&\\
	&1&-1&&\\
	&&\ddots&\ddots&&\\
	&&&1&-1
\end{bmatrix}.$$
Let $\boldsymbol{A}^\ast$ be the conjugate transpose of $\boldsymbol{A}$. Then we can denote the anisotropic TV (ATV) norm of a matrix $\boldsymbol{A}\in\mathbb{R}^{M\times N}$ as $$\left\Vert\boldsymbol{A}\right\Vert_{ATV}\triangleq\left\Vert\boldsymbol{D}_M\boldsymbol{A}\right\Vert_1+\left\Vert\boldsymbol{A}\boldsymbol{D}_N^*\right\Vert_1.$$
For a $3$-way tensor $\mathcal{A}\in\mathbb{R}^{I_1\times I_2\times I_3}$, we use $\mathcal{A}[i_1,i_2,i_3]$ to denote the $(i_1,i_2,i_3)$th element of $\mathcal{A}$, and $\mathcal{A}_{:,:,i_3}$ to denote its $i_3$th frontal slice.
Its Frobenius norm is then denoted by $\left\Vert\mathcal{A}\right\Vert_F$ and defined as $\left\Vert\mathcal{A}\right\Vert_F=\sqrt{\sum_{i_1,i_2,i_3}{\cal A}^2[i_1,i_2,i_3]}$. Denoting the mode-$n$ product \cite{kolda2009tensor} with $\times_n$, the TV and ATV norms of $\mathcal{A}$ are defined as:
\begin{align*}
	\left\Vert\mathcal{A}\right\Vert_{TV}&\triangleq\sqrt{\left\Vert\mathcal{A}\times_1\boldsymbol{D}_M\right\Vert_F^2+\left\Vert\mathcal{A}\times_2\boldsymbol{D}_N\right\Vert_F^2},
	\\\left\Vert\mathcal{A}\right\Vert_{ATV}&\triangleq\left\Vert\mathcal{A}\times_1\boldsymbol{D}_M\right\Vert_1+\left\Vert\mathcal{A}\times_2\boldsymbol{D}_N\right\Vert_1.
\end{align*}
By $\circ$, we denote the outer product \cite[Def 1.3]{doi:10.1137/070690729}. 
Besides, with a little notation abuse, we also denote the tensor conjugate transpose \cite{8606166} of $\mathcal{A}$ with $\mathcal{A}^\ast$. To facilitate presentation, we use $\mathtt{svds}(\cdot)$ and $\mathtt{permute}(\cdot)$ to denote two operators with the same functionality of their MATLAB namesakes.
\subsection{Tensor Preliminaries}
The tools to depict the tensor rank are various, among which this paper employs the tensor singular value decomposition (t-SVD), whose definition is as follows.
\begin{definition}[t-SVD \cite{9381277}]
	For a $3$-way tensor $\mathcal{A}\in\mathbb{R}^{I_1\times I_2\times I_3}$, there exist $\mathcal{U}\in\mathbb{R}^{I_1\times I_1\times I_3},\,\mathcal{S}\in\mathbb{R}^{I_1\times I_2\times I_3},\,\mathcal{V}\in\mathbb{R}^{I_2\times I_2\times I_3}$, such that,
	$$\mathcal{A}=\mathcal{U}\star\mathcal{S}\star\mathcal{V}^\ast$$
	in which $\star$ denotes the t-product \cite{KILMER2011641}, $\mathcal{U}$ and $\mathcal{V}$ are orthogonal tensors \cite{8606166}, and $\mathcal{S}$ is a f-diagonal tensor \cite{8606166}.
	\label{Def: t-SVD}
\end{definition}
Accordingly, denoting the Fourier transformation of $\mathcal{S}$ along mode-$3$ as $\overline{\mathcal{S}}$, \textit{i.e.}, $\overline{\mathcal{S}}=\mathcal{S}\times_3\boldsymbol{F}_L$, then the TNN of $\mathcal{A}$ is defined as:
\begin{definition}[TNN \cite{8606166}]
	For the $3$-way tensor $\mathcal{A}$ in Definition \ref{Def: t-SVD}, its TNN can be written as:
	$$\left\Vert\mathcal{A}\right\Vert_{\ast}\triangleq\frac{1}{I_3}\sum_{j=1}^{I_3}\sum_{i=1}^{\min{\left\{I_1,I_2\right\}}}\overline{\mathcal{S}}[i,i,j].$$
\end{definition}
Hence, we can define t-CTV as:
\begin{definition}[t-CTV \cite{10078018}]
	For a $3$-way tensor $\mathcal{A}\in\mathbb{R}^{I_1\times I_2\times I_3}$, its t-CTV can be written as:
	$$\left\Vert\mathcal{A}\right\Vert_{\overset{\sim}{\ast}}\triangleq\frac{1}{\left|\mathfrak{C}\right|}\sum_{n\in\mathfrak{C}\subset\left\{1,2,3\right\}}\left\Vert\mathcal{A}\times_n\boldsymbol{D}_{I_n}\right\Vert_{\ast}$$
	in which $\left|\mathfrak{C}\right|$ denotes the cardinality of $\mathfrak{C}$ and $\mathcal{A}\times_n\boldsymbol{D}_{I_n}$ is called the mode-$n$ gradient tensor.
\end{definition}

To improve the approximation tightness of TNN to the t-SVD rank \cite{10078018}, several works have proposed to use non-convex surrogates, thus the non-convex tensor pseudo nuclear norm (NTPNN) is defined as:
\begin{definition}[NTPNN \cite{9340243}]
	For the $3$-way tensor $\mathcal{A}$ in Definition \ref{Def: t-SVD}, its NTPNN is:
	$$\left\Vert\mathcal{A}\right\Vert_{\ast,\psi}\triangleq\frac{1}{I_3}\sum_{j=1}^{I_3}\sum_{i=1}^{\min{\left\{I_1,I_2\right\}}}\psi\left(\overline{\mathcal{S}}[i,i,j]\right)$$
	in which $\psi(\cdot):[0,+\infty)\rightarrow[0,+\infty)$ is some non-convex surrogate.
\end{definition}
Note that both TNN and NTPNN are restricted to the third mode, which lacks flexibility dealing with the multi-dimensional structure of tensors. This paper then exploits the idea of GTNN to comprehensively consider the all-around low-t-SVD-rankness (LTSVDR) along all modes, for which we define the NTPNN for mode-1 and mode-2:
\begin{definition}[Mode-$\left\{1,2\right\}$ NTPNN]
	For a $3$-way tensor $\mathcal{A}\in\mathbb{R}^{I_1\times I_2\times I_3}$, we denote a permutation operator $\mathtt{P}_n\left(A\right)\triangleq\mathtt{permute}\left(\mathcal{A},[3-n,3,n]\right)$. Then, the mode-$n$ NTPNN of $\mathcal{A}$ is:
	$$\left\Vert\mathcal{A}\right\Vert_{\overset{n}{\ast},\psi}\triangleq\left\Vert\mathtt{P}_n\left(A\right)\right\Vert_{\ast,\psi}, n=1,2.$$
\end{definition}
\subsection{HSR Problem Formulation}
For the convenience of derivation, we model the degradation from HSSI to HSI and MSI following TF-HSR \cite{10904006}. For the sake of consistency, this is reported here. 
\begin{definition}[TF-HSR \cite{10904006}]
	{\it Given the HSI $\mathcal{X}\in \mathbb{R} ^{i_1\times i_2\times I_3}$,
		the MSI $\mathcal{Y}\in \mathbb{R} ^{I_1\times I_2\times i_3}$, the spatial-degradation matrices $\boldsymbol{P}_1\in
		\mathbb{R}^{i_1\times I_1}$ and $\boldsymbol{P}_2\in
		\mathbb{R}^{i_2\times I_2}$, and the spectral-degradation matrix $\boldsymbol{P}_3\in \mathbb{R}^{i_3\times
			I_3}$ with $i_1< I_1$, $i_2< I_2$, and $i_3<
		I_3$, the TF-HSR problem estimates the most appropriate HSSI
		$\mathcal{Z}\in \mathbb{R} ^{I_1\times I_2\times I_3}$, such that
		\begin{equation}
			\begin{aligned}
				\mathcal{X} & =\mathcal{Z}\times_1\boldsymbol{P}_1\times_2\boldsymbol{P}_2, \\
				\mathcal{Y} & =\mathcal{Z}\times_3\boldsymbol{P}_3.
			\end{aligned}
			\label{eq: Tensor Formulation}
	\end{equation}}
	\label{Def: Tensor Formulation}
\end{definition}
\section{Methodology}
\label{sec:Methodology}
\subsection{Multi-level-prior Observation of HSSI}
\begin{figure*}[t]
	\centering
	\includegraphics[width=\linewidth]{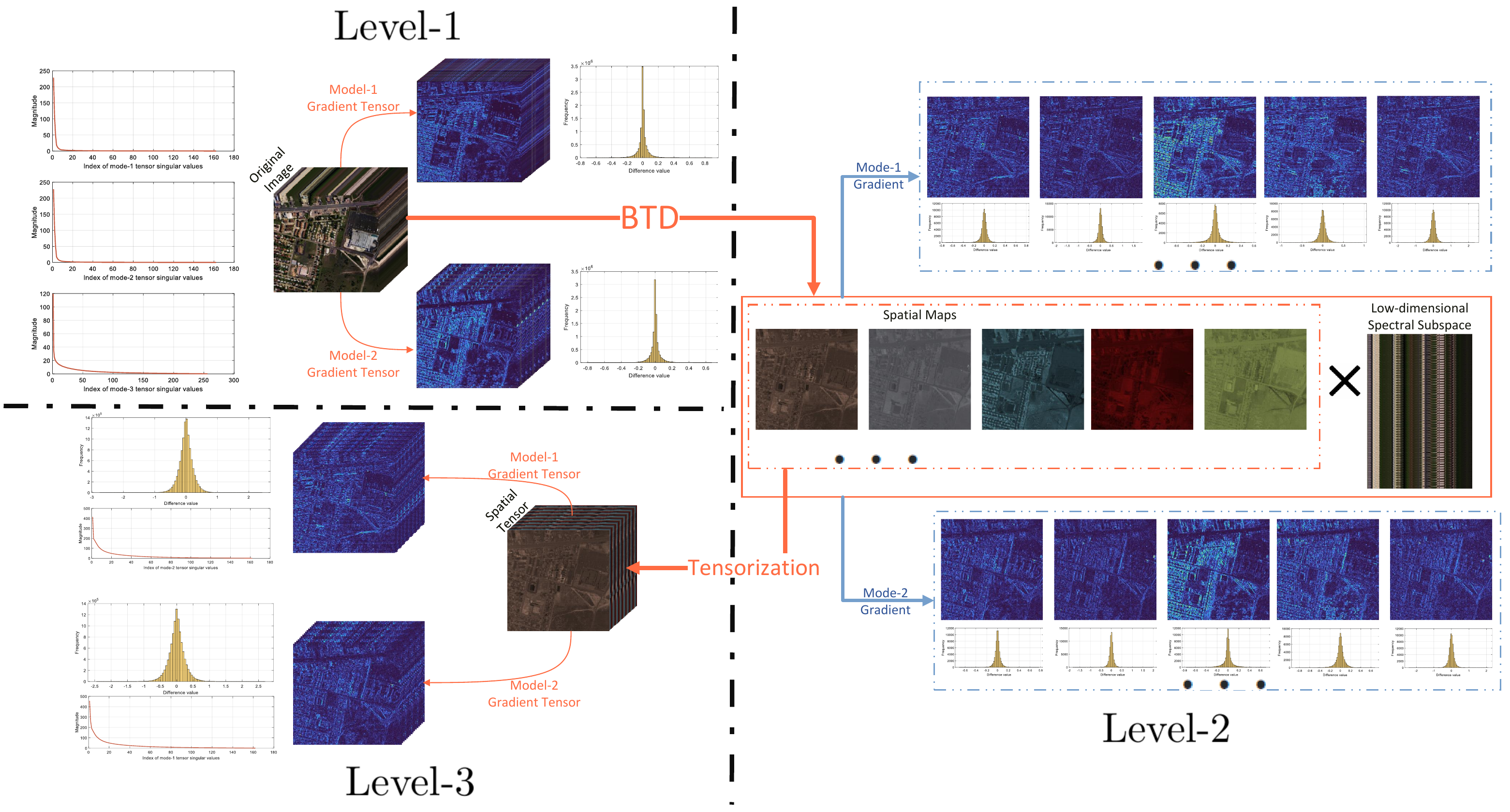}
	\caption{\label{fig: Overall}Multi-level-prior observation. Level-1: the original data domain; level-2: the factorization domain; Level-3: the enhanced tensorization domain.}
\end{figure*}
Since a real HSSI $\mathcal{Z}$ is by no means accessible, we stay in line with other HSR researches (\textit{e.g.}, \cite{9556548,10904006,9076843}) to observe the priors of HSI as a surrogate. Firstly, as shown by the Level-1 part of Fig. \ref{fig: Overall}, in the original HSSI, there underlies the multi-mode LTSVDR, \textit{i.e.}, the tensor singular values of $\mathtt{P}_1\left(\mathcal{Z}\right)$, $\mathtt{P}_2\left(\mathcal{Z}\right)$, and $\mathcal{Z}$ exhibit rapid decay. Besides, taking into account the spatial TV prior, the mode-{1,2} gradient tensors of $\mathcal{Z}$ should exhibit sparsity.

Furthermore, following the BTD paradigm, the HSSI can be decomposed into a spectral subspace and a series of spatial maps, which is illustrated by the Level-2 part of Fig. \ref{fig: Overall} and written as:
\begin{align}
	\mathcal{Z}=\mathcal{A}\times_3\boldsymbol{S}=\sum_{r=1}^{R}\mathcal{A}_{:,:,r}\circ\boldsymbol{s}_r,
	\label{eq: BTD}
\end{align}
where $\mathcal{A}\in\mathbb{R}^{I_1\times I_2 \times R}$ is the spatial tensor whose frontal slices $\left\{\mathcal{A}_{:,:,r}\right\}_{r=1}^R$ are the spatial maps and $\boldsymbol{S}\in\mathbb{R}^{I_3\times R}$ is the spectral subspace with $\boldsymbol{s}_r$ its $r$th column. Thereupon, the spectral low-rankness of the HSSI is reflected by the low-dimensionality of $\boldsymbol{S}$, while the spatial smoothness is encoded in the horizontal and vertical gradients of the spatial maps. 

Moreover, \cite{11005667} has suggested that by inspecting the stacked spatial tensor $\mathcal{A}$ instead of the spatial maps, we get the enhanced version of spatial priors. To this end, and also inspired by t-CTV, we examine the spatial TV and low-rankness both from the gradient tensors of $\mathcal{A}$, as visualized by the Level-3 part of Fig. \ref{fig: Overall}.

Thus, taking $\mathcal{Z}=\mathcal{A}\times_3\boldsymbol{S}$ into Eq. \eqref{eq: Tensor Formulation}, the HSR problem is reformulated as:
\begin{equation}
	\begin{aligned}
			\text{Find}\,\,\mathcal{A},\,\boldsymbol{S},\quad s.t.,\,\,
		&\mathcal{X}=\mathcal{A}\times_1\boldsymbol{P}_1\times_2\boldsymbol{P}_2\times_3\boldsymbol{S},\\
		&\mathcal{Y}=\mathcal{A}\times_3\boldsymbol{P}_3\boldsymbol{S}.
	\end{aligned}
	\nonumber
\end{equation}
\subsection{Multi-level-prior Tensor Representation}
As it is observed that multi-level-priors are equipped within the HSSI in the BTD paradigm, we now establish a constraint to compactly represent these multi-level-priors. While t-CTV jointly enforces low-rankness and smoothness via a nuclear norm constraint on gradient tensors, it computes the norm uniformly across all modes. This fixed-mode pattern lacks the flexibility to capture the distinct structural priors along different dimensions of HSSI data. To address this limitation, we propose NMS-t-CTV, which adaptively imposes mode-specific low-rank constraints on gradient tensors for more accurate prior characterization.
\begin{definition}[NMS-t-CTV]
	For the spatial tensor $\mathcal{A}$ defined in Eq. \eqref{eq: BTD}, its NMS-t-CTV, denoted as $\left\lVert\mathcal{A}\right\lVert_{\overset{\sim}{\ast},\mathtt{\psi}}$, is defined as:
	\begin{equation}
		\left\lVert\mathcal{A}\right\lVert_{\overset{\sim}{\ast},\mathtt{\psi}}=\frac{1}{2}\sum_{n=1}^{2}\,\left\lVert\mathcal{A}\times_n\boldsymbol{D}_{I_n}\right\lVert_{\overset{3-n}{\ast},\psi}
		\nonumber
	\end{equation}
	\label{Def: NMS-t-CTV}
\end{definition}
In the definition above, the mode along which the LTSVDR is employed with respect to each gradient tensor is designed. Specifically, for the mode-1 gradient tensor, its mode-2 LTSVDR is considered while the mode-1 LTSVDR is imposed on the mode-2 gradient tensor. This facilitates the cross-representation of multi-level-priors in a compact way. Besides, we employ the NTPNN in replacement of the convex TNN to achieve tighter approximation of the tensor rank. By doing so, the proposed CMlpTR model for HSR is as follows:
\begin{equation}
	\begin{aligned}
		\min_{\mathcal{A},\boldsymbol{S}}&\,\,\frac{1}{2}\sum_{n=1}^{2}\,\left\lVert\mathcal{A}\times_n\boldsymbol{D}_{I_n}\right\lVert_{\overset{3-n}{\ast},\psi},\\
		s.t.,\, &\mathcal{X}=\mathcal{A}\times_1\boldsymbol{P}_1\times_2\boldsymbol{P}_2\times_3\boldsymbol{S},\\
		&\mathcal{Y}=\mathcal{A}\times_3\boldsymbol{P}_3\boldsymbol{S}.\\
	\end{aligned}
	\label{eq: CMlpTR}
\end{equation}

Model \eqref{eq: CMlpTR} is then designed to compactly co-represent the aforementioned multi-level-priors with only two-block variables and two constraints. To help justifying its ability, we deliver Prop. \ref{Prop: Mlp} below.
\begin{proposition}
	Regarding Eq. \eqref{eq: BTD}, if $\boldsymbol{S}\in\mathbb{R}^{I_3\times R}$ is semi-unitary, \textit{i.e.}, $\boldsymbol{S}^\ast\boldsymbol{S}=\boldsymbol{I}_R$, there follows: 
	\begin{align}
		rank_{\overset{3-n}{t}}\left(\mathcal{Z}\right)-1\leq rank_{\overset{3-n}{t}}\left(\mathcal{A}\times_n\boldsymbol{D}_{I_n}\right)\leq rank_{\overset{3-n}{t}}\left(\mathcal{Z}\right)\label{eq: Mlp1},\notag\\\quad n=1,2.
	\end{align}
	If we further have that $\mathcal{A}$ is bounded and $\psi(\cdot)$ is endowed with the properties that:
	
	\textbf{P1:} $\psi(\cdot)$ is concave and non-decreasing, with $\psi(0)=0$,
	
	\textbf{P2:} $\psi'(\cdot)$ is convex and non-increasing, with $\underset{x\rightarrow0^+}{\lim}\psi'(x)<\infty$,
	
	\noindent{then the NMS-t-CTV is compatible with the TV norms, \textit{i.e.}, $\exists\,a_1,a_2,a_3,a_4>0,$ such that} 
	\begin{equation}
		\begin{aligned}
			a_1\left\lVert{\mathcal{A}}\right\lVert_{TV}\leq\left\lVert\mathcal{A}\right\lVert_{\overset{\sim}{\ast},\mathtt{\psi}}\leq a_2\left\lVert{\mathcal{A}}\right\lVert_{TV},\\
			a_3\left\lVert{\mathcal{A}}\right\lVert_{ATV}\leq\left\lVert\mathcal{A}\right\lVert_{\overset{\sim}{\ast},\mathtt{\psi}}\leq a_4\left\lVert{\mathcal{A}}\right\lVert_{ATV}.
		\end{aligned}
		\label{eq: Mlp2}
	\end{equation}
	\label{Prop: Mlp}
\end{proposition}
\begin{proof}
	Deferred to Sec. I in Appendix.
\end{proof}

As a result, the co-representation ability of our CMlpTR can be explained as follows.

\noindent{(1) Since $\boldsymbol{S}$ encodes the spectral subspace of HSI/HSSI, the spectral low-rankness of $\mathcal{Z}$ can be modeled by controlling the dimensionality of $\boldsymbol{S}$.}

\noindent{(2) According to Remark 1 in \cite{10078018}, the multi-dimensional mode-\{1,2\} LTR of $\mathcal{A}$ is well captured.}

\noindent{(3) Based on Eq. \eqref{eq: Mlp1}, our CMlpTR can express the spatial low-rankness of the original $\mathcal{Z}$, which together with (1) realizes the co-representation of the multi-dimensional LTR within $\mathcal{Z}$.}

\noindent{(4) Eq. \eqref{eq: Mlp2} asserts that the spatial TV prior of HSI/HSSI, reflected on the spatial tensor $\mathcal{A}$, is well characterized. Moreover, since we have $\left\lVert{\mathcal{A}}\right\lVert_{ATV}$=$\sum_{r=1}^{R}\left\lVert{\mathcal{A}_{;,:,r}}\right\lVert_{ATV}$, the TV prior on the 2D spatial maps of the BTD level is also expressed.}

Besides, it is known from Prop. \ref{Prop: Mlp} that to take (1)-(4) into effect, some additional conditions should be satisfied by $\boldsymbol{S}$ and $\psi(\cdot)$. Thus, we specify the non-convex surrogate as $\psi(x)=\frac{\log(\gamma x+1)}{\log(\gamma+1)}$ \cite{friedman2012fast} and add the semi-unitary constraint on $\boldsymbol{S}$ to obtain:
\begin{equation}
	\begin{aligned}
		\min_{\mathcal{A},\boldsymbol{S}}&\,\,\frac{1}{2}\sum_{n=1}^{2}\,\left\lVert\mathcal{A}\times_n\boldsymbol{D}_{I_n}\right\lVert_{\overset{3-n}{\ast},\psi},\\
		s.t.,\, &\mathcal{X}=\mathcal{A}\times_1\boldsymbol{P}_1\times_2\boldsymbol{P}_2\times_3\boldsymbol{S},\\
		&\mathcal{Y}=\mathcal{A}\times_3\boldsymbol{P}_3\boldsymbol{S},\\
		&\boldsymbol{S}^\ast\boldsymbol{S}=\boldsymbol{I}_R.\\
	\end{aligned}
	\label{eq: CMlpTR completed}
\end{equation}
\subsection{Optimization Algorithm}
To find a solution to Problem \eqref{eq: CMlpTR completed}, it is noticed that auxiliary variables are unavoidable in solving the $\mathcal{A}$-subproblems, leading to the expansion of the problem scale. In \cite{10904006}, the authors demonstrate that the subspace and coefficient can be determined in a sequential manner without hurting the model performance. Thus, we extract the semi-unitary spectral subspace $\boldsymbol{S}$ simply from HSI $\mathcal{X}$ as:
$$\left[\boldsymbol{S},\sim,\sim\right]=\mathsf{svds}\left(\mathcal{X}_{[3]},R\right).$$
Then, the problem can be reduced as:
\begin{equation}
	\begin{aligned}
		\min_{\mathcal{A}}&\,\,\frac{1}{2}\sum_{n=1}^{2}\,\left\lVert\mathcal{A}\times_n\boldsymbol{D}_{I_n}\right\lVert_{\overset{3-n}{\ast},\psi},\\
		s.t.,\, &\mathcal{X}=\mathcal{A}\times_1\boldsymbol{P}_1\times_2\boldsymbol{P}_2\times_3\boldsymbol{S},\\
		&\mathcal{Y}=\mathcal{A}\times_3\boldsymbol{P}_3\boldsymbol{S},
	\end{aligned}
	\nonumber
	\label{eq: CMlpTR reduced}
\end{equation}
to solve it, we introduce some auxiliary variables, \textit{i.e.}, $\mathcal{G}_1$ and $\mathcal{G}_2$, and we obtain:
\begin{equation}
	\begin{aligned}
		&\min_{\mathcal{A},\,\left\{\mathcal{G}_n\right\}_{n=1,2}}\,\,\frac{1}{2}\sum_{n=1}^{2}\,\left\lVert\mathcal{G}_n\right\lVert_{\overset{3-n}{\ast},\psi},\\
		s.t.,\, &\mathcal{X}=\mathcal{A}\times_1\boldsymbol{P}_1\times_2\boldsymbol{P}_2\times_3\boldsymbol{S},\\
		&\mathcal{Y}=\mathcal{A}\times_3\boldsymbol{P}_3\boldsymbol{S},\\
		&\mathcal{G}_n=\mathcal{A}\times_n\boldsymbol{D}_{I_n}, n=1,2.
	\end{aligned}
	\label{eq: CMlpTR prepared}
\end{equation}
Usually, problem like \eqref{eq: CMlpTR prepared} is easily solved by ADMM \cite{Wright-Ma-2022}. However, the $\mathcal{A}$-subproblem suffers from the high-dimensionality of $\mathcal{A}$, compelling us to make adjustment to its ADMM-induced updating steps. The resulting algorithm can be viewed as a composition of LADMM. Below, we elaborate on it and provide its convergence analysis in Thm. \ref{Thm: Convergence}.

Let us start defining the augmented Lagrangian function of \eqref{eq: CMlpTR prepared} as:
\begin{align}
		&\quad L(\mathcal{A},\,\left\{\mathcal{G}_n\right\}_{n=1,2},\,\mathcal{M}_x,\,\mathcal{M}_y,\,\left\{\mathcal{M}_n\right\}_{n=1,2})\notag\\&=\frac{1}{2}\sum_{n=1}^{2}\,\left\lVert\mathcal{G}_n\right\lVert_{\overset{3-n}{\ast},\psi}+\frac{\rho}{2}\Bigg(\Big\lVert\mathcal{X}+\frac{\mathcal{M}_x}{\rho}-\mathcal{A}\times_1\boldsymbol{P}_1\times_2\notag\\&\quad\boldsymbol{P}_2\times_3\boldsymbol{S}\Big\lVert_F^2+\Big\lVert\mathcal{Y}+\frac{\mathcal{M}_y}{\rho}-\mathcal{A}\times_3\boldsymbol{P}_3\boldsymbol{S}\Big\lVert_F^2\notag\\&\quad+\sum_{n=1}^{2}\Big\lVert\mathcal{G}_n+\frac{\mathcal{M}_n}{\rho}-\mathcal{A}\times_n\boldsymbol{D}_{I_n}\Big\lVert_F^2\Bigg)\notag\\&\quad-\frac{1}{2\rho}\Bigg(\Big\lVert\mathcal{M}_x\Big\lVert_F^2+\Big\lVert\mathcal{M}_y\Big\lVert_F^2+\sum_{n=1}^{2}\Big\lVert\mathcal{M}_n\Big\lVert_F^2\Bigg)
	\label{eq: Lag}
\end{align}
where $\left\{\mathcal{M}_x,\,\mathcal{M}_y,\,\left\{\mathcal{M}_n\right\}_{n=1,2}\right\}$ are the Lagrangian multipliers and $\rho$ is the penalty parameter.
\begin{itemize}
	\item $\mathcal{A}$-subproblem: Fixing $\left\{\mathcal{G}_n\right\}_{n=1,2}$, $\mathcal{A}$ is updated via solving
	\begin{equation}
		\begin{aligned}
			\min_{\mathcal{A}}\,L_1\left(\mathcal{A}\right)&\triangleq\Big\lVert\mathcal{X}+\frac{\mathcal{M}_x}{\rho}-\mathcal{A}\times_1\boldsymbol{P}_1\times_2\boldsymbol{P}_2\times_3\boldsymbol{S}\Big\lVert_F^2\\&\quad+\Big\lVert\mathcal{Y}+\frac{\mathcal{M}_y}{\rho}-\mathcal{A}\times_3\boldsymbol{P}_3\boldsymbol{S}\Big\lVert_F^2\\&\quad+\sum_{n=1}^{2}\Big\lVert\mathcal{G}_n+\frac{\mathcal{M}_n}{\rho}-\mathcal{A}\times_n\boldsymbol{D}_{I_n}\Big\lVert_F^2.
		\end{aligned}
		\label{eq: A-sub}
	\end{equation}
	Typically, Eq. \eqref{eq: A-sub} can be easily solved by setting its gradient to zero. However, the high-dimensionality of $\mathcal{A}$ is a curse for doing so. Thus, inspired by LADMM, we approximate its solution by a simple one-step gradient descend, \textit{i.e.}
	\begin{equation}
	    \mathcal{A}\leftarrow\mathcal{A}-\frac{\bigtriangledown L_1\left(\mathcal{A}\right)}{\tau}
	    \label{eq: A-update}
	\end{equation}
	in which the gradient is calculated as
	\begin{equation}
		\begin{aligned}
			&\quad\bigtriangledown  L_1\left(\mathcal{A}\right)\\&=2\Bigg(\mathcal{A}\times_1\boldsymbol{P}_1^*\boldsymbol{P}_1\times_2\boldsymbol{P}_2^*\boldsymbol{P}_2+\mathcal{A}\times_3\boldsymbol{S}^*\boldsymbol{P}_3^*\boldsymbol{P}_3\boldsymbol{S}\\&\quad+\sum_{n=1}^{2}\mathcal{A}\times_n\boldsymbol{D}_{I_n}^*\boldsymbol{D}_{I_n}-\left(\mathcal{X}+\frac{\mathcal{M}_x}{\rho}\right)\times_1\boldsymbol{P}_1^*\times_2\boldsymbol{P}_2^*\\&\quad\times_3\boldsymbol{S}^*-\left(\mathcal{Y}+\frac{\mathcal{M}_y}{\rho}\right)\times_3\boldsymbol{S}^*\boldsymbol{P}_3^*-\\&\quad\sum_{n=1}^{2}\left(\mathcal{G}_n+\frac{\mathcal{M}_n}{\rho}\right)\times_n\boldsymbol{D}_{I_n}^*\Bigg)
		\end{aligned}
		\label{eq: A-Gradient}
	\end{equation}
	and the updating stride $\tau$ should satisfy the Lipschitz continuity condition that
	\begin{equation}
		\Big\Vert \bigtriangledown L_1\left(\mathcal{A}\right)- \bigtriangledown L_1\left(\mathcal{A}'\right)\Big\Vert_F\leq \tau\Big\lVert \mathcal{A}- \mathcal{A}'\Big\Vert_F
		\label{eq: Lipschitz}
	\end{equation}
	for any $\mathcal{A},\mathcal{A}'$. To this end, we set
	\begin{equation}
		\begin{aligned}
			\tau=2\Big(\left\|\boldsymbol{P}_1\right\|^2\left\|\boldsymbol{P}_2\right\|^2+\left\|\boldsymbol{P}_3\boldsymbol{S}\right\|^2+\left\|\boldsymbol{P}_1\right\|^2+\left\|\boldsymbol{P}_2\right\|^2\Big).
		\end{aligned}
		\label{eq: Lipschitz Variable}
	\end{equation}
	\item $\mathcal{G}_n$-subproblem, $n=1,2$: Fixing $\mathcal{A}$, $\mathcal{G}_n$ is updated by solving
	\begin{equation}
		\begin{aligned}
			\min_{\mathcal{G}_n}\,L_{n+1}\left(\mathcal{G}_n\right)&\triangleq\left\lVert\mathcal{G}_n\right\lVert_{\overset{3-n}{\ast},\psi}\\&+\rho\Big\lVert\mathcal{G}_n+\frac{\mathcal{M}_n}{\rho}-\mathcal{A}\times_n\boldsymbol{D}_{I_n}\Big\lVert_F^2
		\end{aligned}
		\label{eq: G-sub}
	\end{equation}
	which is equivalent to
	\begin{equation}
		\begin{aligned}
			&\min_{\mathcal{G}_n}\,L_{n+1}\left(\mathcal{G}_n\right)\triangleq\left\lVert\mathtt{P}_{3-n}\left(\mathcal{G}_n\right)\right\lVert_{\overset{n}{\ast},\psi}\\&+\rho\Big\lVert\mathtt{P}_{3-n}\left(\mathcal{G}_n\right)+\mathtt{P}_{3-n}\left(\frac{\mathcal{M}_n}{\rho}-\mathcal{A}\times_n\boldsymbol{D}_{I_n}\right)\Big\lVert_F^2.
		\end{aligned}
		\nonumber
	\end{equation}
	By Thm. 2 in \cite{11005667}, we firstly obtain the t-SVD of $\mathtt{P}_{3-n}\left(\mathcal{A}\times_n\boldsymbol{D}_{I_n}-\frac{\mathcal{M}_n}{\rho}\right)$ as $\mathcal{U}\star\mathcal{S}\star\mathcal{V}^\ast$, then the solution to Eq. \eqref{eq: G-sub} is $\mathtt{P}_{3-n}^{-1}\left(\mathcal{U}\star\mathcal{T}\star\mathcal{V}^\ast\right)$ where $\mathcal{\tau}$ is an f-diagonal tensor with
	\begin{equation}
		\begin{aligned}
			&\overline{\mathcal{T}}[{j,j,i}]=\arg\min_{x\geq 0}\,\psi(x)+\rho(x-\overline{\mathcal{S}}[{j,j,i}])^2\\
			&i=1,2,\cdots,I_{3-n},\,j=1,2,\cdots,\min\left\{I_n,R\right\}
		\end{aligned}
		\nonumber
	\end{equation}
	which is solved via the generalized accelerated iterative \cite{9916142}, and the remaining elements being zero.
	\item Auxiliary variables: Following the convention of the ADMM framework, we update the multipliers as
	\begin{equation}
		\begin{aligned}
			\mathcal{M}_x&\leftarrow\mathcal{M}_x+\rho(\mathcal{X}-\mathcal{A}\times_1\boldsymbol{P}_1\times_2\boldsymbol{P}_2\times_3\boldsymbol{S}),\\
			\mathcal{M}_y&\leftarrow\mathcal{M}_y+\rho(\mathcal{Y}-\mathcal{A}\times_3\boldsymbol{P}_3\boldsymbol{S}),\\
			\mathcal{M}_n&\leftarrow\mathcal{M}_n+\rho(\mathcal{G}_n-\mathcal{A}\times_n\boldsymbol{D}_{I_n}),\,n=1,2.\\
		\end{aligned}
	   \label{eq: Multipliers Update}
	\end{equation}
	To help forcing convergence, we update the penalty parameter as
	$\rho\leftarrow\nu\rho$ where $\nu>1$.
\end{itemize}
\subsection{Complexity and Convergence Analysis}
The overall optimization algorithm is summarized in Alg. \ref{alg: CmlpTR}.
\begin{algorithm}[htbp]
	\caption{CMlpTR Optimization Algorithm}
	\label{alg: CmlpTR}
	\begin{algorithmic}[1]
		\REQUIRE The observed HSI $\mathcal{X}$, the MSI $\mathcal{Y}$, the spatial-degradation matrices
		$\boldsymbol{P}_1$ and $\boldsymbol{P}_2$, the spectral-degradation matrix
		$\boldsymbol{P}_3$, the subspace dimensionality $R$, the surrogate parameter $\gamma$.
		\ENSURE Estimated HSSI 
		$\widehat{\mathcal{Z}}$.
		\STATE Initialize penalty parameters $\rho>0$, $\nu>1$, the convergence tolerance $\epsilon>0$, the difference matrices $\boldsymbol{D}_{I_1}$ and $\boldsymbol{D}_{I_2}$.
		\STATE Set $\mathcal{A},\mathcal{G}_1,\mathcal{G}_2,\mathcal{M}_1,\mathcal{M}_2,\mathcal{M}_x$ and $\mathcal{M}_y$
		to $\mathbf{0}$.
		\STATE $\left[\boldsymbol{S},\sim,\sim\right]=\mathsf{svds}\left(\mathcal{X}_{[3]},R\right)$.
		\WHILE{$\max\Bigg\{\Big\Vert\mathcal{X}-\mathcal{A}\times_1\boldsymbol{P}_1\times_2\boldsymbol{P}_2\times_3\boldsymbol{S}\Big\Vert_F,\Big\Vert\mathcal{Y}-\mathcal{A}\times_3\boldsymbol{P}_3\boldsymbol{S}\Big\Vert_F,\left\{\Big\Vert\mathcal{G}_n-\mathcal{A}\times_n\boldsymbol{D}_{I_n}\Big\Vert_F\right\}_{n=1}^2\Bigg\}>\epsilon$}
		\STATE Update $\mathcal{A}$ via Eq. \eqref{eq: A-update}.
		\FOR{$n=1:2$}
		\STATE Update $\mathcal{G}_n$ via Eq. \eqref{eq: G-sub}.
		\ENDFOR
		\STATE Update $\mathcal{M}_1,\mathcal{M}_2,\mathcal{M}_x$ and $\mathcal{M}_y$ via Eq. \eqref{eq: Multipliers Update}.
		\STATE $\rho\leftarrow \nu\rho$.
		\ENDWHILE
		\STATE $\widehat{\mathcal{Z}}\leftarrow\mathcal{A}\times_3\boldsymbol{S}.$
	\end{algorithmic}
\end{algorithm}
In the updating of the spatial tensor $\mathcal{A}$, the main complexity lies in the calculation of the gradient, \textit{i.e.}, Eq. \eqref{eq: A-Gradient}, whose complexity is
\begin{align*}
	C_1\triangleq\mathcal{O}\Big(2*(I_1^2I_2R+I_1I_2^2R)+R^2I_1I_2+i_1i_2I_1I_3\\+i_2I_1I_2I_3+I_1I_2I_3R+i_3I_1I_2R+(I_1-1)I_1I_2R\\+I_1(I_2-1)I_2R\Big)\\=\mathcal{O}\Big(i_1i_2I_1^2I_2^2I_3R^2\Big).
\end{align*}
To update the auxiliary variables $\mathcal{G}_1$ and $\mathcal{G}_2$, their t-SVDs drive the complexity, which in total is
\begin{align*}
	C_2\triangleq\mathcal{O}\Big((I_1(I_2-1)+I_2(I_1-1))R^2\Big)\\=\mathcal{O}\Big(I_1I_2R^2\Big).
\end{align*}
The last part of the complexity comes from the tensor-matrix product in the updating of multipliers, \textit{i.e.}, Eq. \eqref{eq: Multipliers Update}:
\begin{align*}
	C_3\triangleq\mathcal{O}\Big(I_1I_2R(i_1+i_2+i_3+I_1-1+I_2-1+I_3)\Big)\\=\mathcal{O}\Big(i_1i_2i_3I_1^2I_2^2I_3R^2\Big).
\end{align*}
Therefore, the overall complexity of Alg. \ref{alg: CmlpTR} is
\begin{align*}
	C_1+C_2+C_3=\mathcal{O}\Big(i_1i_2i_3I_1^2I_2^2I_3R^2\Big).
\end{align*}

\begin{figure}[b]
	\centering
	\setlength{\tabcolsep}{0.2mm}
	\begin{tabular}{m{0.33\linewidth}m{0.33\linewidth}m{0.33\linewidth}}
		\includegraphics[width=\linewidth]{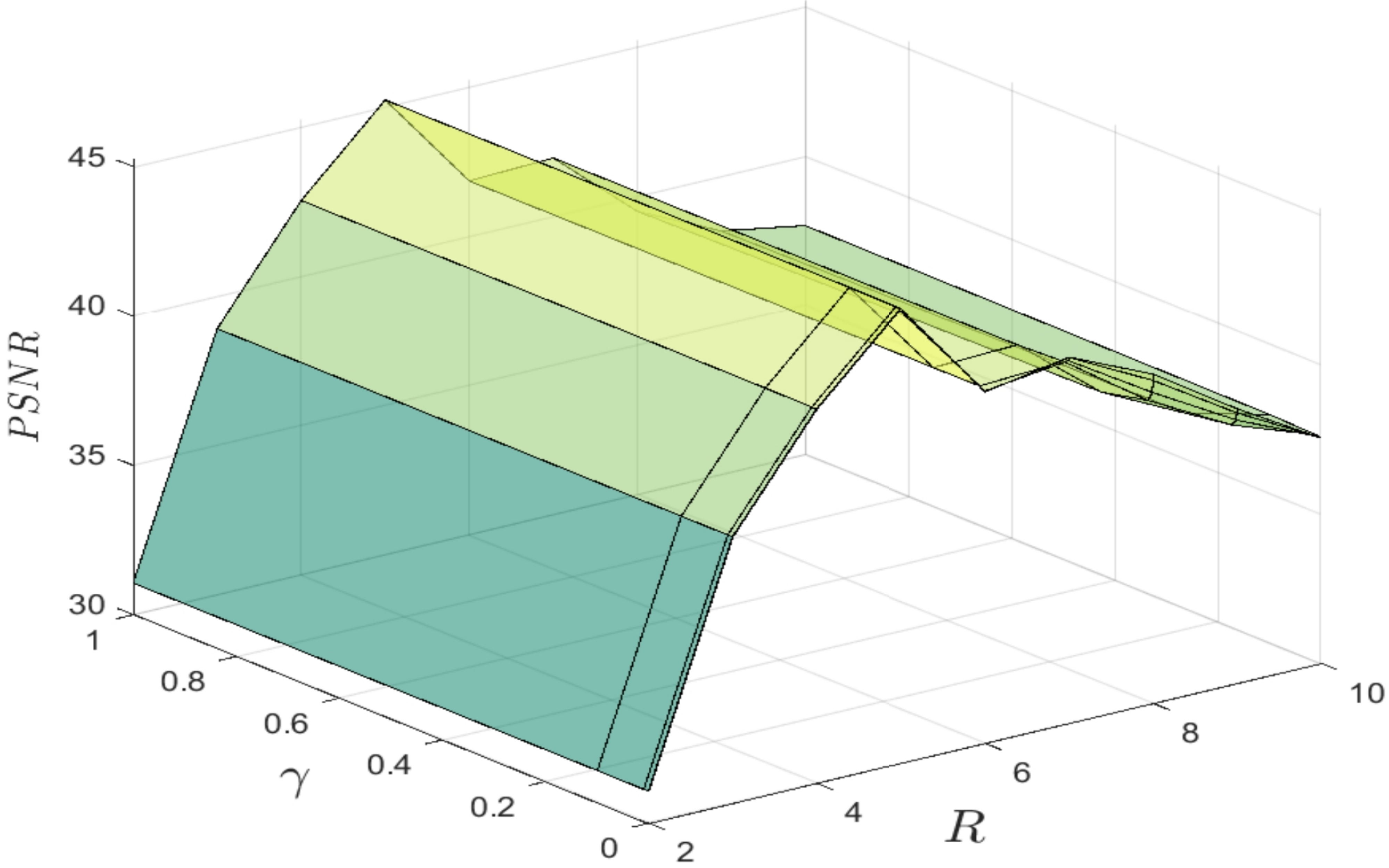}&
		\includegraphics[width=\linewidth]{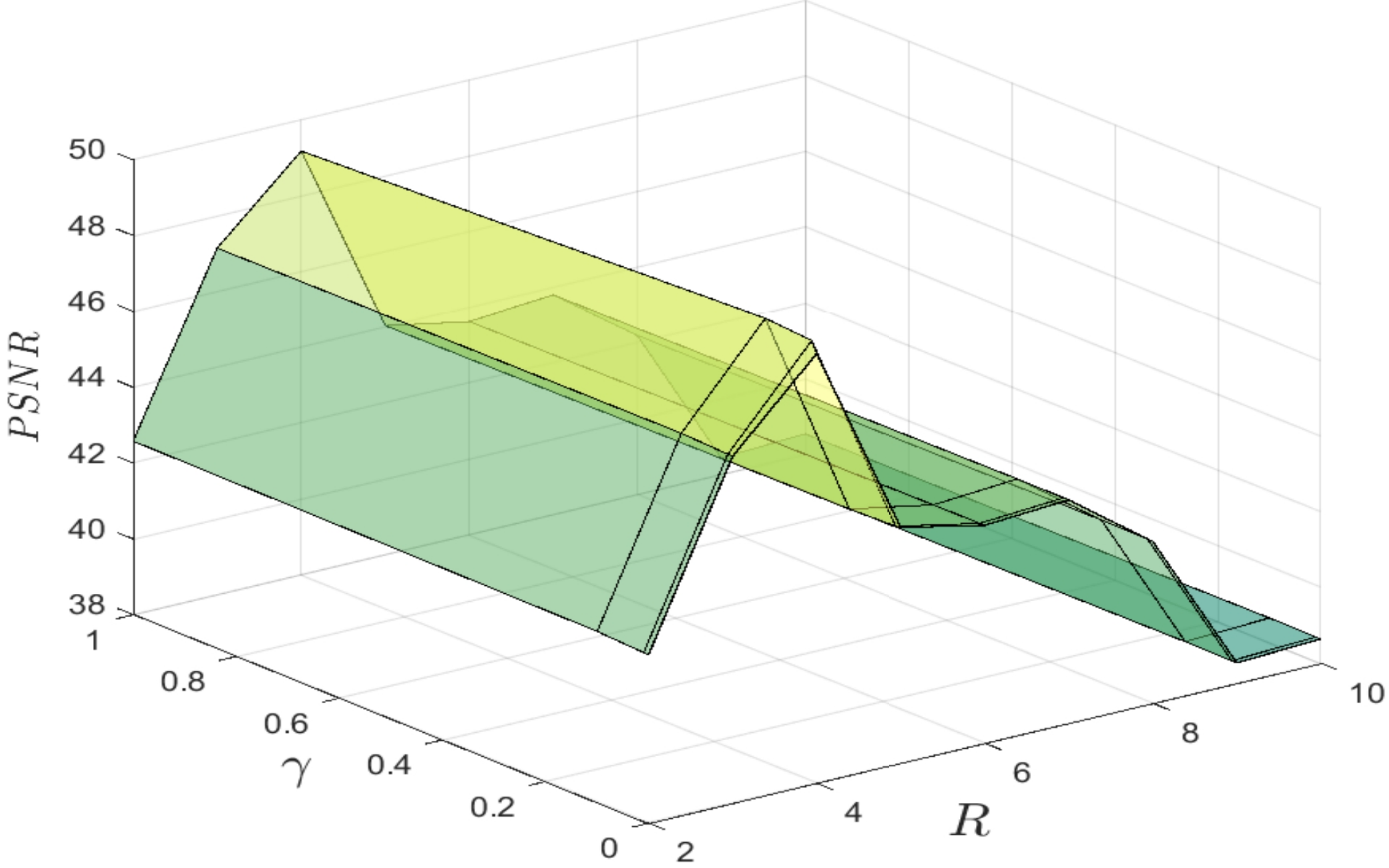}  &
		\includegraphics[width=\linewidth]{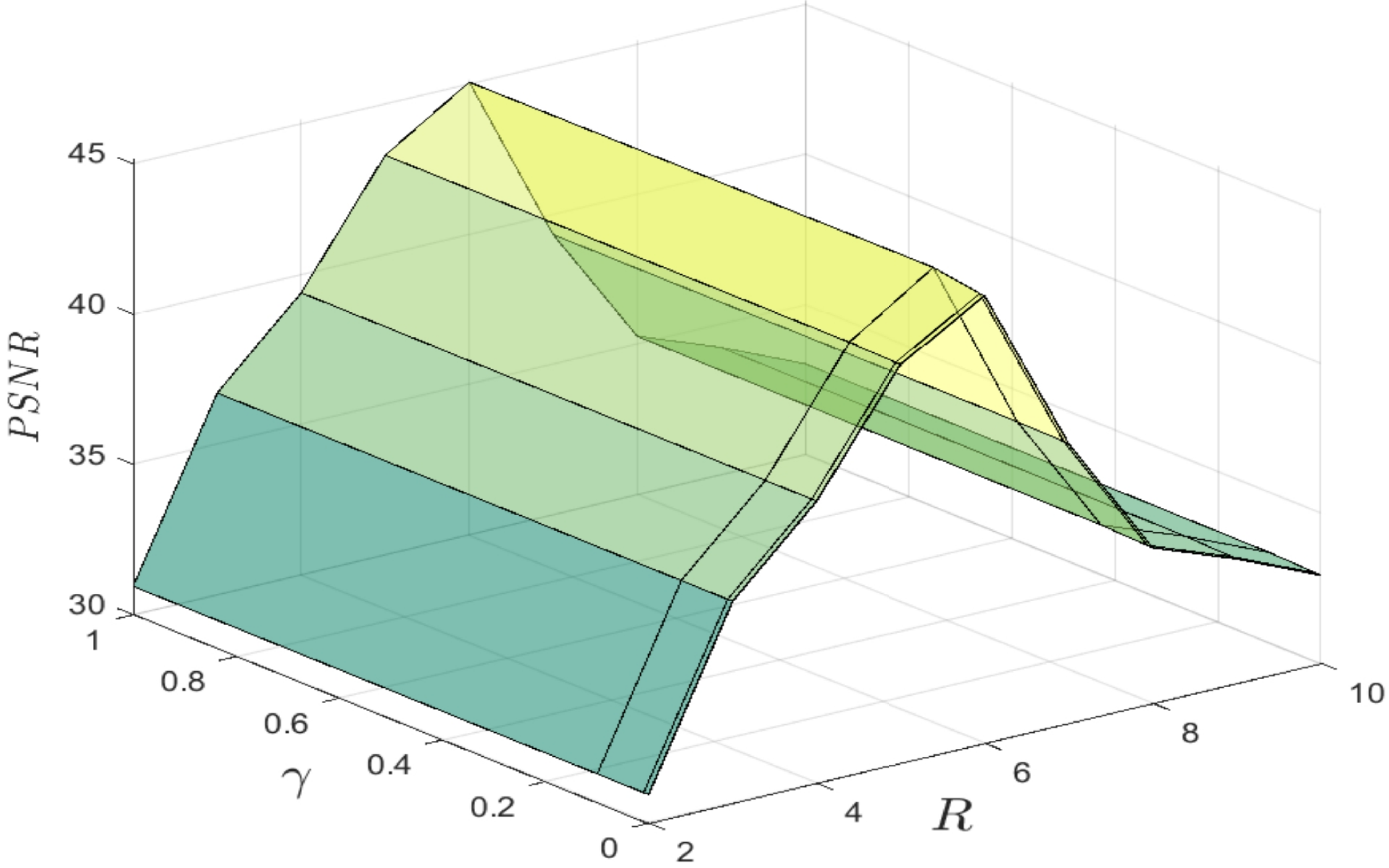}   \\
		\multicolumn{1}{c}{(a)}
		&\multicolumn{1}{c}{(b)}&\multicolumn{1}{c}{(c)}
	\end{tabular}
	\caption{\label{fig: Parameter Tuning} {Hyperparameter sensitivity reflected by the PSNR results for non-blind fusion on (a) URBAN; (b) Houston; (c) WDC.}}
\end{figure}
Alg. \ref{alg: CmlpTR} does not fall into any category of an existing convergence-guaranteed optimization paradigm. Its rough structure resembles the ADMM framework while the solution to some subproblems is customized following the LADMM. Besides, the employment of non-convex surrogate brings more challenges to its convergence issue. Hence, we provide a rigorous convergence analysis below.
\begin{theorem}
	Denoting the sequence generated by Alg. \ref{alg: CmlpTR} as: 
	\begin{equation}
		\begin{aligned}
			\mathfrak{S}_t=\left\{{\cal A}_t,\left\{\mathcal{G}_{n,t}\right\}_{n=1,2},\mathcal{M}_{x,t},\mathcal{M}_{y,t},\left\{\mathcal{M}_{n,t}\right\}_{n=1,2}\right\}_{t\in\mathbb{N}},
		\end{aligned}
		\nonumber
	\end{equation}
	given the boundedness of $\left\{\mathcal{M}_{x,t},\mathcal{M}_{y,t}\right\}$, it follows that:
	
	(1) $\left\{{\cal A}_t,\left\{\mathcal{G}_{n,t}\right\}_{n=1,2}\right\}_{t\in\mathbb{N}}$ is Cauchy sequence,
	
	(2) $\mathfrak{S}_t$ is bounded,
	
	(3) Any accumulation point of $\mathfrak{S}_t$, denoted as
	\begin{equation}
		\begin{aligned}
			\mathfrak{S}_*=\left\{{\cal A}_*,\left\{\mathcal{G}_{n,*}\right\}_{n=1,2},\mathcal{M}_{x,*},\mathcal{M}_{y,*},\left\{\mathcal{M}_{n,*}\right\}_{n=1,2}\right\},
		\end{aligned}
		\nonumber
	\end{equation}
	is a KKT point, \textit{i.e.,} it  satisfies
		\begin{equation}
			\left\{  
			\begin{aligned}  
				&\mathcal{M}_{n,*}\in-\frac{1}{2}\partial\left\lVert\mathcal{G}_{n,\ast}\right\lVert_{\overset{3-n}{\ast},\psi},\,n=1,2,   \\  
				&\bigtriangledown L_1\left(\mathcal{A}_\ast\right)=0,\\
				&\mathcal{X}=\mathcal{A}_*\times_1\boldsymbol{P}_1\times_2\boldsymbol{P}_2\times_3\boldsymbol{S},\\
				&\mathcal{Y}=\mathcal{A}_*\times_3\boldsymbol{P}_3\boldsymbol{S},\\
				&\mathcal{G}_{n,*}=\mathcal{A}_*\times_n\boldsymbol{D}_{I_n}, n=1,2.   
			\end{aligned}  
			\right.	
			\nonumber		
		\end{equation} 
\label{Thm: Convergence}
\end{theorem}
\begin{proof}
	Deferred to Sec. II in Appendix.
\end{proof}
\section{Experiments}
\label{sec:Experiments}
In this section, we empirically validate the proposed method's effectiveness for the HSR task. Both non-blind and blind HSR are tested, among which the degradation matrices for blind HSR are estimated via \cite{7000523}. Unsupervised state-of-the-art methods, including model-based Hysure \cite{7000523}, LTTR \cite{8603806}, LRTA \cite{9328229}, ASLA \cite{10149108}, GTNN \cite{10522984} and data-driven SURE \cite{9924190}, ZSL \cite{10137388} are involved for comparison. All model-based methods are implemented in MATLAB
R2020b on Intel\textsuperscript{\textregistered}
Core\textsuperscript{TM} i5-1135G7 CPU @ 2.40 GHz with 16-GB RAM, while the data-driven ones are implemented by Python 3.11.5 with Pytorch 1.10.2 on an NVIDIA RTX 3090 GPU with 24-GB RAM. The performance is measured in terms of peak signal-to-noise ratio (PSNR),
{relative dimensionless global error synthesis ({ERGAS})}, spectral angle mapper (SAM), and structural similarity metric (SSIM) {\cite{WANG2025103166}}. The higher the PSNR and SSIM, the better the performance. The lower the {ERGAS} and SAM, the higher the performance. Ideal values are +$\infty$, 0, 0, and 1 for PSNR, {ERGAS}, SAM, and SSIM, respectively.
\begin{figure}[htbp!]
	\centering
	\setlength{\tabcolsep}{0.2mm}
	\begin{tabular}{m{0.2\linewidth}m{0.2\linewidth}m{0.2\linewidth}m{0.2\linewidth}m{0.2\linewidth}}
		\includegraphics[width=\linewidth]{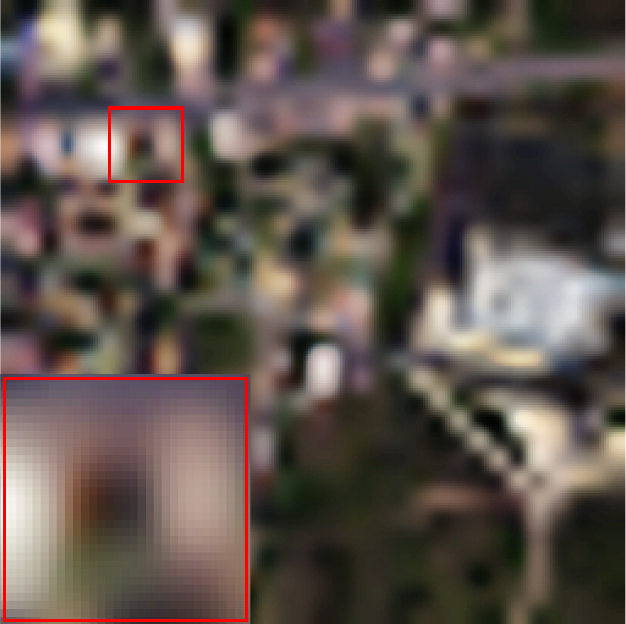}&
		\includegraphics[width=\linewidth]{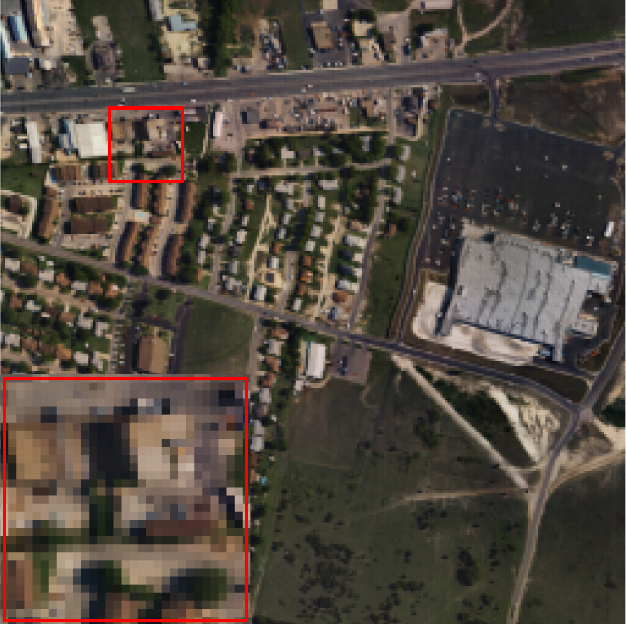}  &
		\includegraphics[width=\linewidth]{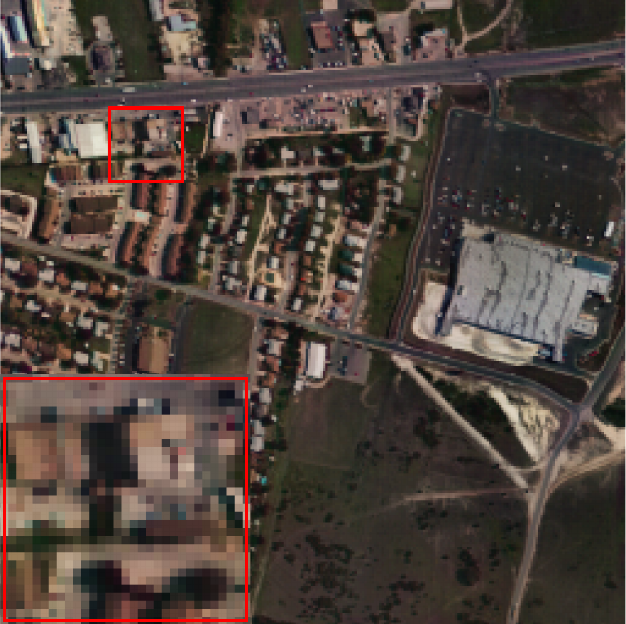}    &
		\includegraphics[width=\linewidth]{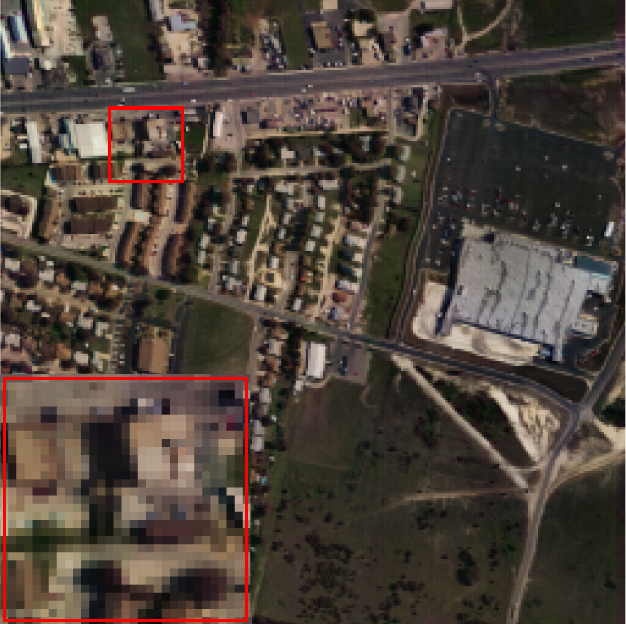}   &
		\includegraphics[width=\linewidth]{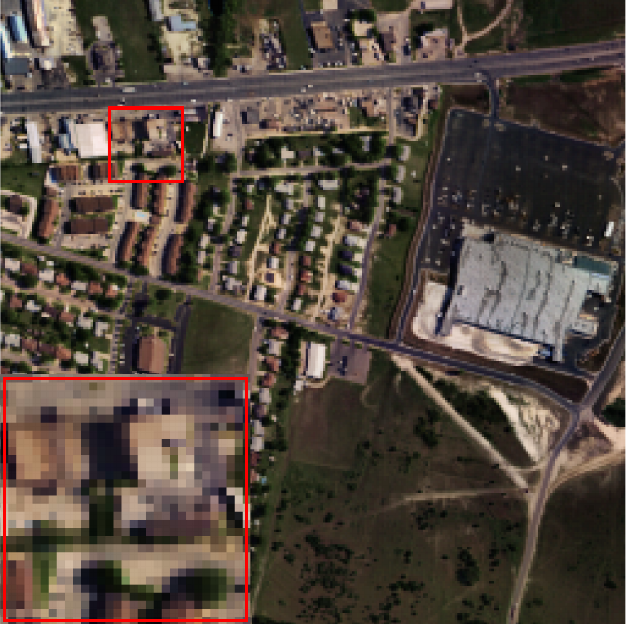}  \\
		\includegraphics[width=\linewidth]{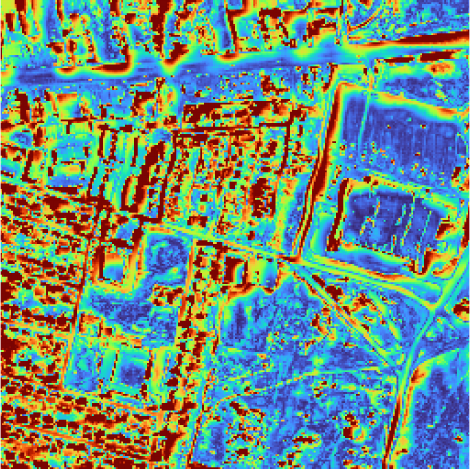}&
		\includegraphics[width=\linewidth]{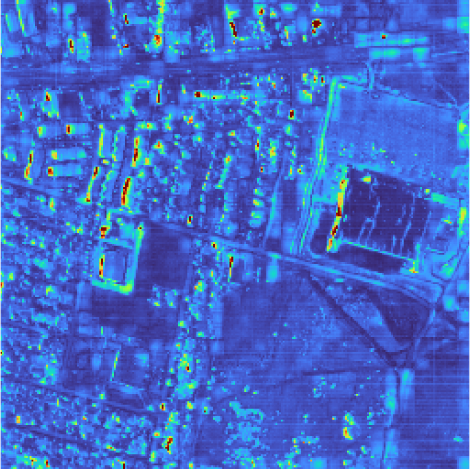}  &
		\includegraphics[width=\linewidth]{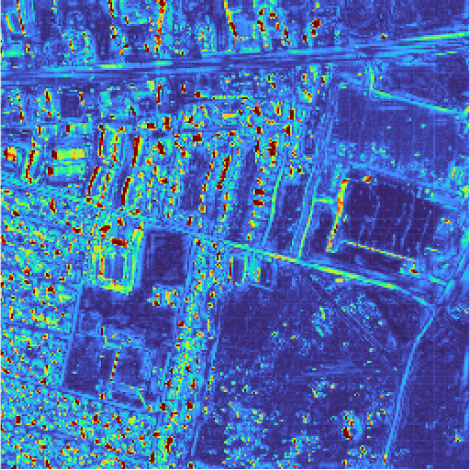}    &
		\includegraphics[width=\linewidth]{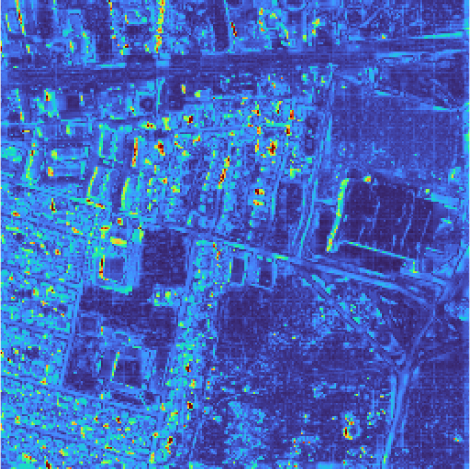}   &
		\includegraphics[width=\linewidth]{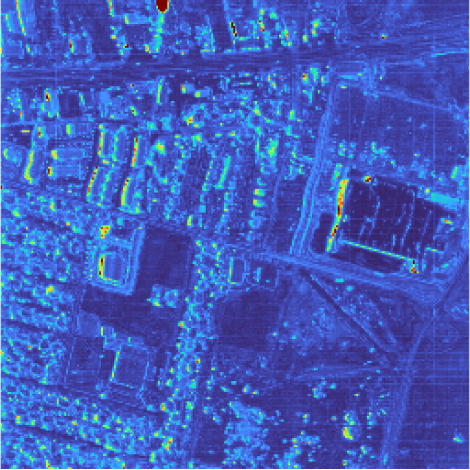} \\
		\multicolumn{1}{c}{\footnotesize{Bicubic}}
		&\multicolumn{1}{c}{\footnotesize{Hysure}}
		& \multicolumn{1}{c}{\footnotesize{LTTR}}
		& \multicolumn{1}{c}{\footnotesize{LRTA}}
		& \multicolumn{1}{c}{\footnotesize{SURE}}\\
		\includegraphics[width=\linewidth]{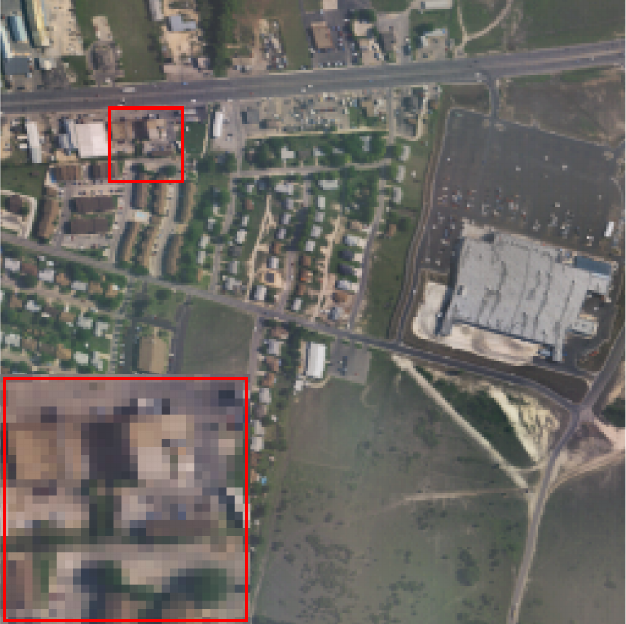}  &
		\includegraphics[width=\linewidth]{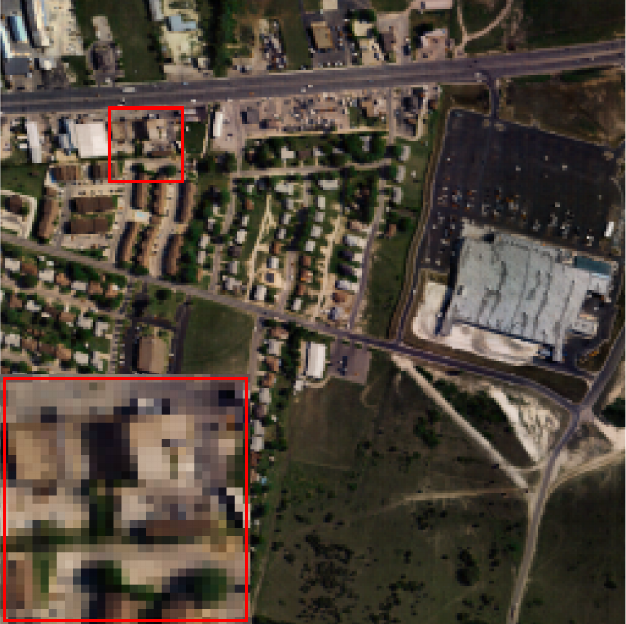}  &
		\includegraphics[width=\linewidth]{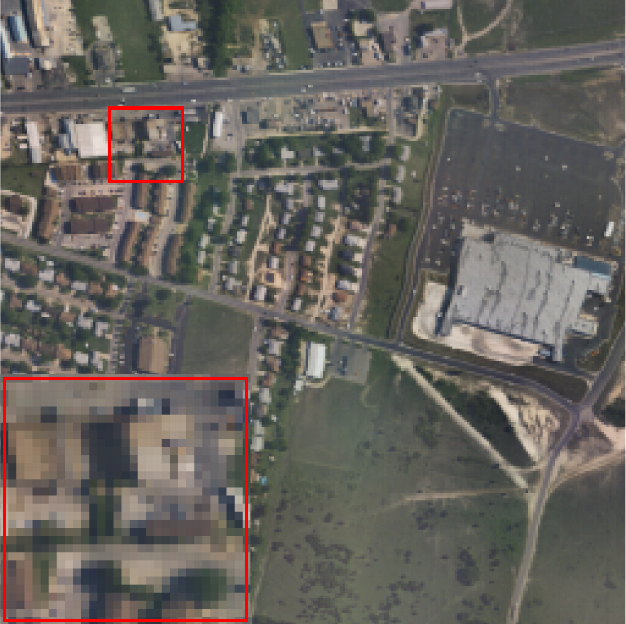}    &
		\includegraphics[width=\linewidth]{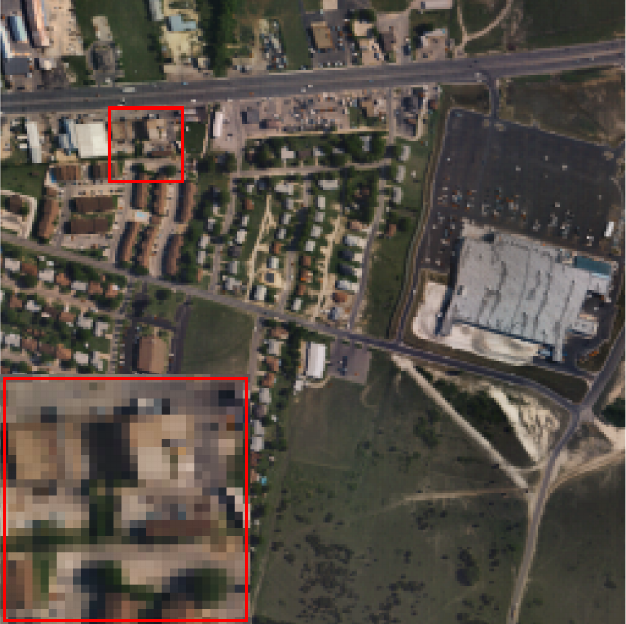}   &
		\includegraphics[width=\linewidth]{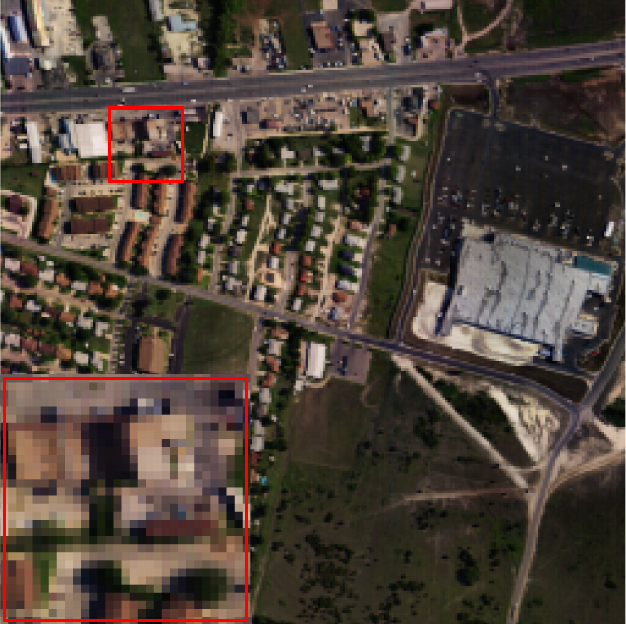}  \\
		\includegraphics[width=\linewidth]{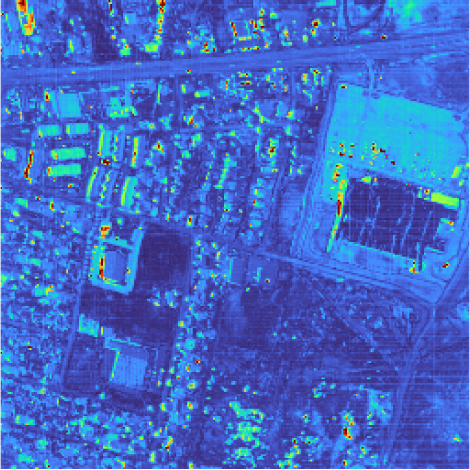}  &
		\includegraphics[width=\linewidth]{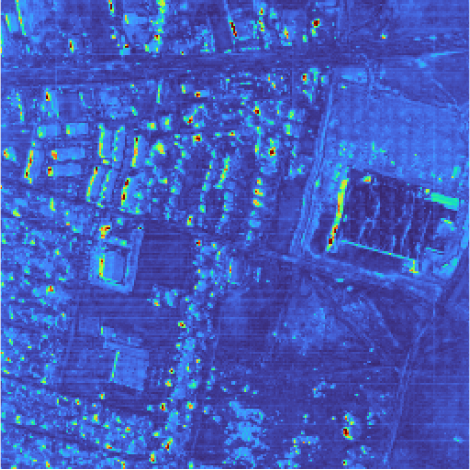}  &
		\includegraphics[width=\linewidth]{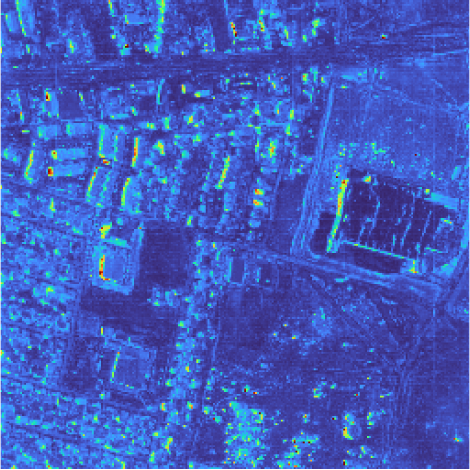}    &
		\includegraphics[width=\linewidth]{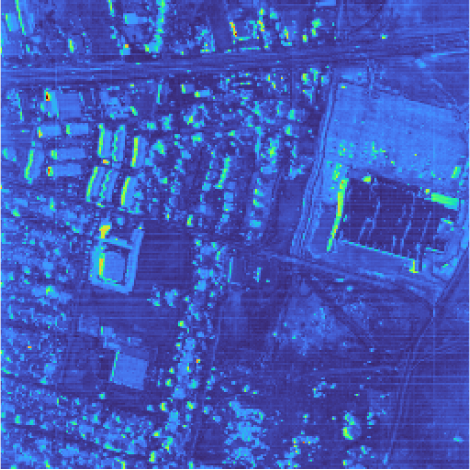}   &
		\includegraphics[width=\linewidth]{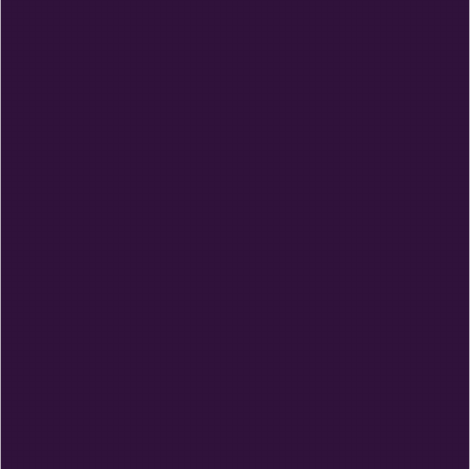} \\
		\multicolumn{1}{c}{\footnotesize{ASLA}}
		&\multicolumn{1}{c}{\footnotesize{ZSL}}
		& \multicolumn{1}{c}{\footnotesize{GTNN}}
		& \multicolumn{1}{c}{\footnotesize{CMlpTR}}
		& \multicolumn{1}{c}{\footnotesize{GT}}\\
		\multicolumn{5}{c}{\includegraphics[width=0.5\linewidth]{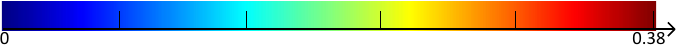}} 
	\end{tabular}
	\caption{\label{fig:URBAN visualization 1} Non-blind fusion results and error maps on the URBAN dataset. {Pseudo-color is composed of bands 40, 30 and 10.} Error maps are calculated by the pixel-wise SAM.}
\end{figure}
\subsection{Datasets Description}
Three hyperspectral datasets, \textit{i.e.}, URBAN\footnote{\url{https://www.erdc.usace.army.mil/Media/Fact-Sheets/Fact-Sheet-Article-View/Article/610433/hypercube/\#}}, Houston\footnote{\url{https://hyperspectral.ee.uh.edu/?page\_id=459}}, and Washington DC (WDC)\footnote{\url{https://github.com/liangjiandeng/HyperPanCollection}} are employed. The URBAN dataset comprises a $210$-band HSI with $307\times 307$ pixels at $2$-meter spatial resolution, covering the spectral range of $400$-$2500$nm at $10$nm intervals. After removing bands $1$-$4$, $76, 87, 101$-$111, 136$-$153$, and $198$-$210$ due to water vapor and atmospheric interference, we extracted a $256\times256\times162$ HSSI from the upper-left corner of the processed data.
\begin{figure}[htbp!]
	\centering
	\setlength{\tabcolsep}{0.2mm}
	\begin{tabular}{m{0.2\linewidth}m{0.2\linewidth}m{0.2\linewidth}m{0.2\linewidth}m{0.2\linewidth}}
		\includegraphics[width=\linewidth]{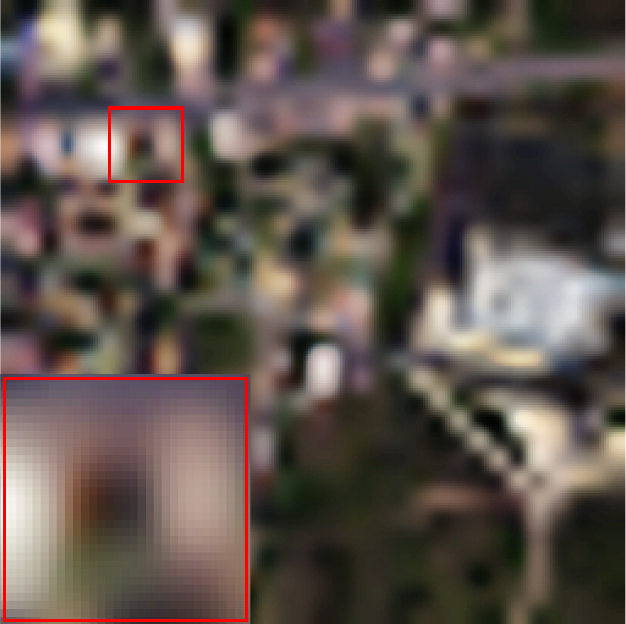}&
		\includegraphics[width=\linewidth]{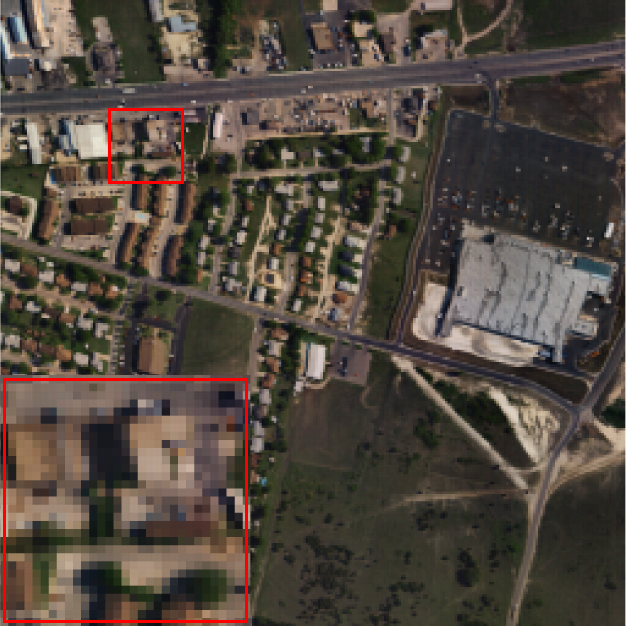}  &
		\includegraphics[width=\linewidth]{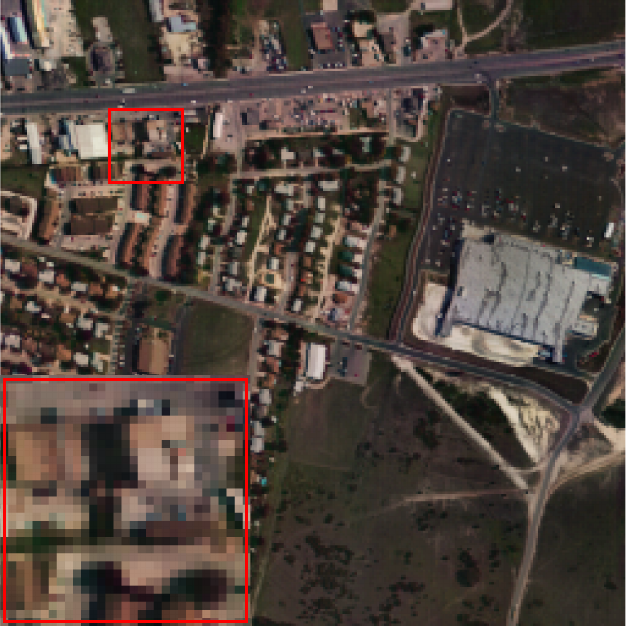}    &
		\includegraphics[width=\linewidth]{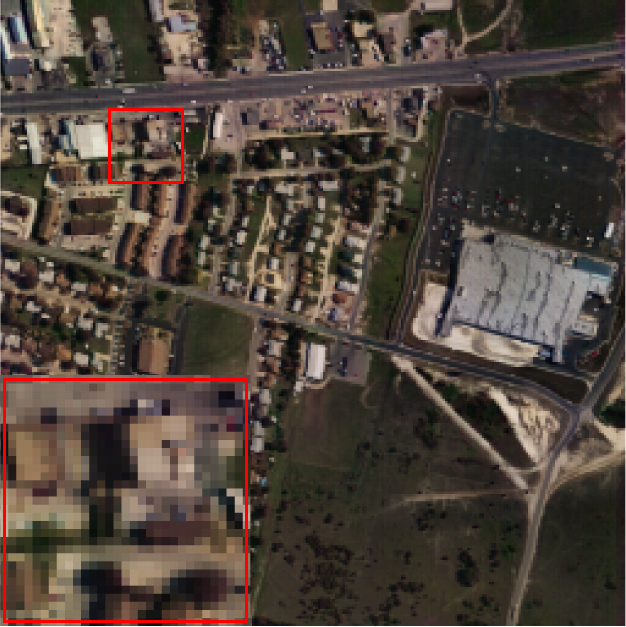}   &
		\includegraphics[width=\linewidth]{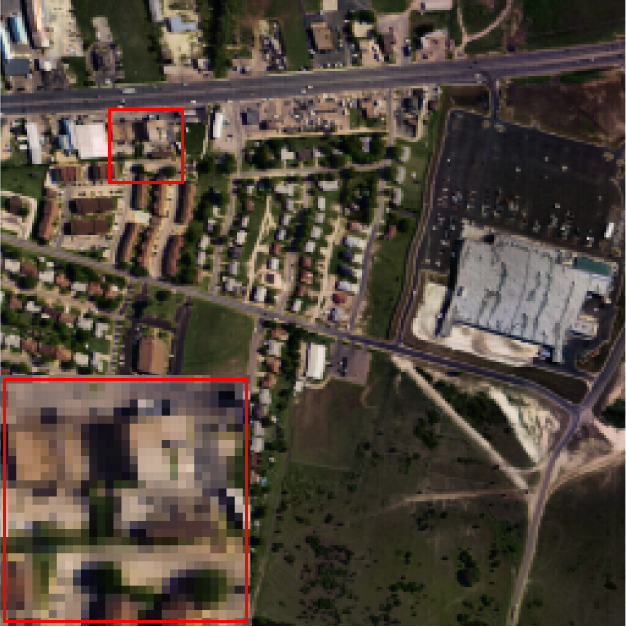}\\
		\includegraphics[width=\linewidth]{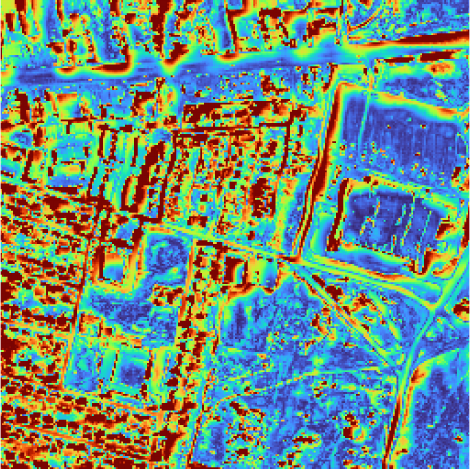}&
		\includegraphics[width=\linewidth]{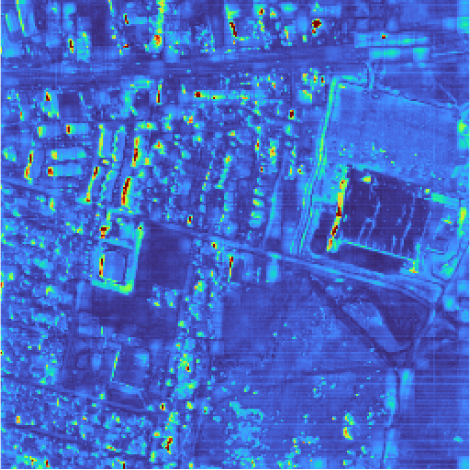}  &
		\includegraphics[width=\linewidth]{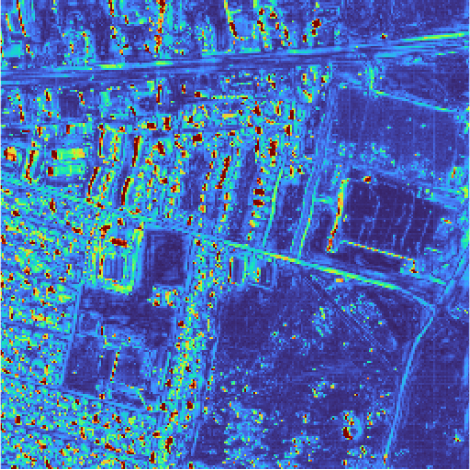}    &
		\includegraphics[width=\linewidth]{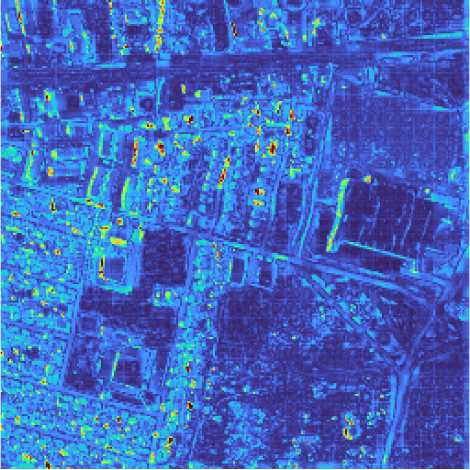}   &
		\includegraphics[width=\linewidth]{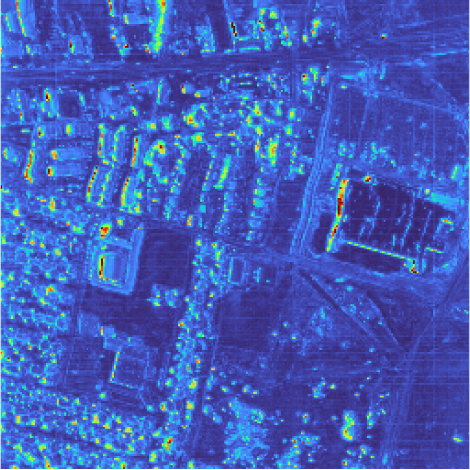} \\
		\multicolumn{1}{c}{\footnotesize{Bicubic}}
		&\multicolumn{1}{c}{\footnotesize{Hysure}}
		& \multicolumn{1}{c}{\footnotesize{LTTR}}
		& \multicolumn{1}{c}{\footnotesize{LRTA}}
		& \multicolumn{1}{c}{\footnotesize{SURE}}\\
		\includegraphics[width=\linewidth]{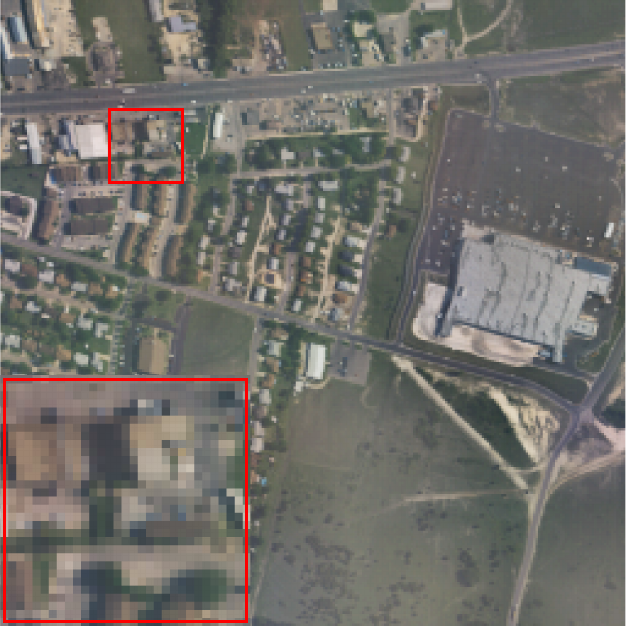}  &
		\includegraphics[width=\linewidth]{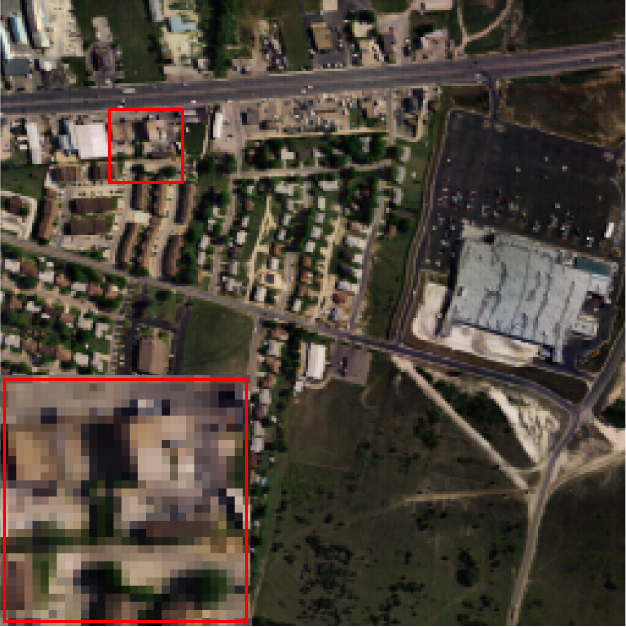}  &
		\includegraphics[width=\linewidth]{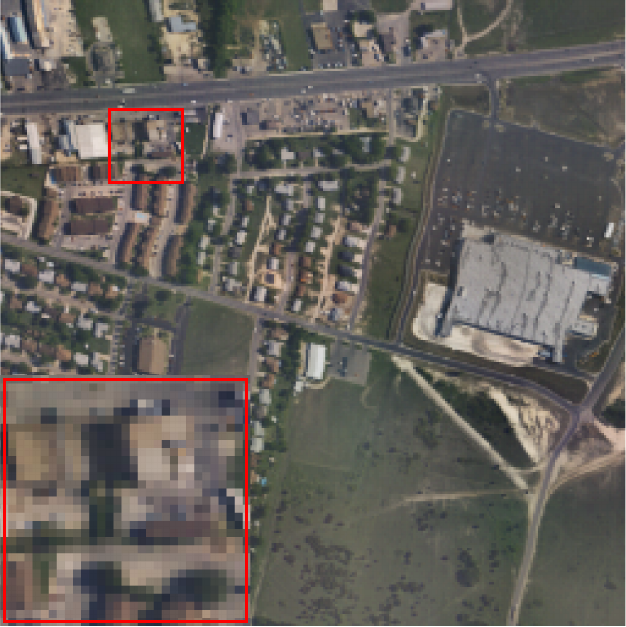}    &
		\includegraphics[width=\linewidth]{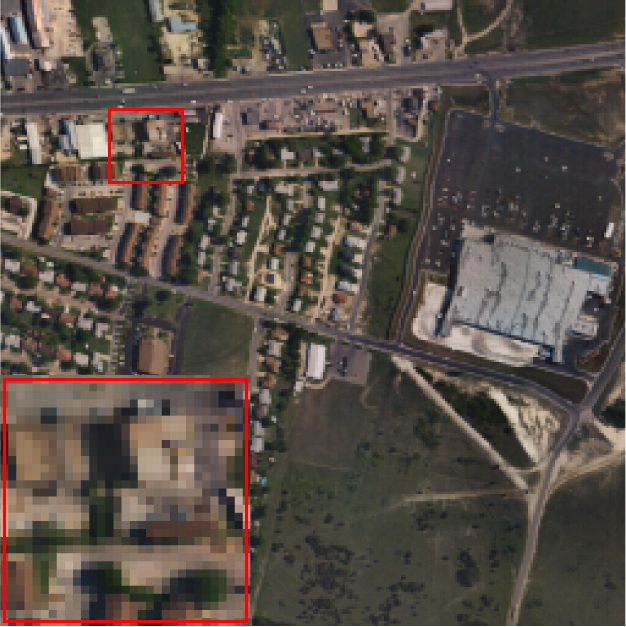}   &
		\includegraphics[width=\linewidth]{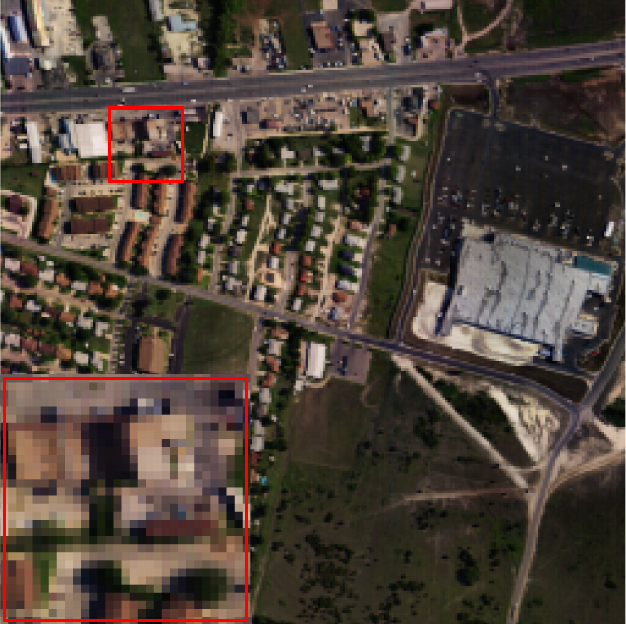}  \\\includegraphics[width=\linewidth]{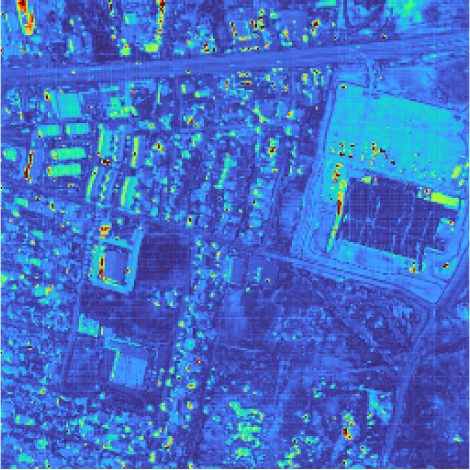}  &
		\includegraphics[width=\linewidth]{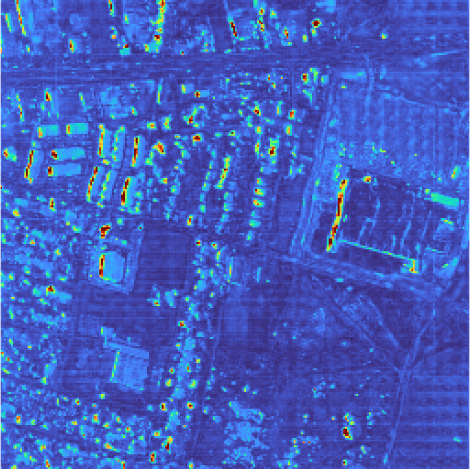}  &
		\includegraphics[width=\linewidth]{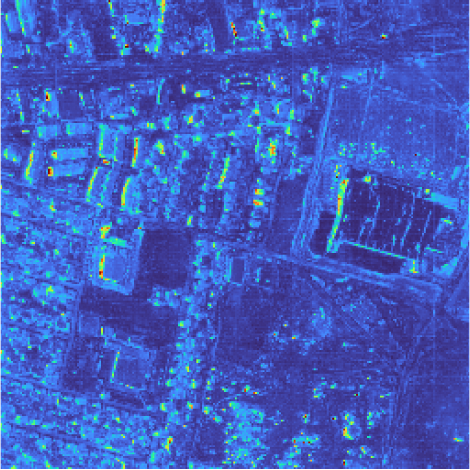}    &
		\includegraphics[width=\linewidth]{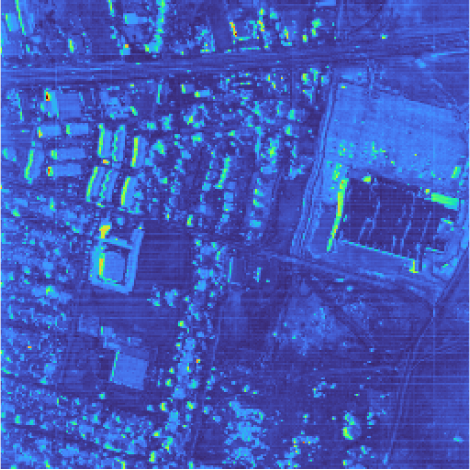}   &
		\includegraphics[width=\linewidth]{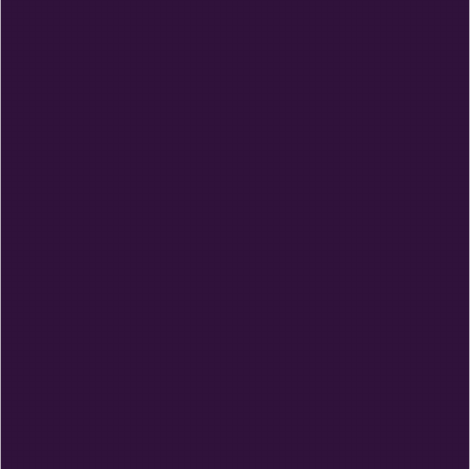} \\
		\multicolumn{1}{c}{\footnotesize{ASLA}}
		&\multicolumn{1}{c}{\footnotesize{ZSL}}
		& \multicolumn{1}{c}{\footnotesize{GTNN}}
		& \multicolumn{1}{c}{\footnotesize{CMlpTR}}
		& \multicolumn{1}{c}{\footnotesize{GT}}\\
		\multicolumn{5}{c}{\includegraphics[width=0.5\linewidth]{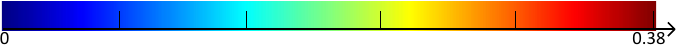}} 
	\end{tabular}
	\caption{\label{fig:URBAN visualization 2} Blind fusion results and error maps on the URBAN dataset. {Pseudo-color is composed of bands 40, 30 and 10.} Error maps are calculated by the pixel-wise SAM.}
\end{figure}

The Houston dataset consists of an HSI with $349\times1905$ spatial pixels and $144$ spectral bands at $2.5$-meter resolution, covering the $380$-$1050$nm spectral range. With zero-valued pixels being discarded, we generated a $322\times1903\times144$ image cube, from which we crop off a $256\times256\times144$ subscene to serve as the HSSI. Besides, since the original image appears dark, we performed Gamma-calibration on the generated HSSI $\mathcal{Z}$ to obtain $\mathcal{Z}^{0.7}$ as the ground-truth (GT) HSSI, where the power was calculated element-wisely.
\begin{table}[htbp]
	\centering
	\renewcommand{\arraystretch}{1}
	\tabcolsep=0.5mm
	\caption{Numerical performance on the URBAN dataset. Best results are in boldface.}
	\resizebox{\linewidth}{!}{\begin{tabular}{c|cccc|cccc}
			\toprule
			\hline
			\multirow{3}{*}{\textbf{Methods}}&\multicolumn{8}{c}{\textbf{Setup}} \\\cline{2-9}\quad& \multicolumn{4}{c|}{\textbf{Non-blind}} &\multicolumn{4}{c}{\textbf{Blind}}\\\cline{2-9} \quad
			&  PSNR$\uparrow$                 & {ERGAS}$\downarrow$    & SAM$\downarrow$    & SSIM$\uparrow$  &    PSNR$\uparrow$                 & {ERGAS}$\downarrow$    & SAM$\downarrow$    & SSIM$\uparrow$   \\\hline {Bicubic}            &  22.4646   &  5.8732   &   10.6064  &   0.4115 &  22.4646   &  5.8732   &   10.6064  &   0.4115 \\
			Hysure            &  42.9189   &  0.6890   &   2.9550  &   0.9920 &  42.7903   &  0.6778   &   2.9662  &   0.9920 \\ LTTR 
			&         43.9124          &  0.8584   &   3.7677  &  0.9790   &         43.7321          &  0.8516   &   3.7948  &  0.9789    \\
			LRTA             &  44.5726   & 0.7171    & 2.8548    &   0.9873    &  44.3073   & 0.7187    & 2.8630    &   0.9873 \\ SURE           &  42.9270   &  0.6169   &   2.4358  &  0.9915 &  42.7237   &  0.6326   &   2.5093  &  0.9914  \\       ASLA          &  44.6256   &  0.6703   &  2.8596   & 0.9917  &  44.3815   &  0.6803   &  2.7664   & 0.9920  \\        ZSL           &   43.1704  &  0.6226   &  2.4037   &  0.9927   &   42.9792  &  0.6431   &  2.5551   &  0.9921 \\        GTNN         &  44.6832  &  0.5864   &  2.3858   &  0.9914  &  44.5186   &   0.5963  &   2.3980  &  0.9914  \\     CMlpTR            &  \textbf{45.2351}   &   \textbf{0.5588}  &  \textbf{2.2375}   &  \textbf{0.9947}    &  \textbf{45.1130}   &   \textbf{0.5610}  &  \textbf{2.2425}   &  \textbf{0.9947}  \\
			\hline
			\bottomrule
	\end{tabular}}
	\label{tab:URBAN metrics}
\end{table}

The WDC dataset is acquired via the same sensor as URBAN. Hence, they have the same spatial and spectral resolutions and quality. Thus, we preprocessed the WDC dataset in the same way as URBAN to obtain a $256\times256\times162$ HSSI.

For each HSSI, we applied a $9\times 9$ Gaussian kernel with a standard deviation of $3.3973$ to blur it, followed by spatial downsampling by a factor of eight, resulting in an HSI of spatial size $32\times32$. For URBAN and WDC, the MSI of size $256\times256\times 6$ was generated by averaging HSSI bands within six spectral ranges ($450$-$520$nm, $520$–$600$nm, $630$–$690$nm, $760$–$900$nm, $1550$–$1750$nm, and $2080$–$2350$nm) to match the spectral response of the USGS/NASA Landsat 7 satellite \cite{10904006}. As for Houston, the corresponding $256\times256\times4$ MSI was generated via an IKONOS-like
reflectance spectral response filter \cite{10154463}.
\begin{figure}[htbp!]
	\centering
	\setlength{\tabcolsep}{0.2mm}
	\begin{tabular}{m{0.2\linewidth}m{0.2\linewidth}m{0.2\linewidth}m{0.2\linewidth}m{0.2\linewidth}}
		\includegraphics[width=\linewidth]{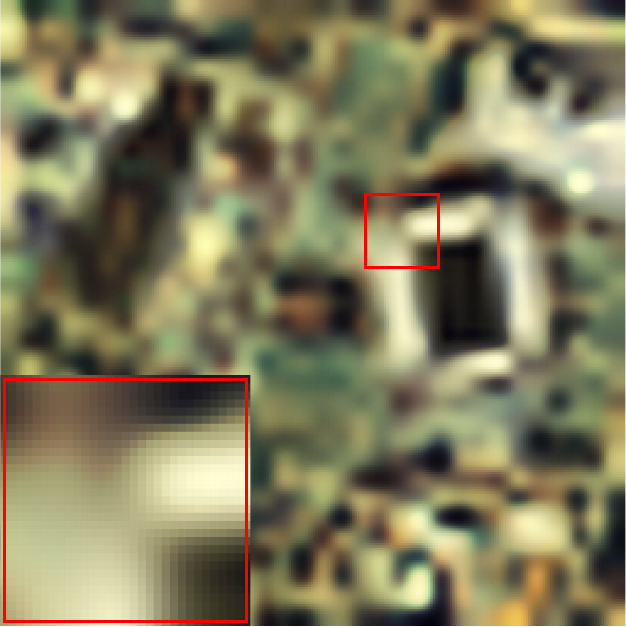}&
		\includegraphics[width=\linewidth]{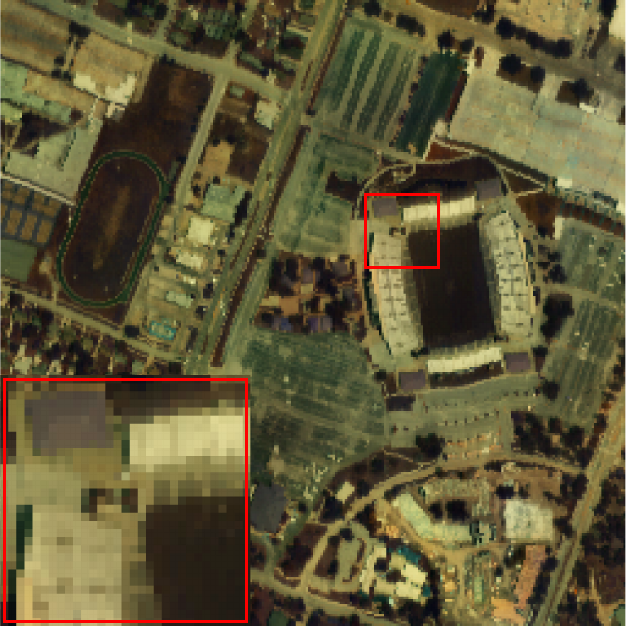}  &
		\includegraphics[width=\linewidth]{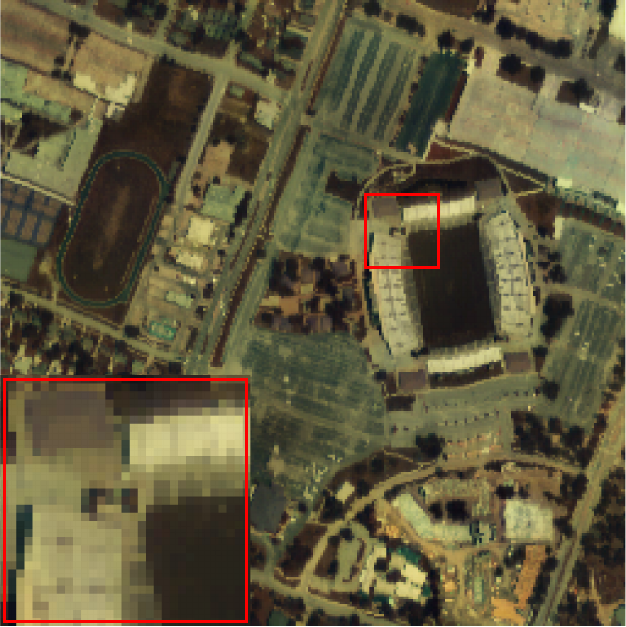}    &
		\includegraphics[width=\linewidth]{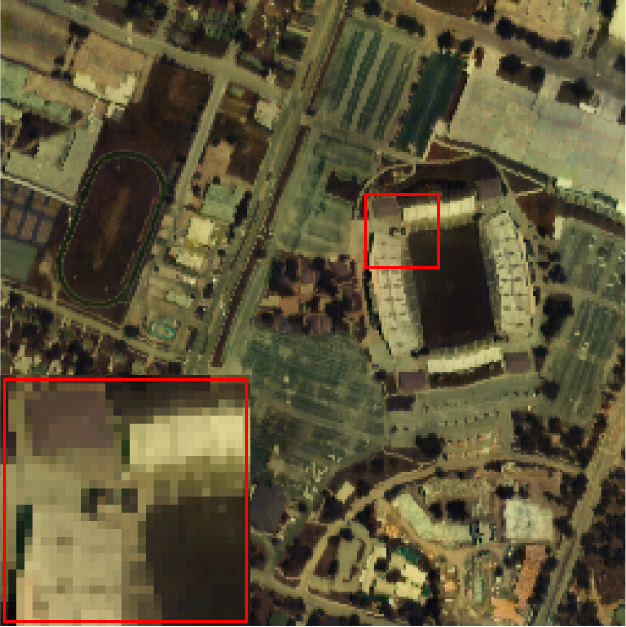}   &
		\includegraphics[width=\linewidth]{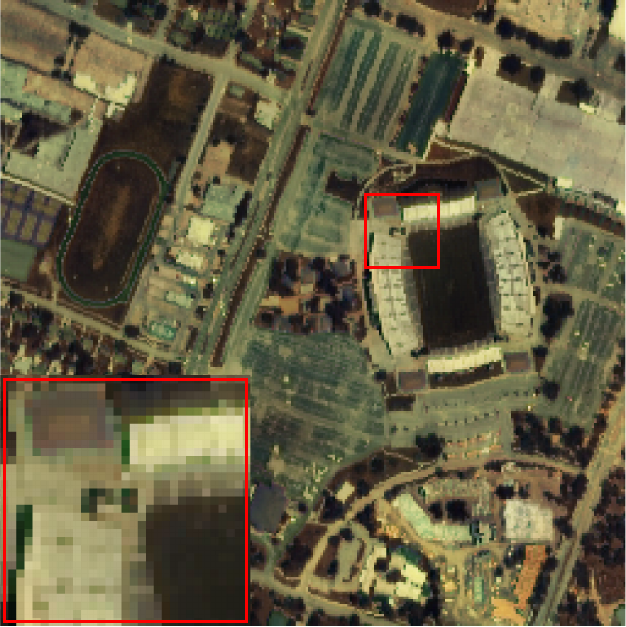}\\
		\includegraphics[width=\linewidth]{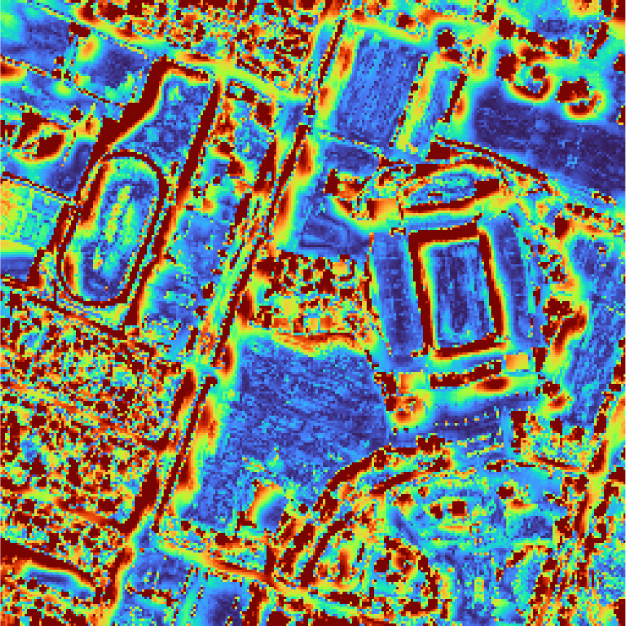}&
		\includegraphics[width=\linewidth]{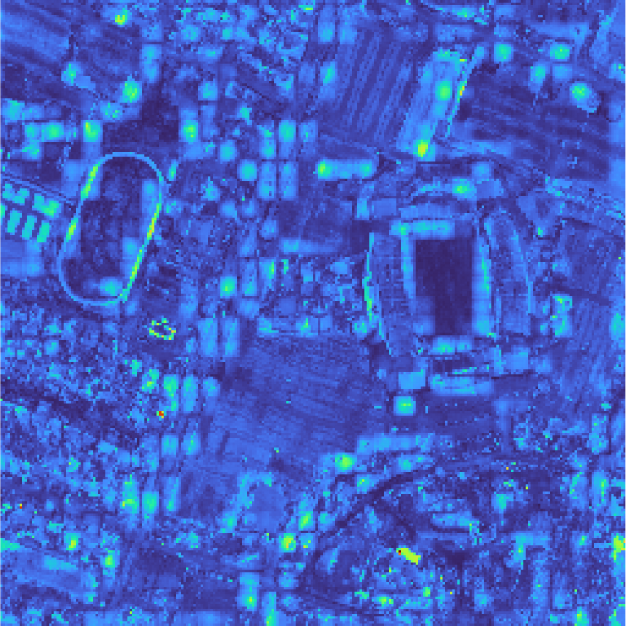}  &
		\includegraphics[width=\linewidth]{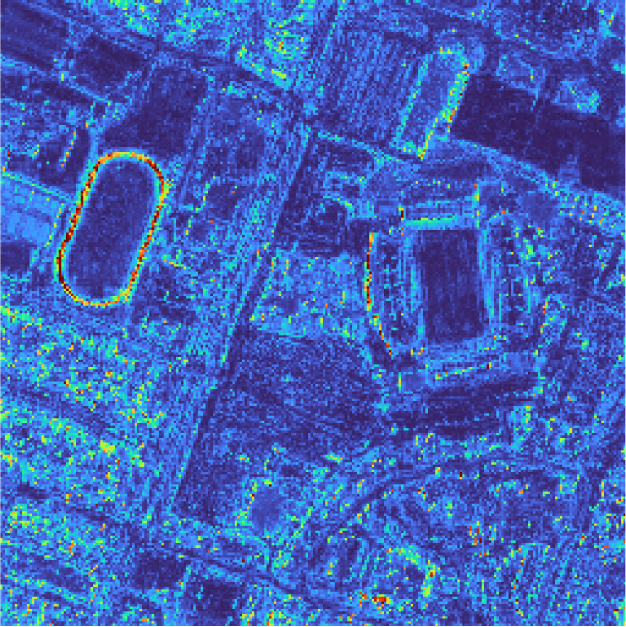}    &
		\includegraphics[width=\linewidth]{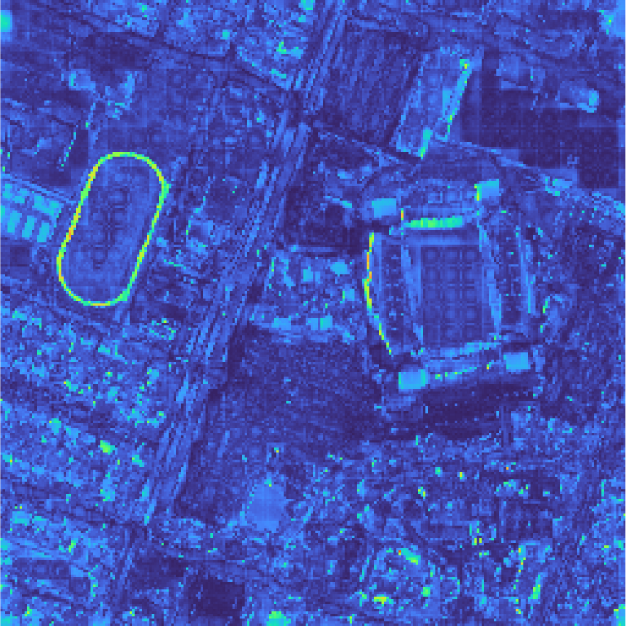}   &
		\includegraphics[width=\linewidth]{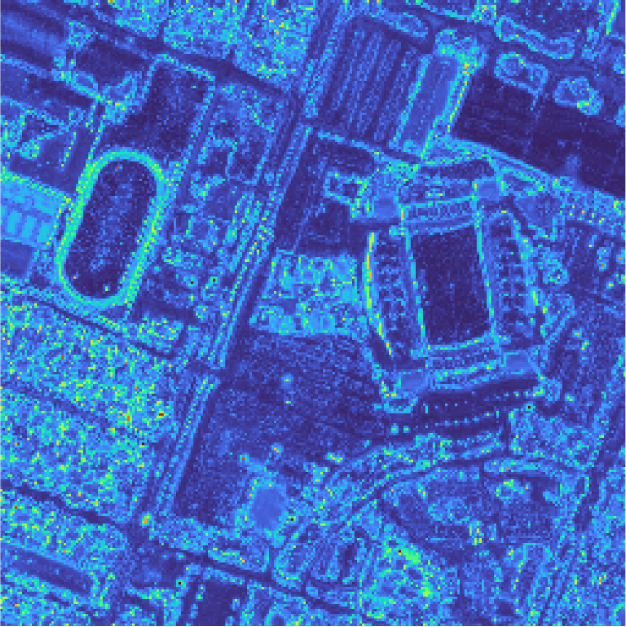}\\
		\multicolumn{1}{c}{\footnotesize{Bicubic}}
		&\multicolumn{1}{c}{\footnotesize{Hysure}}
		& \multicolumn{1}{c}{\footnotesize{LTTR}}
		& \multicolumn{1}{c}{\footnotesize{LRTA}}
		& \multicolumn{1}{c}{\footnotesize{SURE}}\\
		\includegraphics[width=\linewidth]{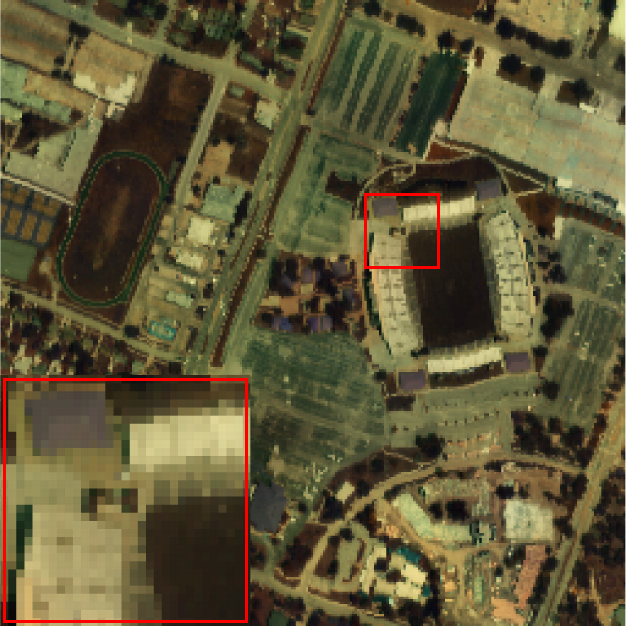}  &
		\includegraphics[width=\linewidth]{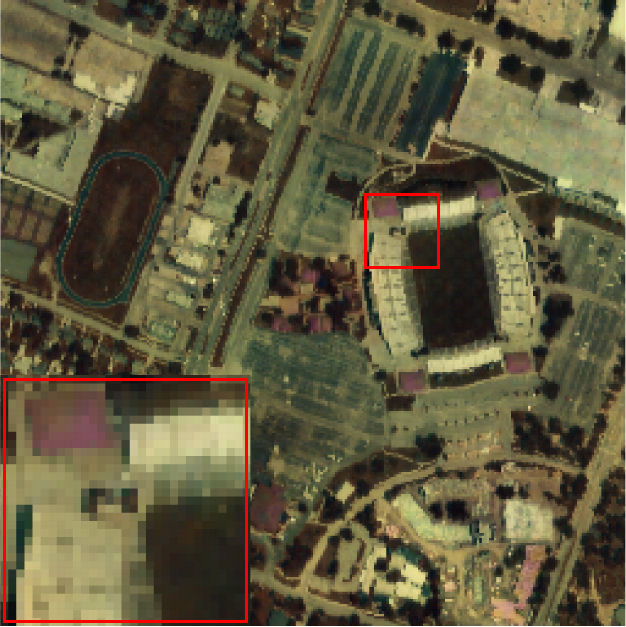}  &
		\includegraphics[width=\linewidth]{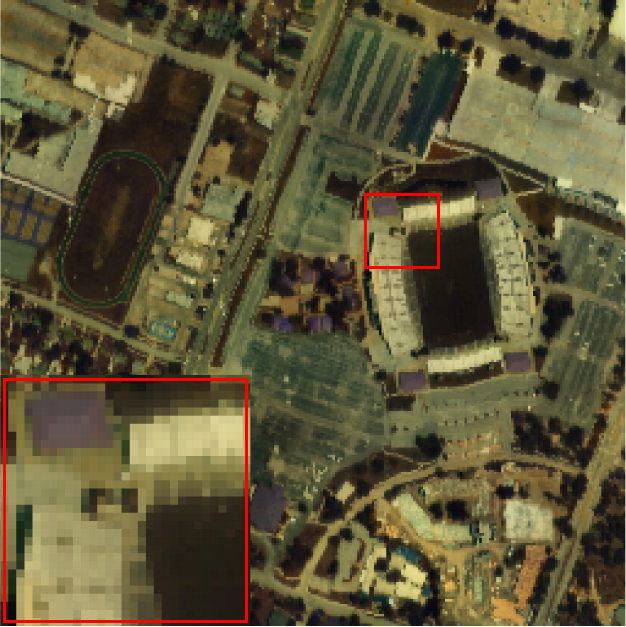}    &
		\includegraphics[width=\linewidth]{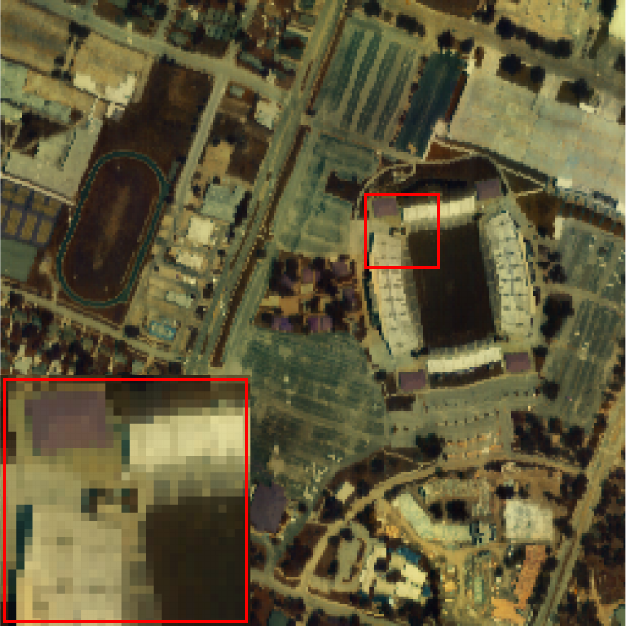}   &
		\includegraphics[width=\linewidth]{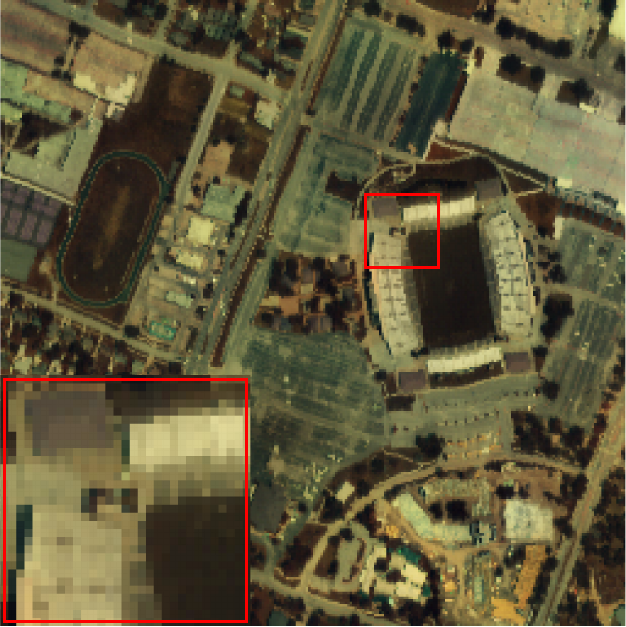}  \\
		\includegraphics[width=\linewidth]{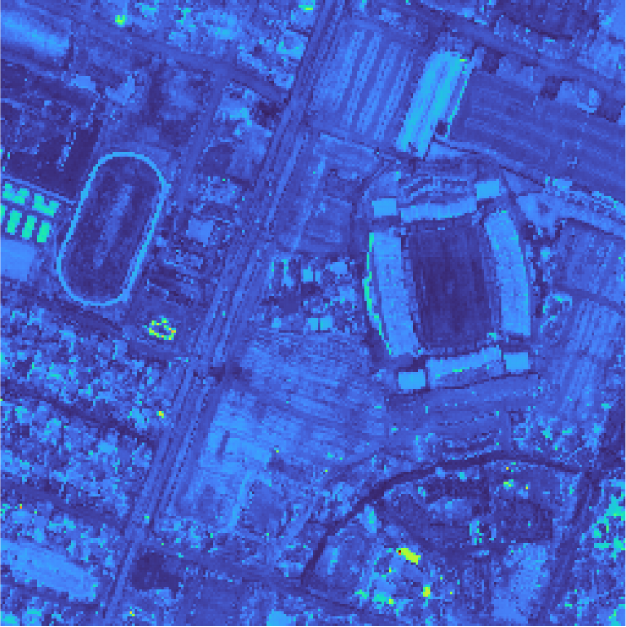}  &
		\includegraphics[width=\linewidth]{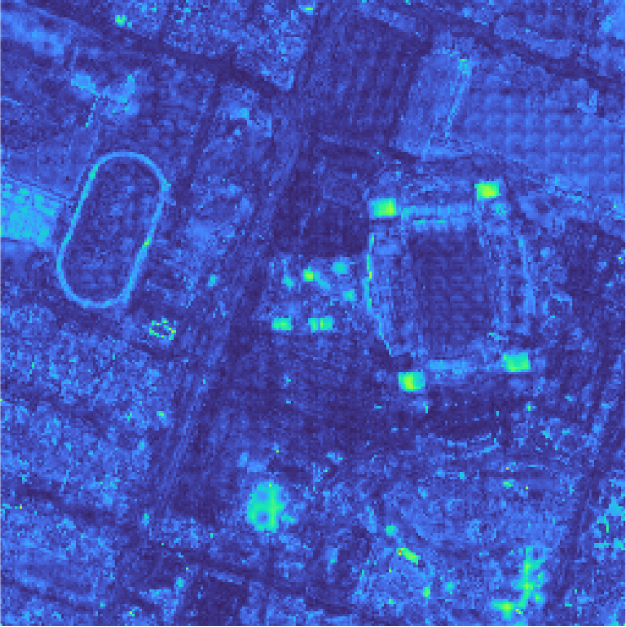}  &
		\includegraphics[width=\linewidth]{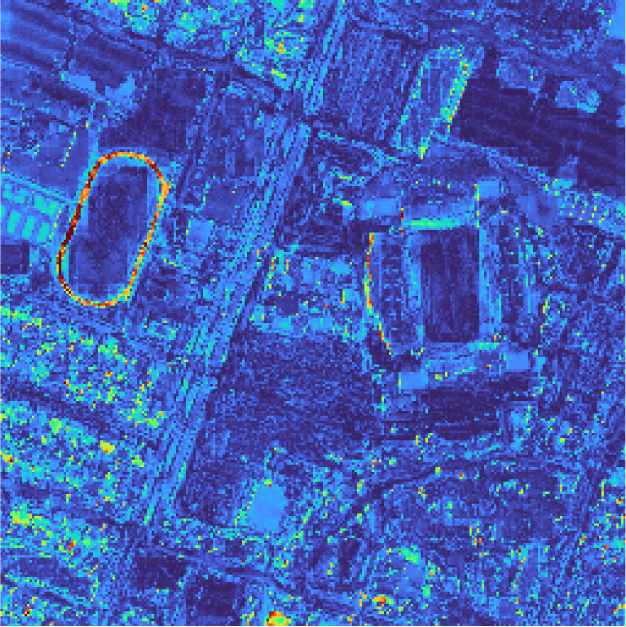}    &
		\includegraphics[width=\linewidth]{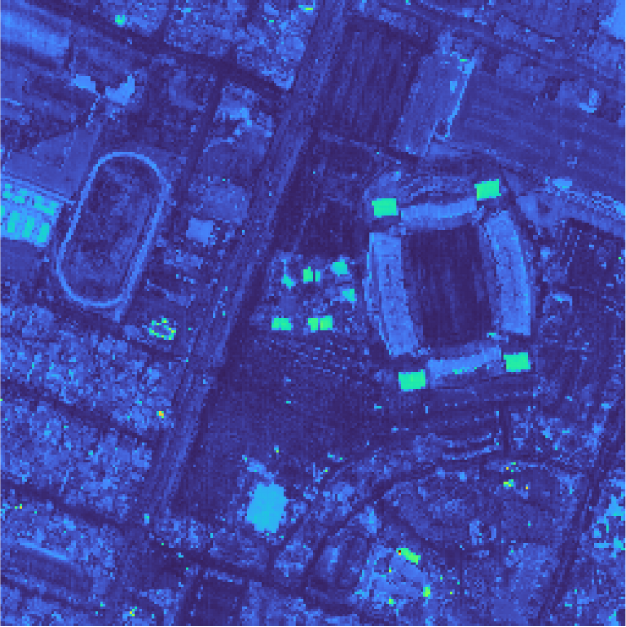}   &
		\includegraphics[width=\linewidth]{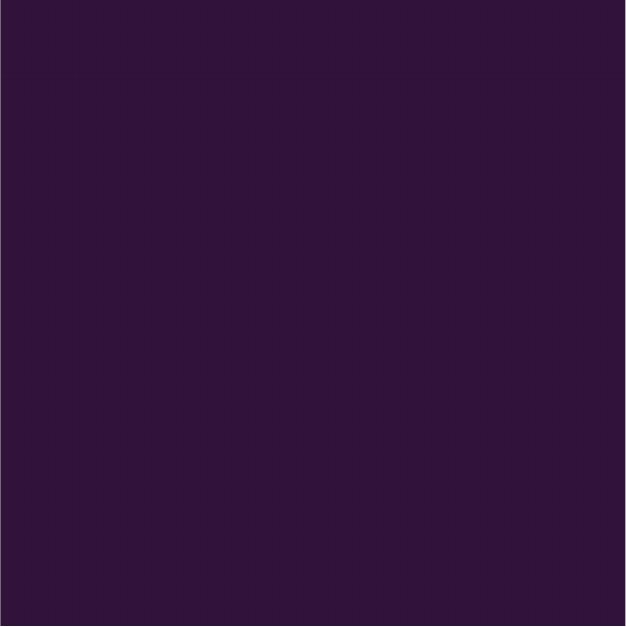} \\
		\multicolumn{1}{c}{\footnotesize{ASLA}}
		&\multicolumn{1}{c}{\footnotesize{ZSL}}
		& \multicolumn{1}{c}{\footnotesize{GTNN}}
		& \multicolumn{1}{c}{\footnotesize{CMlpTR}}
		& \multicolumn{1}{c}{\footnotesize{GT}}\\
		\multicolumn{5}{c}{\includegraphics[width=0.5\linewidth]{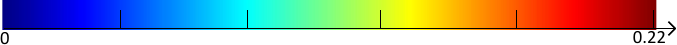}} 
	\end{tabular}
	\caption{\label{fig:Houston visualization 1} Non-blind fusion results and error maps on the Houston dataset. {Pseudo-color is composed of bands 45, 25 and 10.} Error maps are calculated by the pixel-wise SAM.}
\end{figure}
\begin{figure}[htbp!]
	\centering
	\setlength{\tabcolsep}{0.2mm}
	\begin{tabular}{m{0.2\linewidth}m{0.2\linewidth}m{0.2\linewidth}m{0.2\linewidth}m{0.2\linewidth}}
		\includegraphics[width=\linewidth]{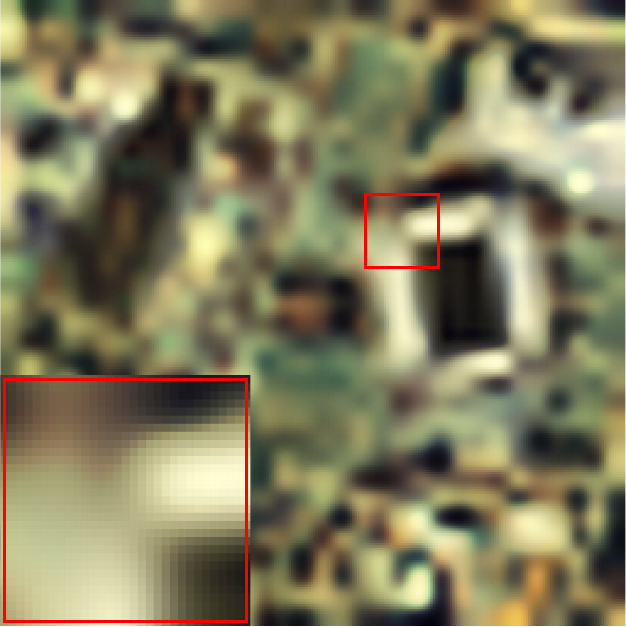}&
		\includegraphics[width=\linewidth]{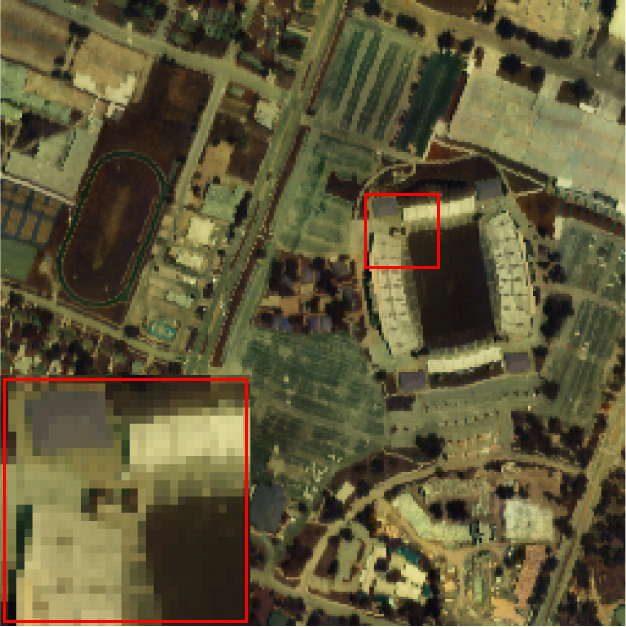}  &
		\includegraphics[width=\linewidth]{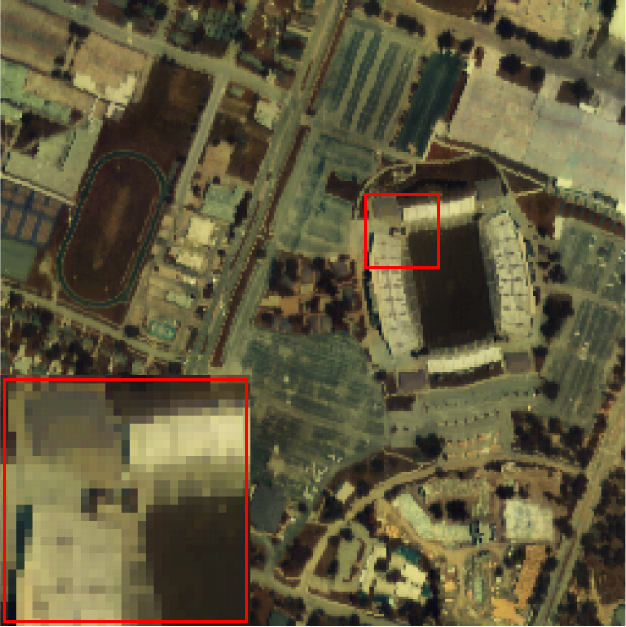}    &
		\includegraphics[width=\linewidth]{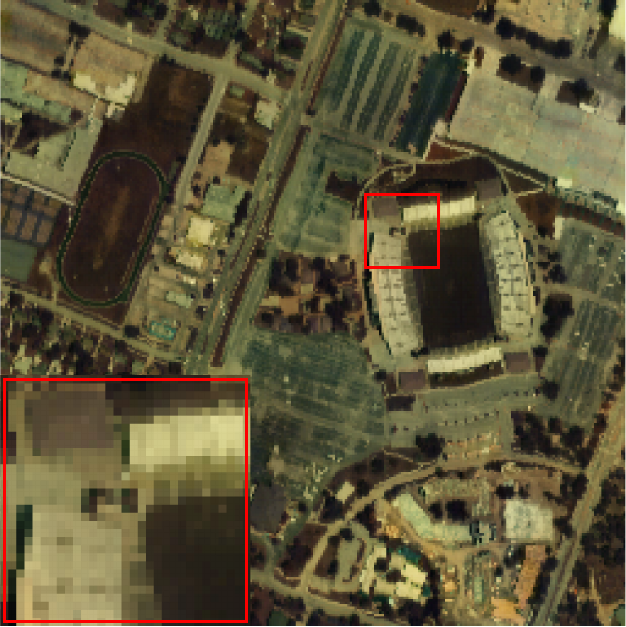}   &
		\includegraphics[width=\linewidth]{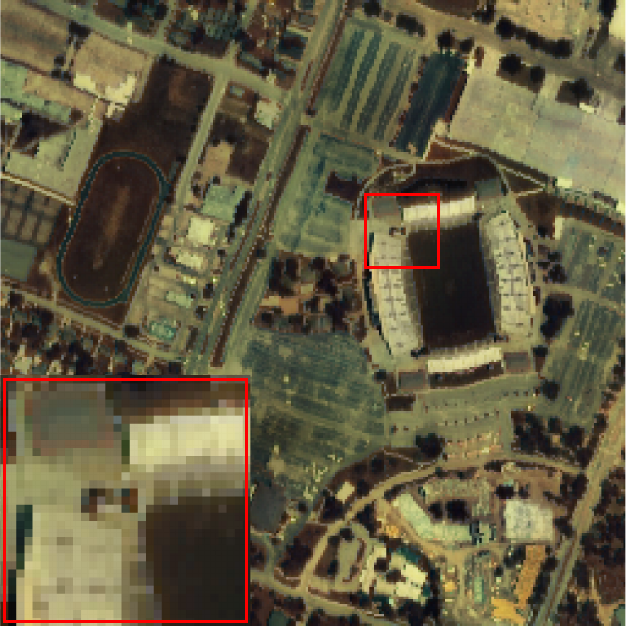}\\
		\includegraphics[width=\linewidth]{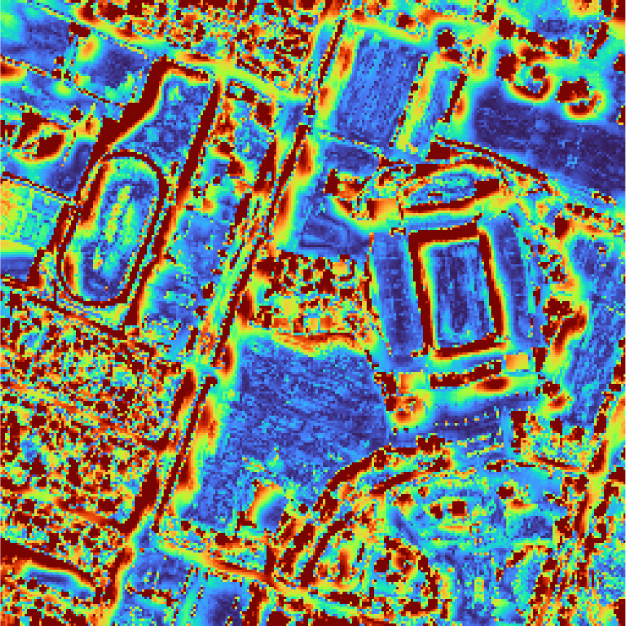}&
		\includegraphics[width=\linewidth]{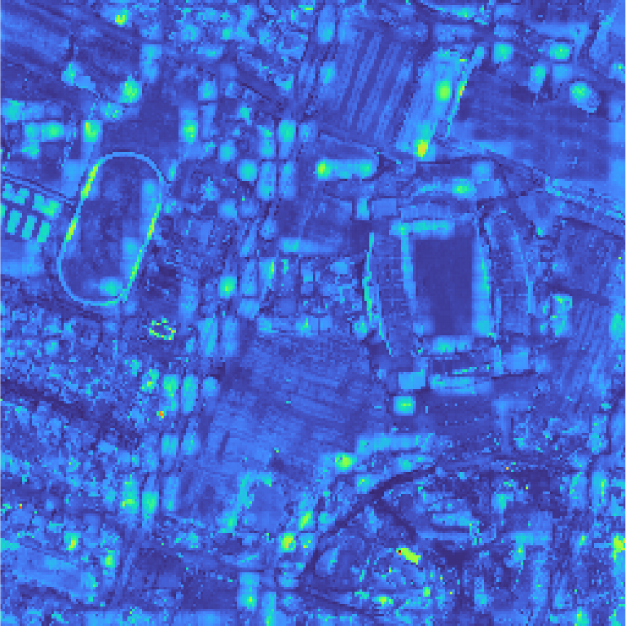}  &
		\includegraphics[width=\linewidth]{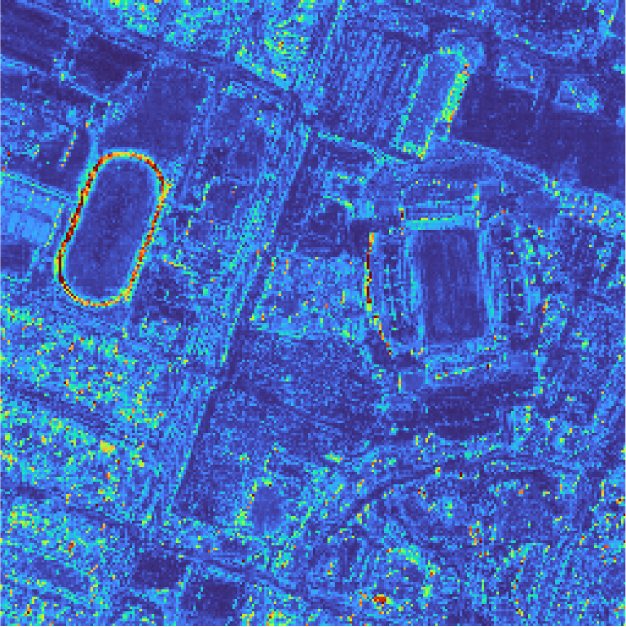}    &
		\includegraphics[width=\linewidth]{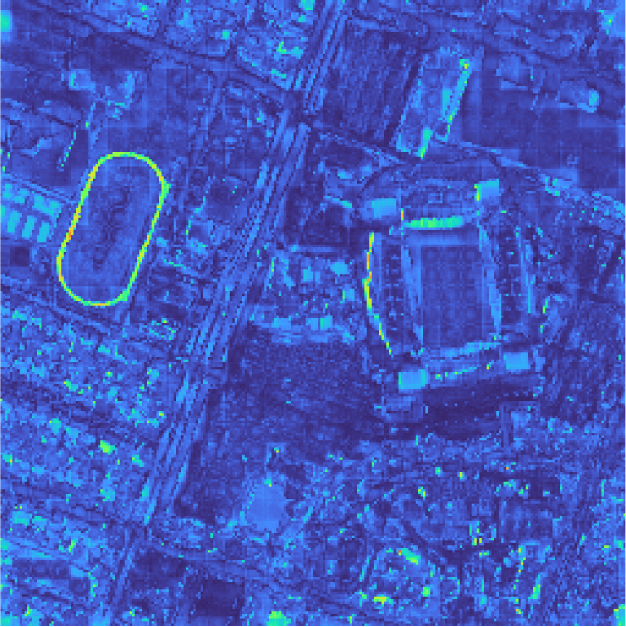}   &
		\includegraphics[width=\linewidth]{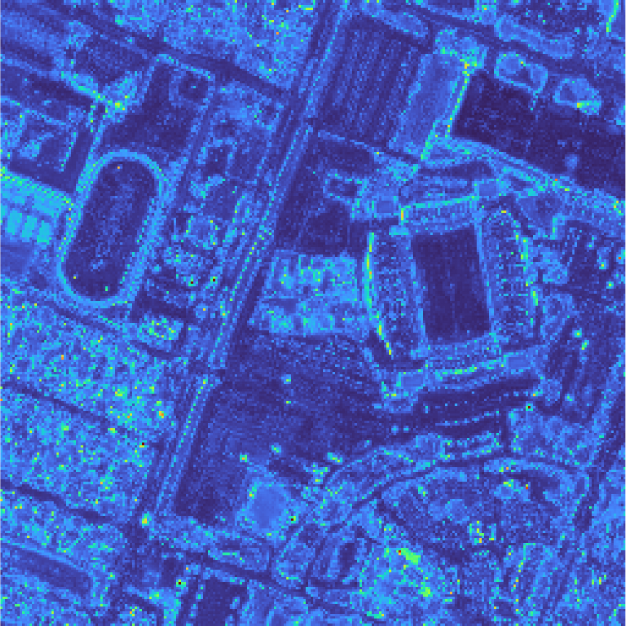}\\
		\multicolumn{1}{c}{\footnotesize{Bicubic}}
		&\multicolumn{1}{c}{\footnotesize{Hysure}}
		& \multicolumn{1}{c}{\footnotesize{LTTR}}
		& \multicolumn{1}{c}{\footnotesize{LRTA}}
		& \multicolumn{1}{c}{\footnotesize{SURE}}\\
		\includegraphics[width=\linewidth]{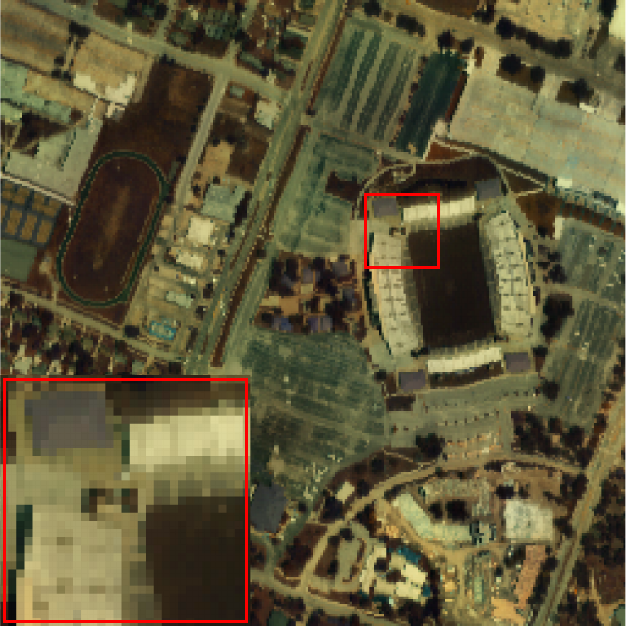}  &
		\includegraphics[width=\linewidth]{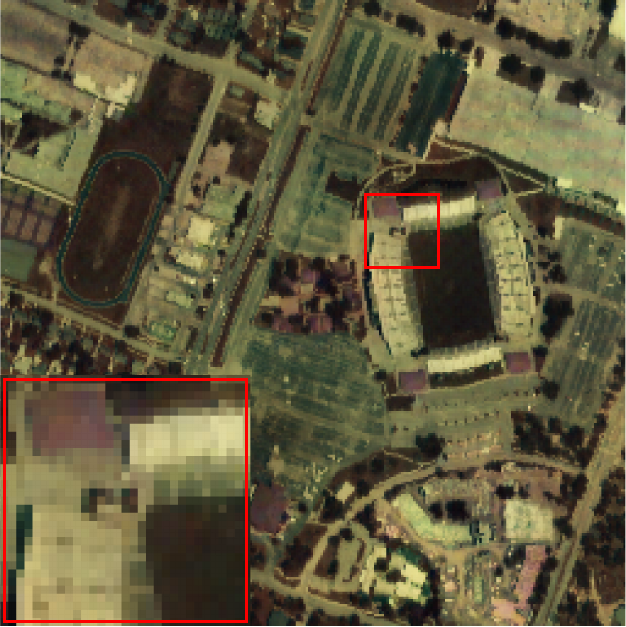}  &
		\includegraphics[width=\linewidth]{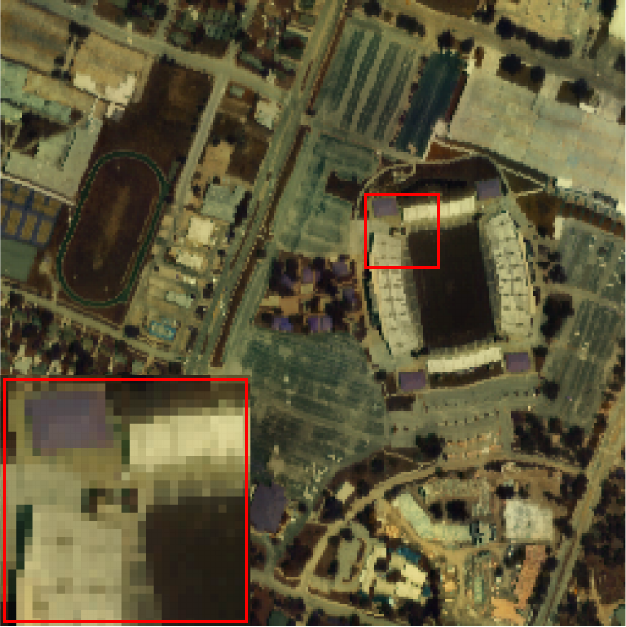}    &
		\includegraphics[width=\linewidth]{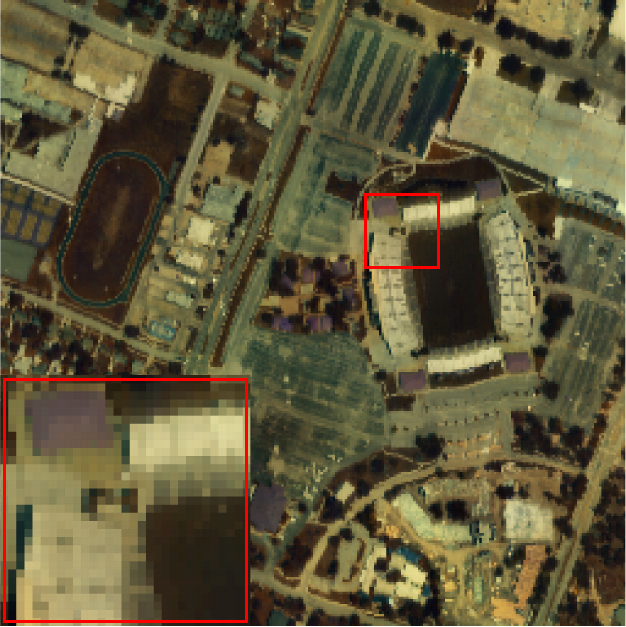}   &
		\includegraphics[width=\linewidth]{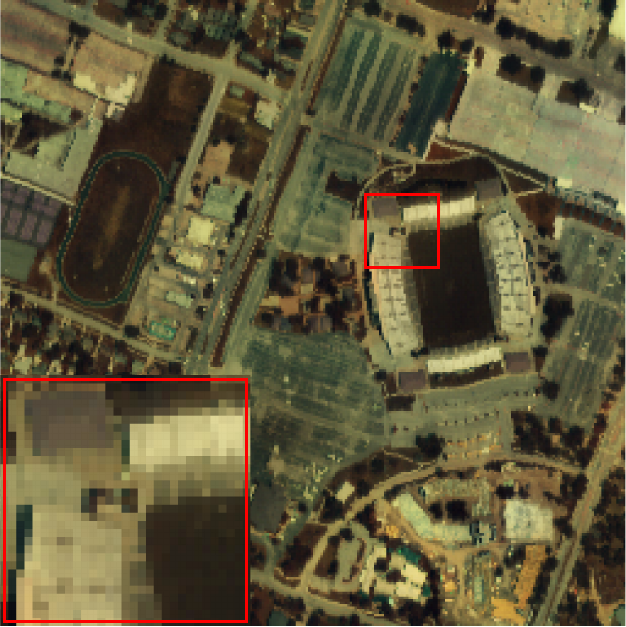}  \\
		\includegraphics[width=\linewidth]{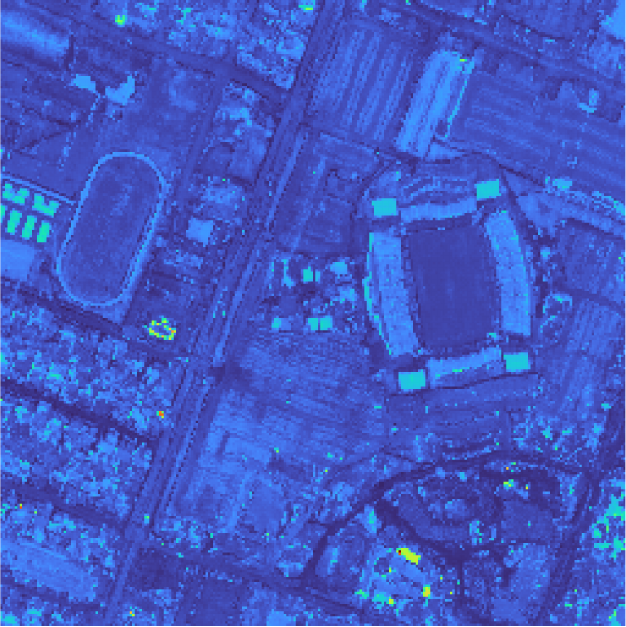}  &
		\includegraphics[width=\linewidth]{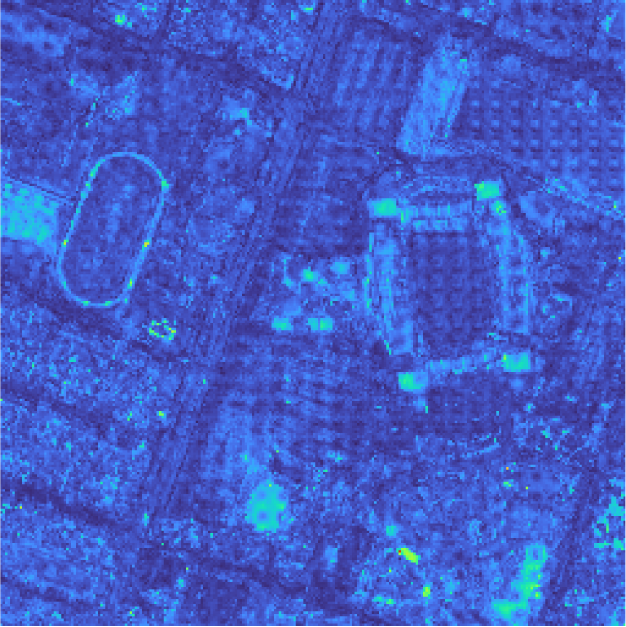}  &
		\includegraphics[width=\linewidth]{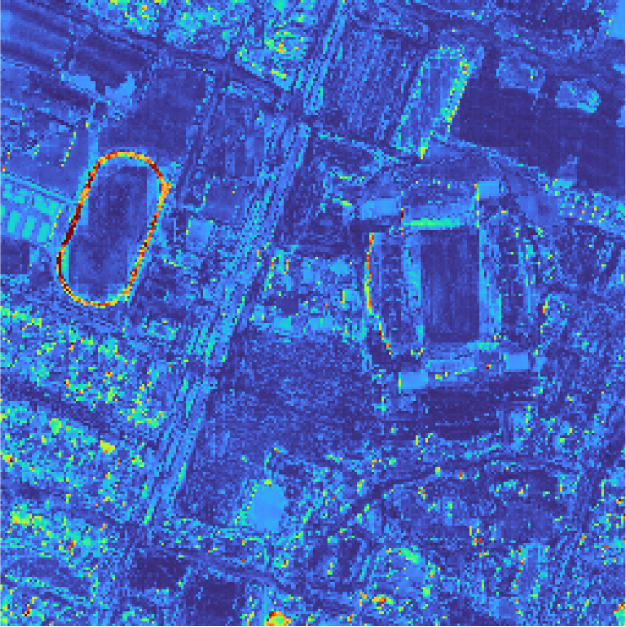}    &
		\includegraphics[width=\linewidth]{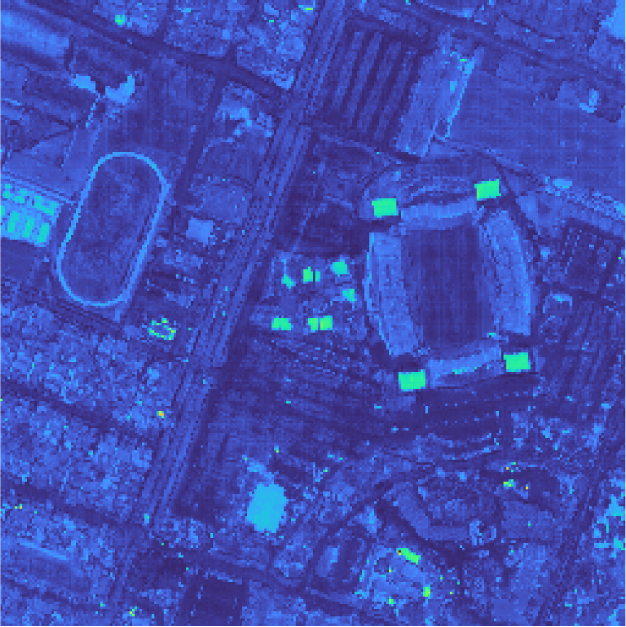}   &
		\includegraphics[width=\linewidth]{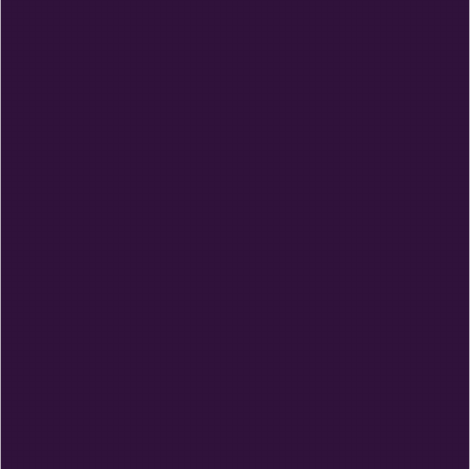} \\
		\multicolumn{1}{c}{\footnotesize{ASLA}}
		&\multicolumn{1}{c}{\footnotesize{ZSL}}
		& \multicolumn{1}{c}{\footnotesize{GTNN}}
		& \multicolumn{1}{c}{\footnotesize{CMlpTR}}
		& \multicolumn{1}{c}{\footnotesize{GT}}\\
		\multicolumn{5}{c}{\includegraphics[width=0.5\linewidth]{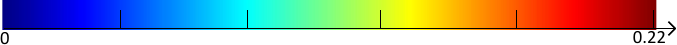}} 
	\end{tabular}
	\caption{\label{fig:Houston visualization 2} Blind fusion results and error maps on the Houston dataset. {Pseudo-color is composed of bands 45, 25 and 10.} Error maps are calculated by the pixel-wise SAM.}
\end{figure}
\subsection{Experimental Results}
\subsubsection{Hyperparameter Sensitivity}
The main hyperparameters include the dimensionality of the spectral subspace, \textit{i.e.}, $R$ that captures the spectral low-rankness of the latent HSSI, and the surrogate parameter $\gamma$ from $\psi(x)=\frac{\log(\gamma x+1)}{\log(\gamma+1)}$. To determine their values, we examine the PSNR sensitivity in non-blind fusion to both $R$ and $\gamma$. The results are visualized in Fig. \ref{fig: Parameter Tuning}. It is clear to see that the proposed model is much more sensitive to $R$ than to $\gamma$. From the distinct peak lines regarding $R$, it is easy to set $R$ to five, four and six for URBAN, Houston and WDC, respectively. As for $\gamma$, though the performance with respect to it remains relatively steady, there exists a black edge in each of the three surfaces at one common position where $\gamma=0.1$. The corresponding PSNR value is higher than its neighbors if $R$ is kept optimal. Thus, throughout these experiments, we set $\gamma$ to $0.1$.
\begin{table}[htbp]
	\centering
	\renewcommand{\arraystretch}{1}
	\tabcolsep=0.5mm
	\caption{Numerical performance on the Houston dataset. Best results are in boldface.}
	\resizebox{\linewidth}{!}{\begin{tabular}{c|cccc|cccc}
			\toprule
			\hline
			\multirow{3}{*}{\textbf{Methods}}&\multicolumn{8}{c}{\textbf{Setup}} \\\cline{2-9}\quad& \multicolumn{4}{c|}{\textbf{Non-blind}} &\multicolumn{4}{c}{\textbf{Blind}}\\\cline{2-9} \quad
			&  PSNR$\uparrow$                 & {ERGAS}$\downarrow$    & SAM$\downarrow$    & SSIM$\uparrow$  &    PSNR$\uparrow$                 & {ERGAS}$\downarrow$    & SAM$\downarrow$    & SSIM$\uparrow$   \\\hline
			{Bicubic}            &  27.9764   &  2.8896   &   6.5494  &   0.6230 &  27.9764   &  2.8896   &   6.5494  &   0.6230\\
			Hysure            &   48.2238   &  0.5736   &   1.4917  &   0.9927 &  46.3990   &  0.5866   &   1.5930  &   0.9922 \\ LTTR 
			&        48.1071          &  0.5894   &   1.6414  &  0.9793   &          46.8301          &  0.5943   &   1.7098  &  0.9793   \\
			LRTA             & 46.7599   & 0.4785    & 1.1938    &   0.9922    & 46.5129   & 0.4790    & 1.3003    &   0.9919  \\ SURE           & 44.3703   &  0.8480   &   1.6479  &  0.9769   &  44.8562   &  0.7868   &   1.4291  &  0.9829   \\       ASLA          &  48.7790   &  0.8554   &  1.5664   & 0.9933 &   47.6438   &  0.7465   &  1.3068   & 0.9943    \\        ZSL           &   46.1368  &  0.5463   &  1.1997   &  0.9911  &   45.5776  &  0.5667   &  1.3697   &  0.9912 \\        GTNN         &   48.2029  &  0.5980   &  1.5338   &  0.9851 &  46.7844   &   0.6262  &   1.6013  &  0.9848  \\     CMlpTR            &  \textbf{49.7612}   &   \textbf{0.4334}  &  \textbf{0.9367}   &  \textbf{0.9953}   & \textbf{48.4184}   &   \textbf{0.4591}  &  \textbf{1.0784}   &  \textbf{0.9946}   \\
			\hline
			\bottomrule
	\end{tabular}}
	\label{tab:Houston metrics}
\end{table}
\begin{figure}[htbp!]
	\centering
	\setlength{\tabcolsep}{0.2mm}
	\begin{tabular}{m{0.2\linewidth}m{0.2\linewidth}m{0.2\linewidth}m{0.2\linewidth}m{0.2\linewidth}}
		\includegraphics[width=\linewidth]{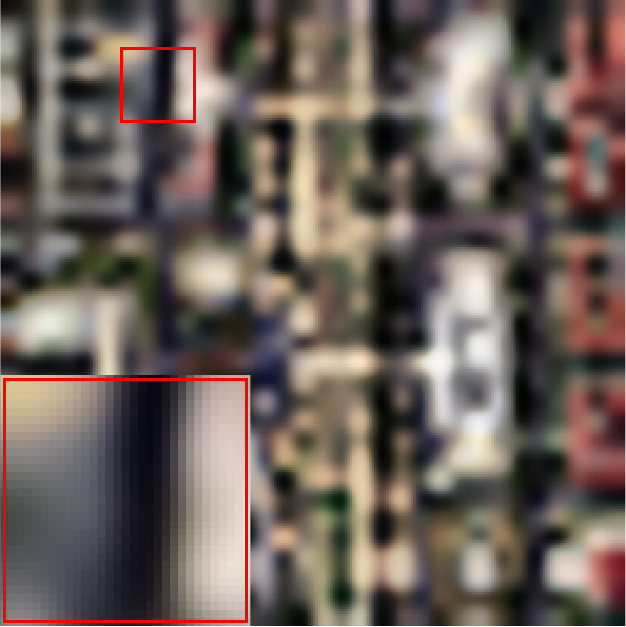}&
		\includegraphics[width=\linewidth]{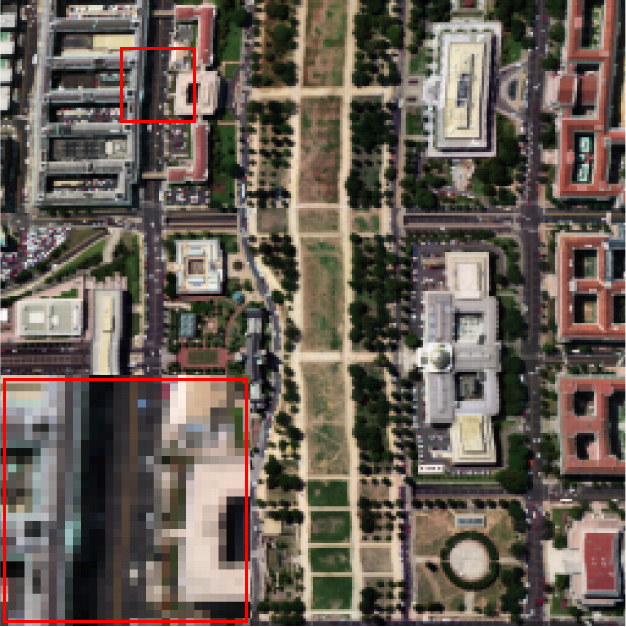}  &
		\includegraphics[width=\linewidth]{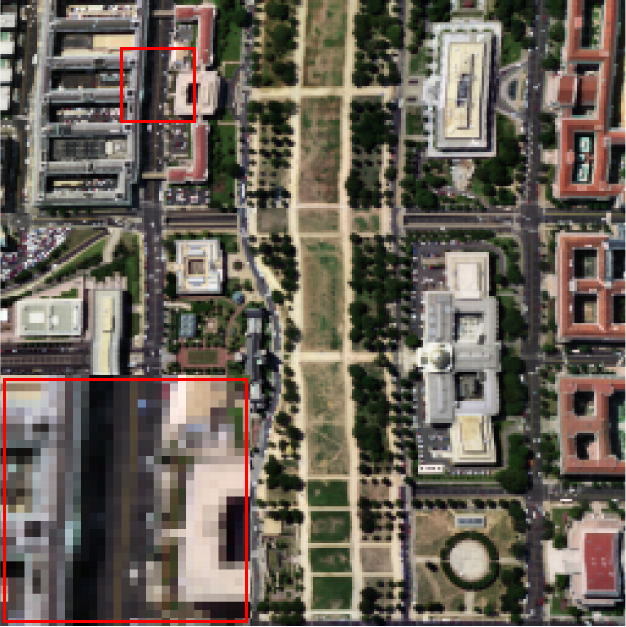}    &
		\includegraphics[width=\linewidth]{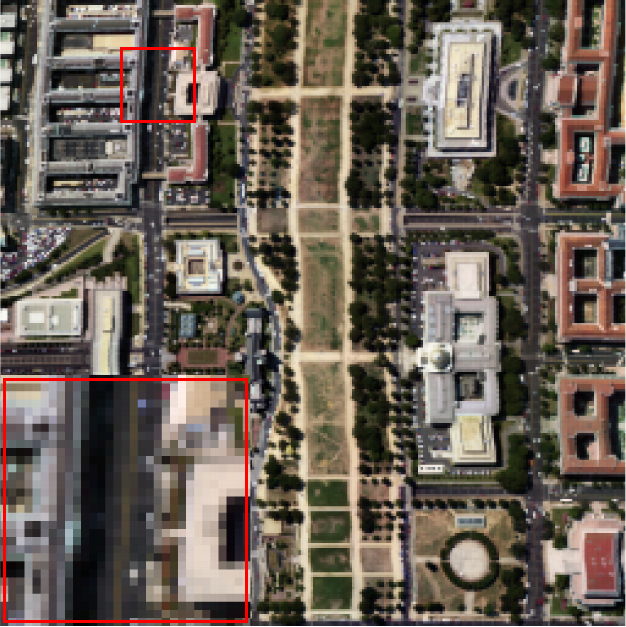}   &
		\includegraphics[width=\linewidth]{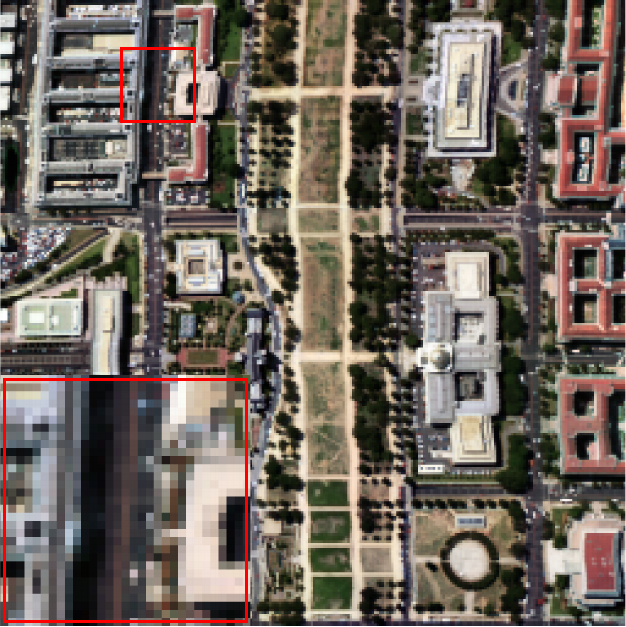}\\
		\includegraphics[width=\linewidth]{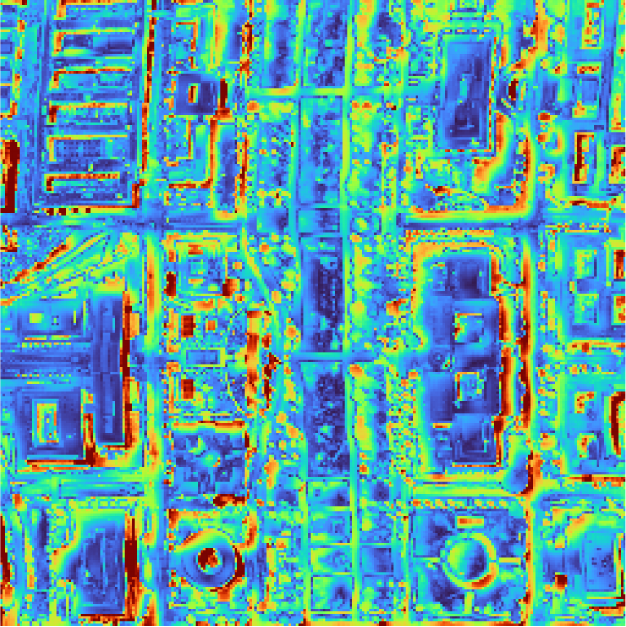}&
		\includegraphics[width=\linewidth]{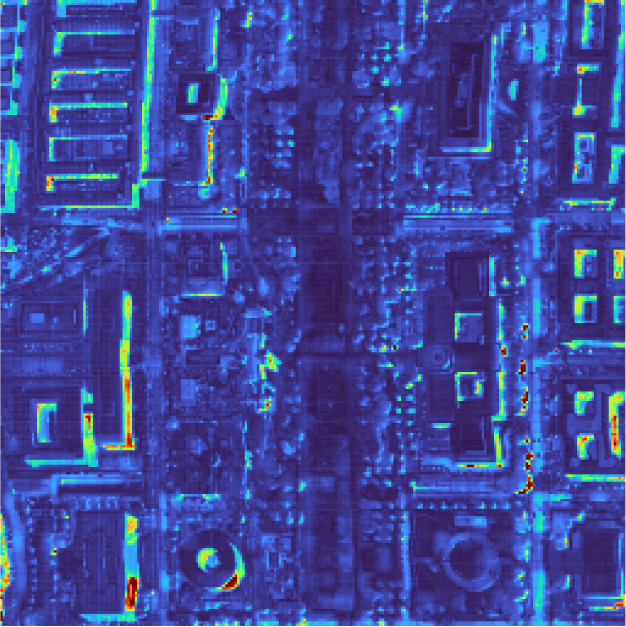}  &
		\includegraphics[width=\linewidth]{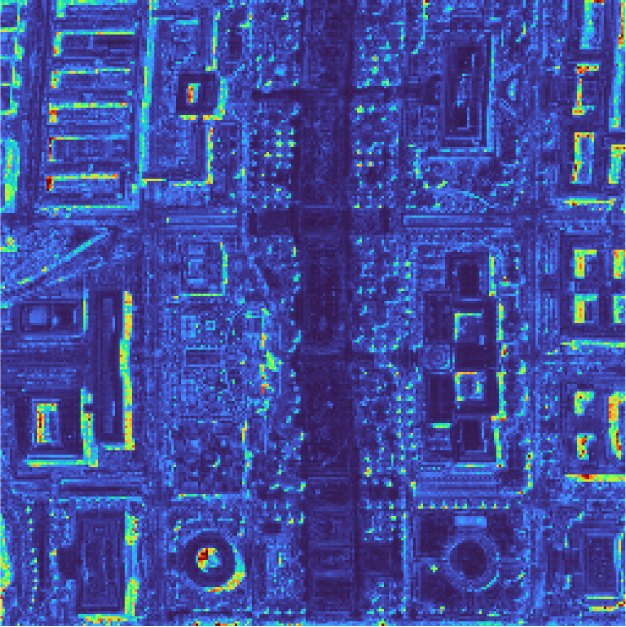}    &
		\includegraphics[width=\linewidth]{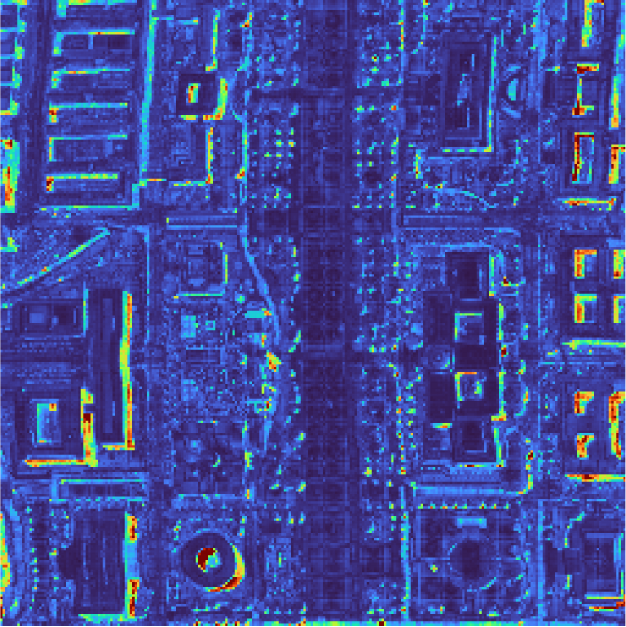}   &
		\includegraphics[width=\linewidth]{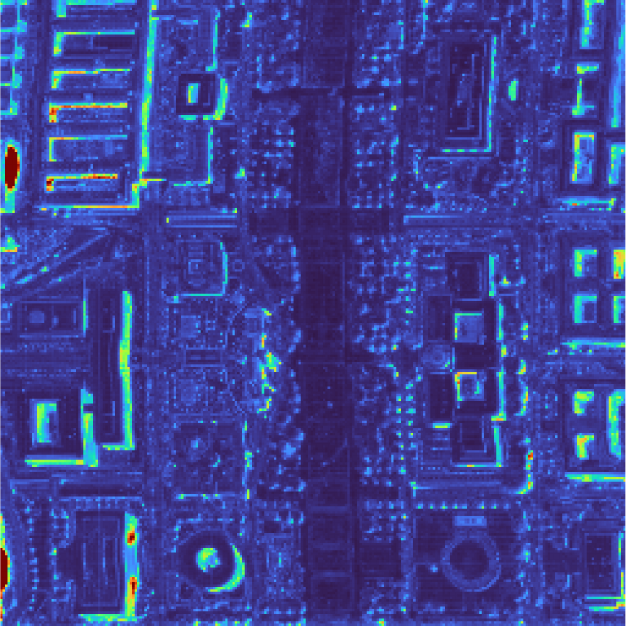}\\
		\multicolumn{1}{c}{\footnotesize{Bicubic}}
		&\multicolumn{1}{c}{\footnotesize{Hysure}}
		& \multicolumn{1}{c}{\footnotesize{LTTR}}
		& \multicolumn{1}{c}{\footnotesize{LRTA}}
		& \multicolumn{1}{c}{\footnotesize{SURE}}\\
		\includegraphics[width=\linewidth]{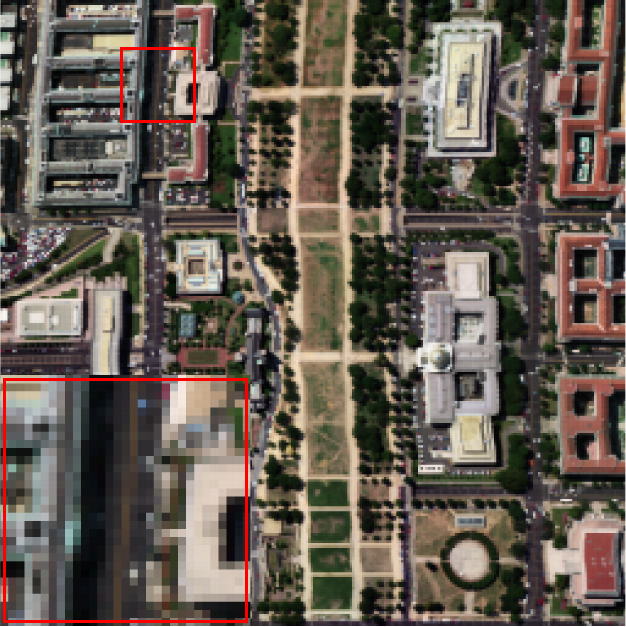}  &
		\includegraphics[width=\linewidth]{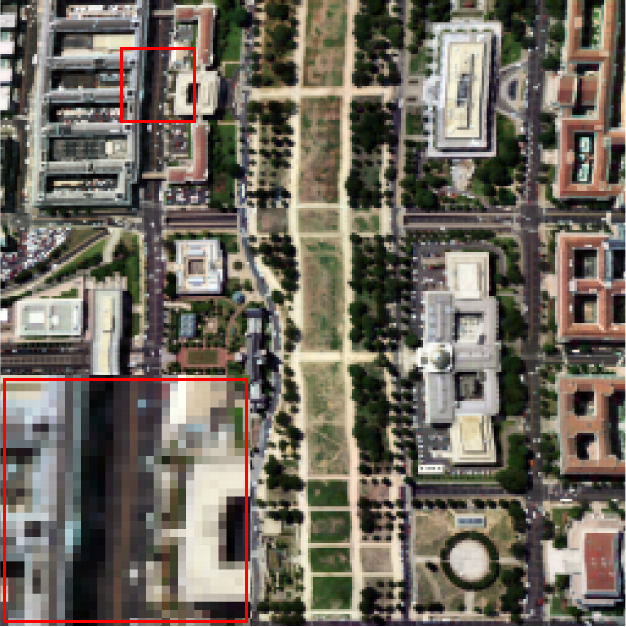}  &
		\includegraphics[width=\linewidth]{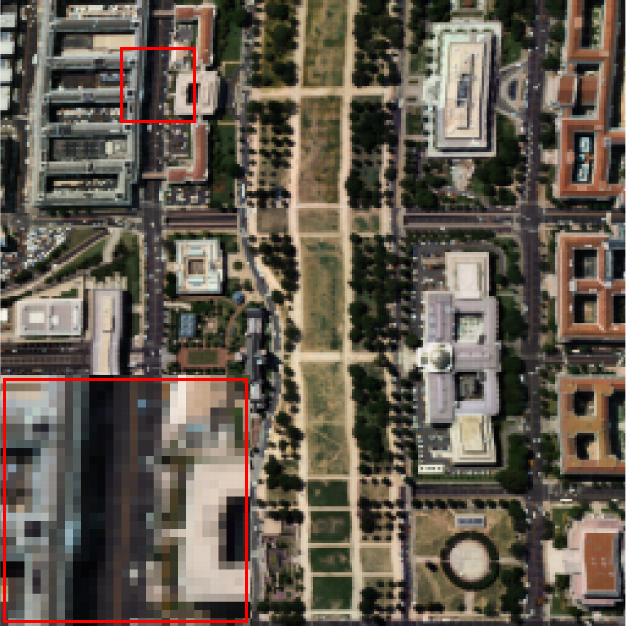}    &
		\includegraphics[width=\linewidth]{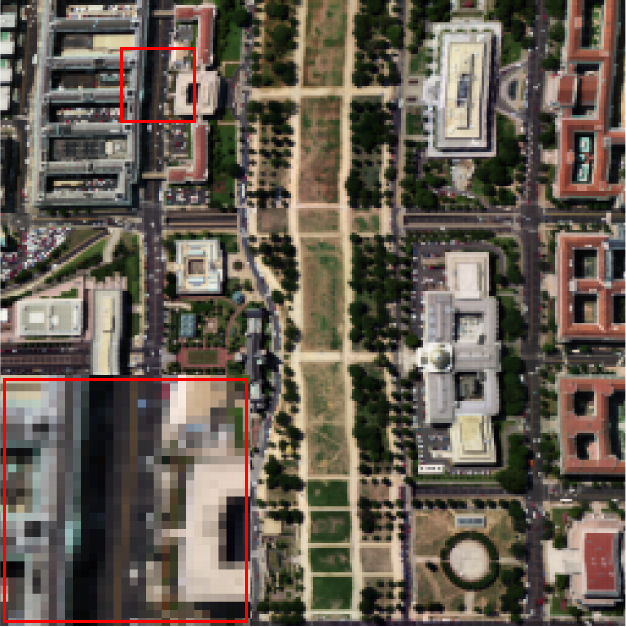}   &
		\includegraphics[width=\linewidth]{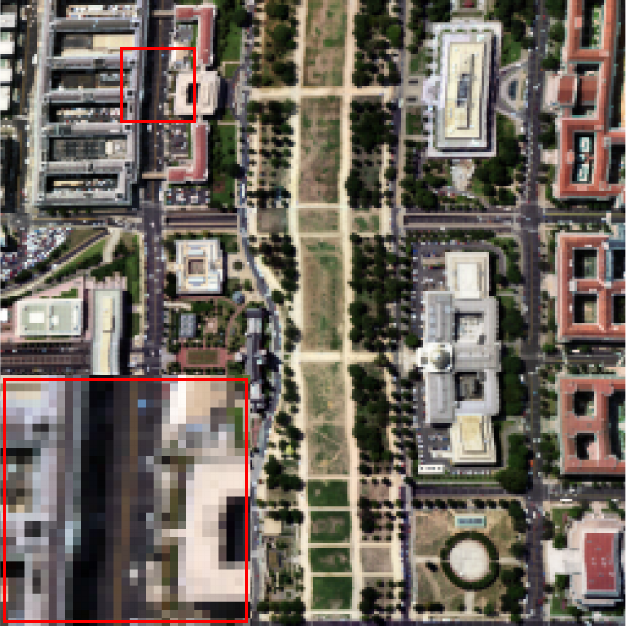}  \\
		\includegraphics[width=\linewidth]{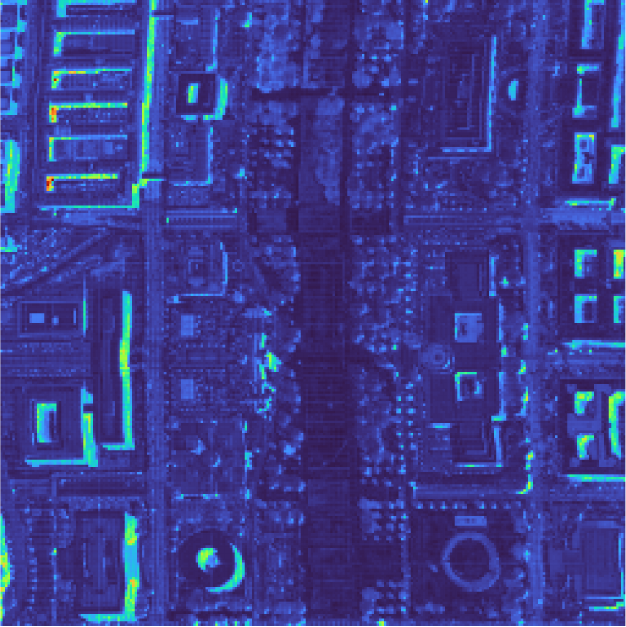}  &
		\includegraphics[width=\linewidth]{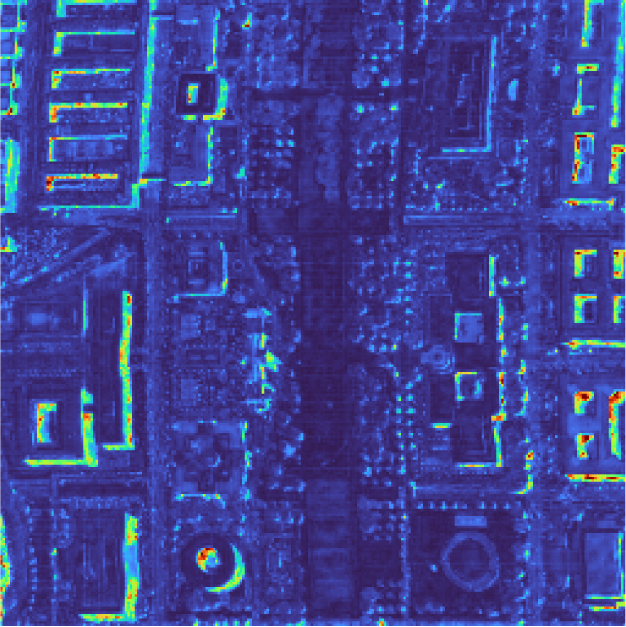}  &
		\includegraphics[width=\linewidth]{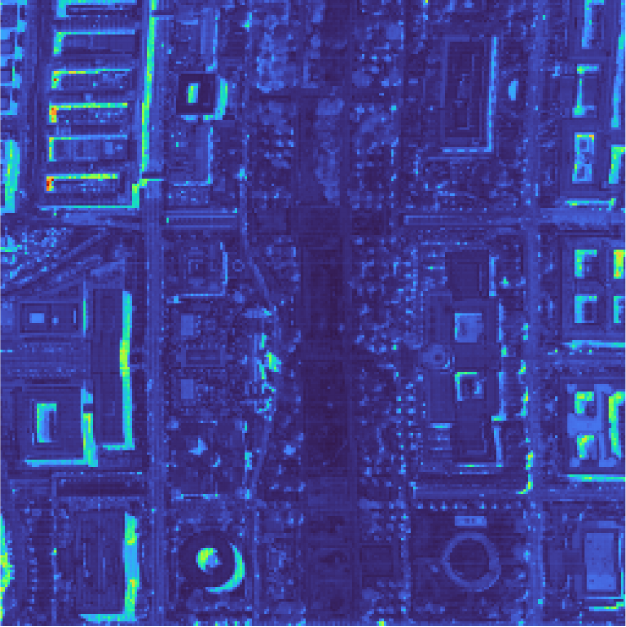}    &
		\includegraphics[width=\linewidth]{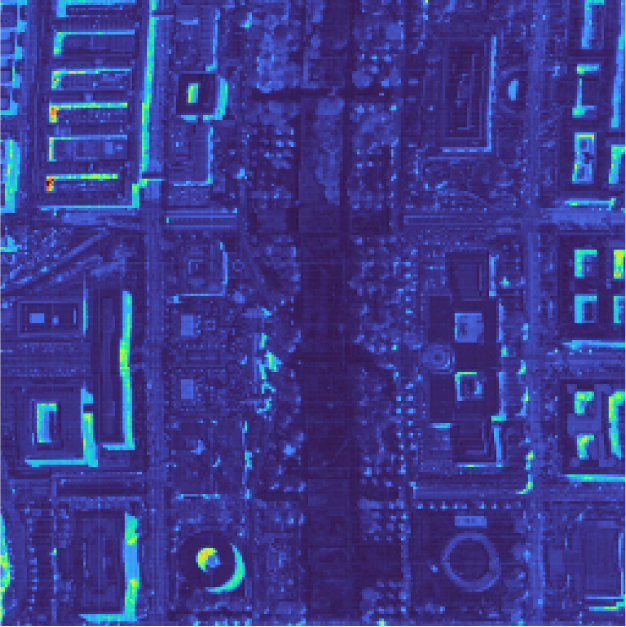}   &
		\includegraphics[width=\linewidth]{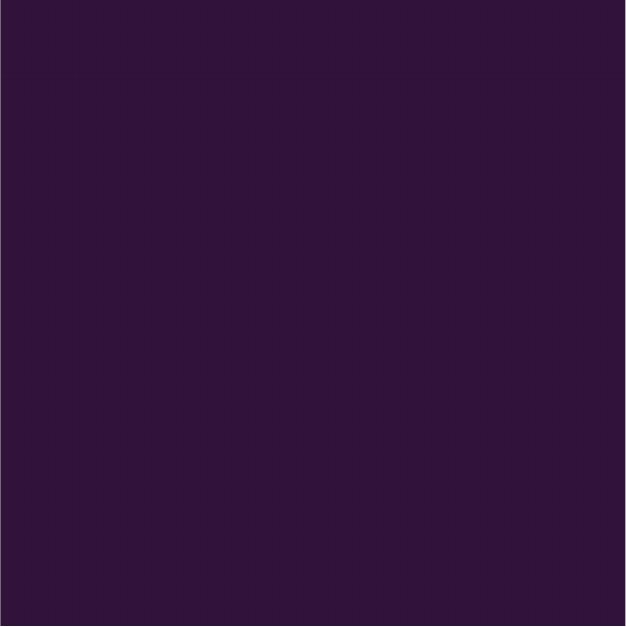} \\
		\multicolumn{1}{c}{\footnotesize{ASLA}}
		&\multicolumn{1}{c}{\footnotesize{ZSL}}
		& \multicolumn{1}{c}{\footnotesize{GTNN}}
		& \multicolumn{1}{c}{\footnotesize{CMlpTR}}
		& \multicolumn{1}{c}{\footnotesize{GT}}\\
		\multicolumn{5}{c}{\includegraphics[width=0.5\linewidth]{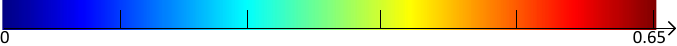}} 
	\end{tabular}
	\caption{\label{fig:WDC visualization 1} Non-blind fusion results and error maps on the WDC dataset. {Pseudo-color is composed of bands 40, 30 and 20.} Error maps are calculated by the pixel-wise SAM.}
\end{figure}
\subsubsection{HSR Results on URBAN Dataset}
The visual results are displayed in Figs. \ref{fig:URBAN visualization 1} and \ref{fig:URBAN visualization 2}. In the non-blind results, all methods end up with fine spatial details, only ASLA and GTNN exhibit slight color deviations and appear hazy. Besides, SURE, ZSL, GTNN and the proposed CMlpTR generate the most darkest error maps, among which our CMlpTR is the best at suppressing higher error values. Similar conclusions can be drawn from the blind results in Fig. \ref{fig:URBAN visualization 2}, except that all error maps have become brighter due to the estimation inaccuracies in the degradation matrices. The conclusions above are further strengthened by the numerical results in Tab. \ref{tab:URBAN metrics}, in which the proposed CMlpTR yielded the best values on all indexes, while SURE, ZSL and GTNN are the most competitive having a look at the SAM metric.
\subsubsection{HSR Results on Houston Dataset}
The results on the Houston dataset are depicted in Figs. \ref{fig:Houston visualization 1}-\ref{fig:Houston visualization 2} and numerically reported in Tab. \ref{tab:Houston metrics}. By the overall brightness of the error maps, it appears that ASLA, ZSL and the proposed CMlpTR deliver the closest reconstruction. However, ZSL suffers from block-like artifacts and ASLA is not as good as our CMlpTR in background suppression. Even though TV constraints are considered in ASLA, the multi-constraint paradigm leaves the problem of mutual interference. Besides, hinted by the numerical results in Tab. \ref{tab:Houston metrics}, it is noticed that LRTA is better at reconstructing the bright regions in our CMlpTR's error map, leading to its competitiveness observing the SAM metric, even though LRTA shows random high-error pixels due to the neglect of TV priors.
\begin{figure}[htbp!]
	\centering
	\setlength{\tabcolsep}{0.2mm}
	\begin{tabular}{m{0.2\linewidth}m{0.2\linewidth}m{0.2\linewidth}m{0.2\linewidth}m{0.2\linewidth}}
		\includegraphics[width=\linewidth]{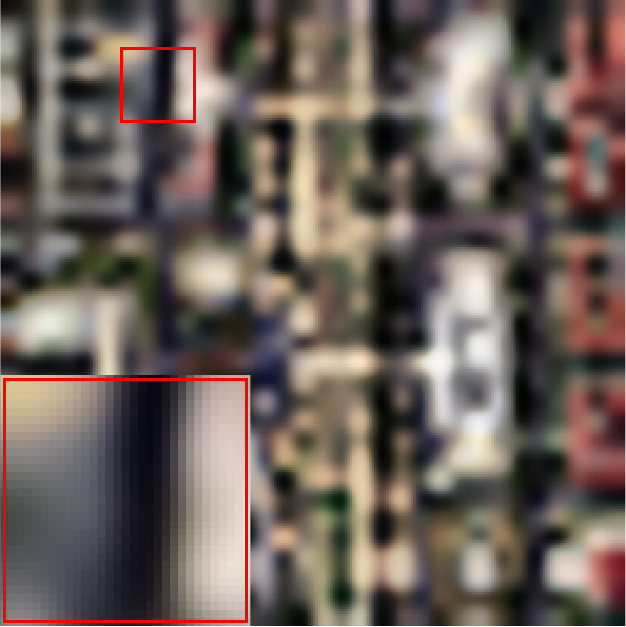}&
		\includegraphics[width=\linewidth]{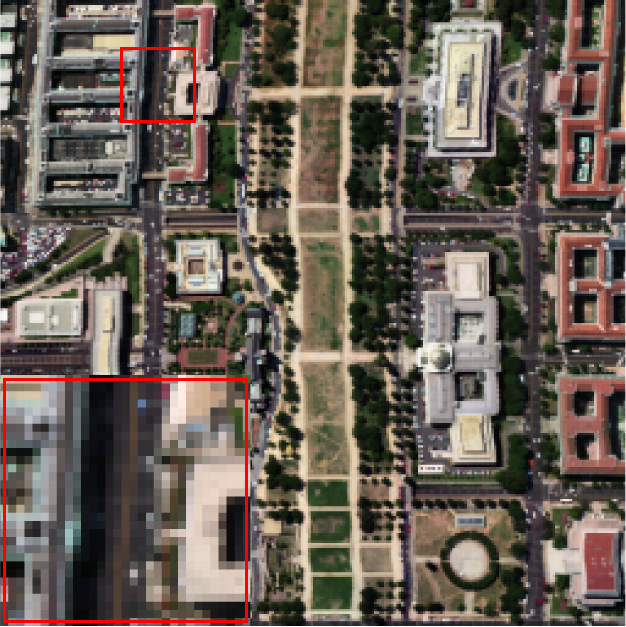}  &
		\includegraphics[width=\linewidth]{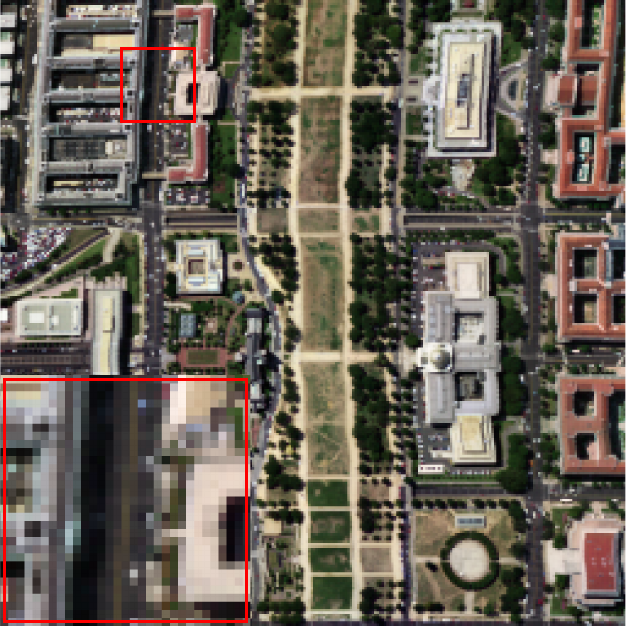}    &
		\includegraphics[width=\linewidth]{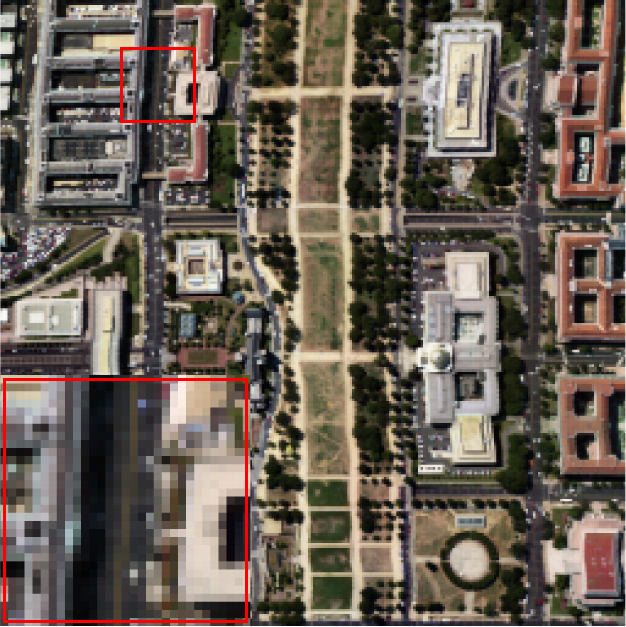}   &
		\includegraphics[width=\linewidth]{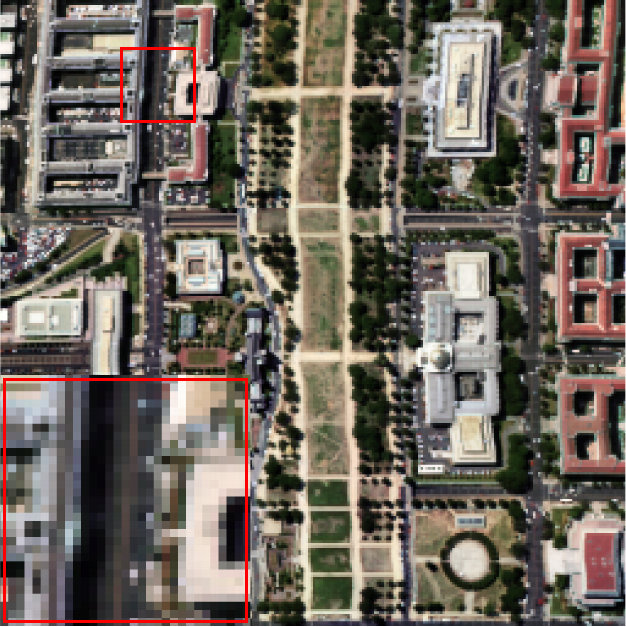}\\
		\includegraphics[width=\linewidth]{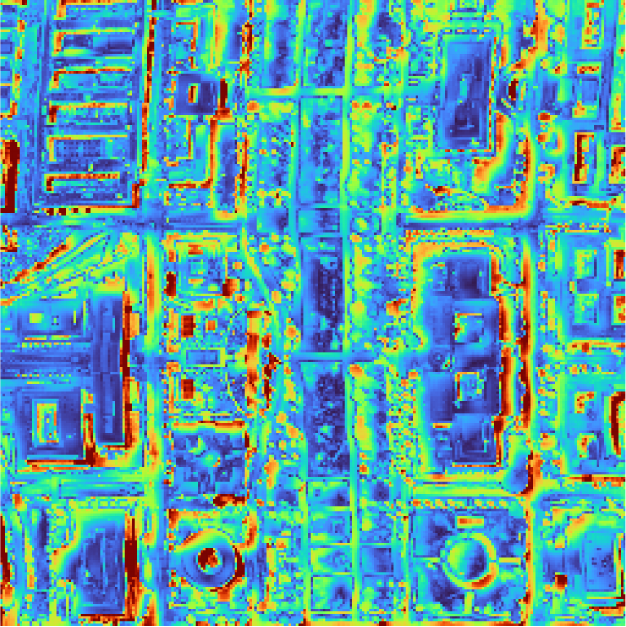}&
		\includegraphics[width=\linewidth]{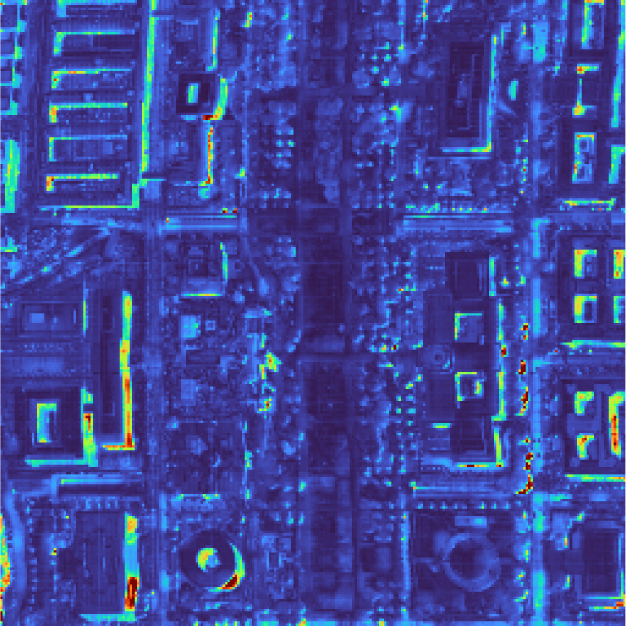}  &
		\includegraphics[width=\linewidth]{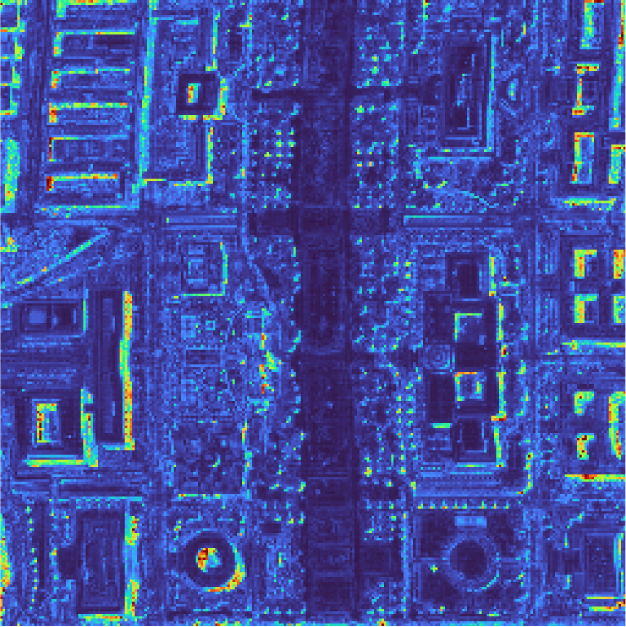}    &
		\includegraphics[width=\linewidth]{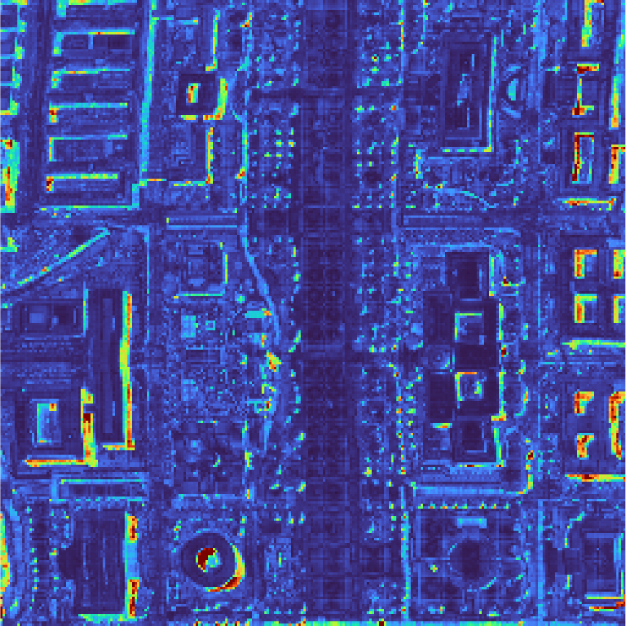}   &
		\includegraphics[width=\linewidth]{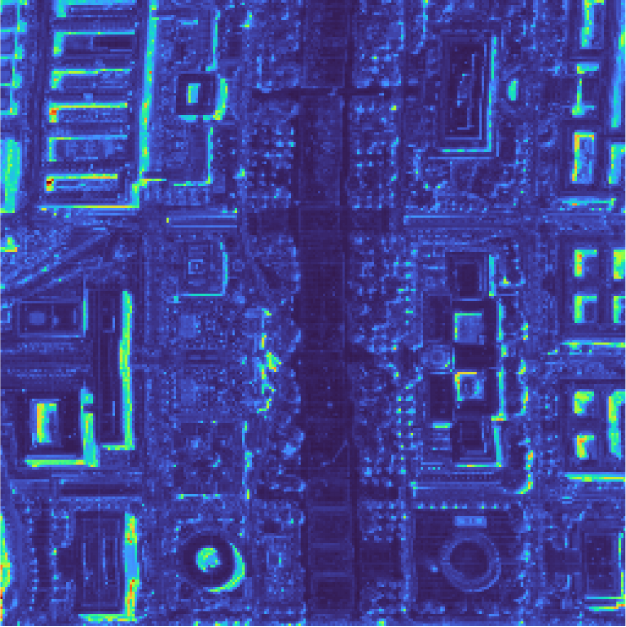}\\
		\multicolumn{1}{c}{\footnotesize{Bicubic}}
		&\multicolumn{1}{c}{\footnotesize{Hysure}}
		& \multicolumn{1}{c}{\footnotesize{LTTR}}
		& \multicolumn{1}{c}{\footnotesize{LRTA}}
		& \multicolumn{1}{c}{\footnotesize{SURE}}\\
		\includegraphics[width=\linewidth]{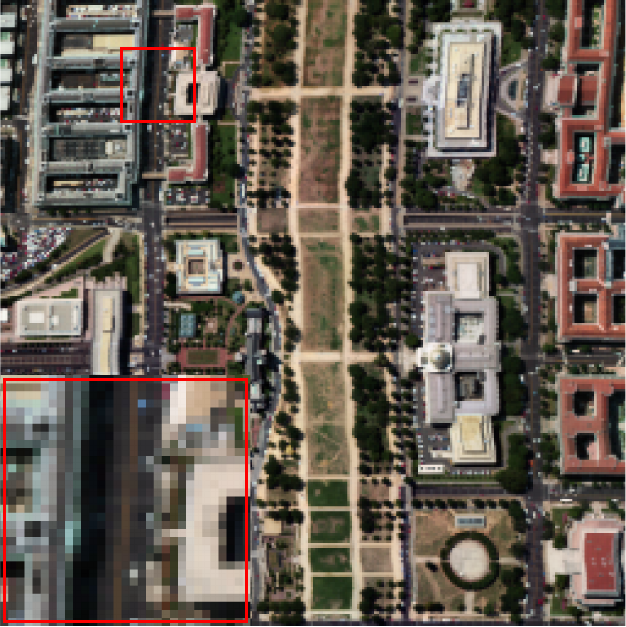}  &
		\includegraphics[width=\linewidth]{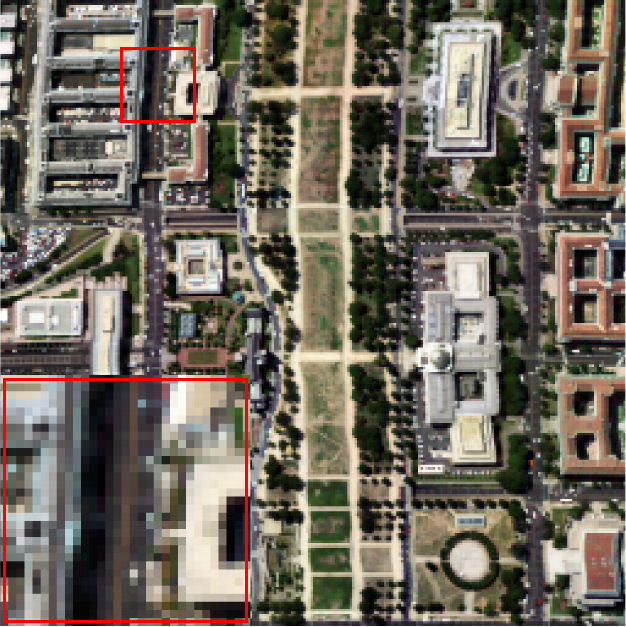}  &
		\includegraphics[width=\linewidth]{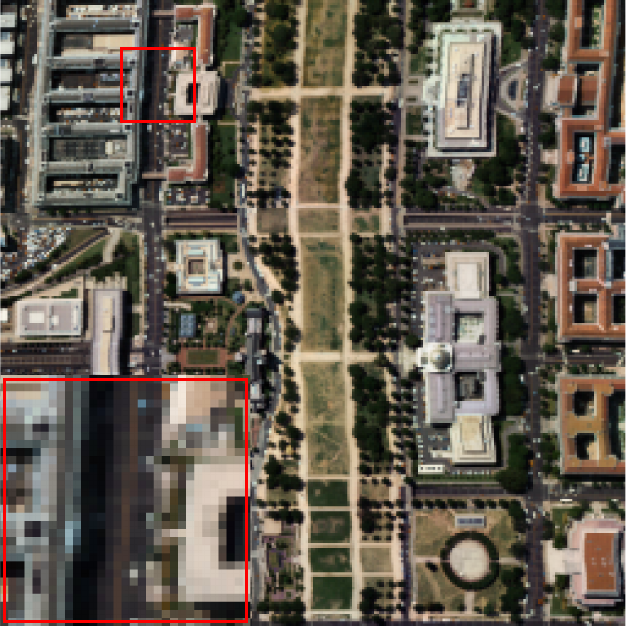}    &
		\includegraphics[width=\linewidth]{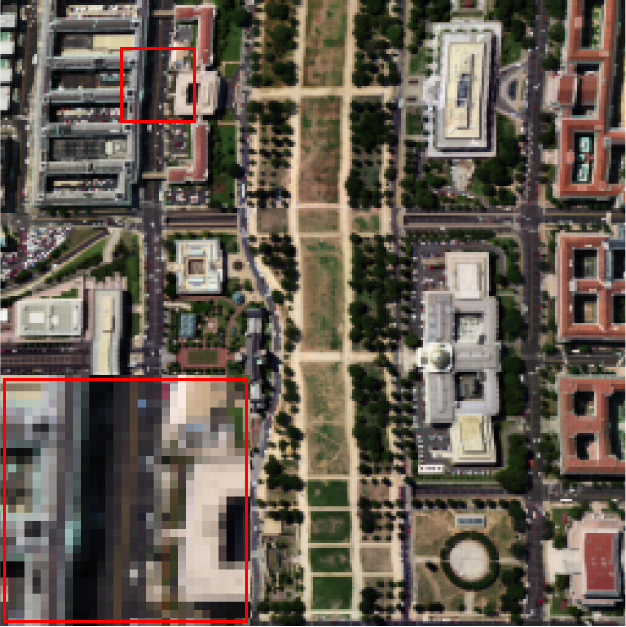}   &
		\includegraphics[width=\linewidth]{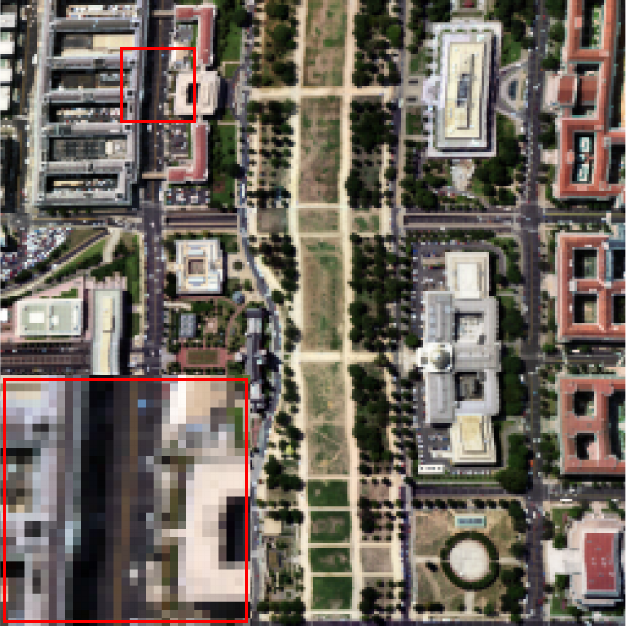}  \\
		\includegraphics[width=\linewidth]{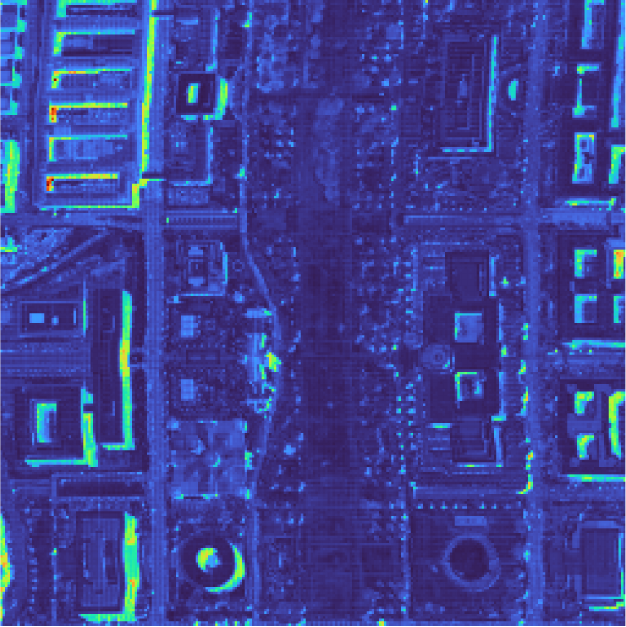}  &
		\includegraphics[width=\linewidth]{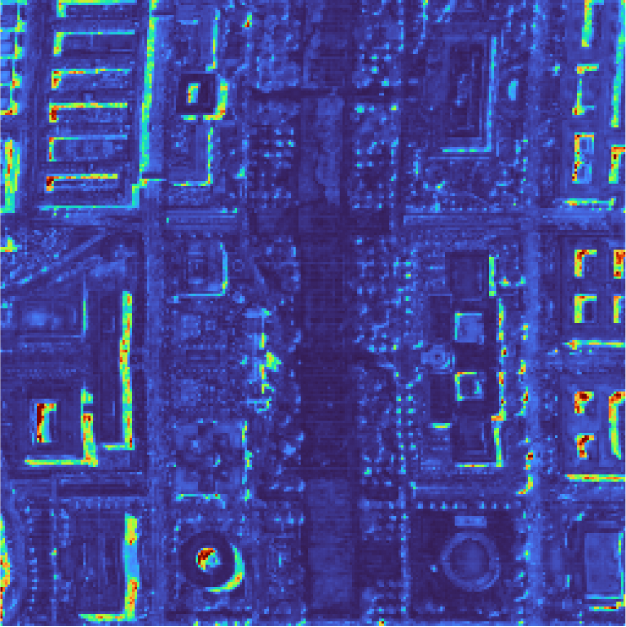}  &
		\includegraphics[width=\linewidth]{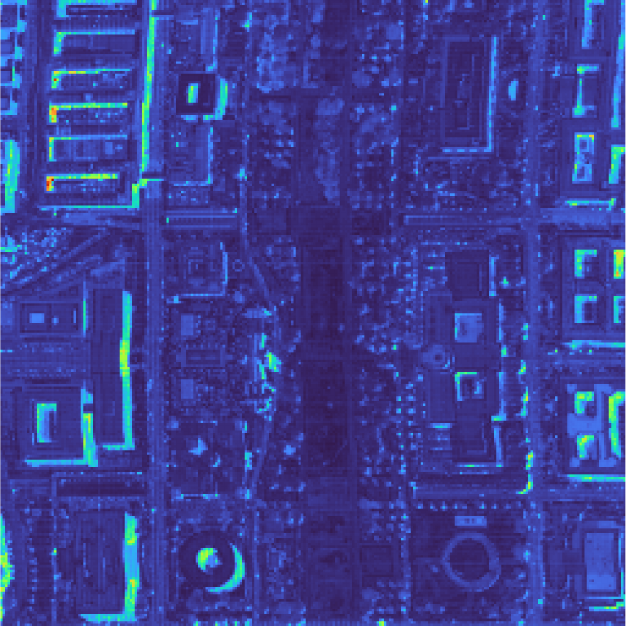}    &
		\includegraphics[width=\linewidth]{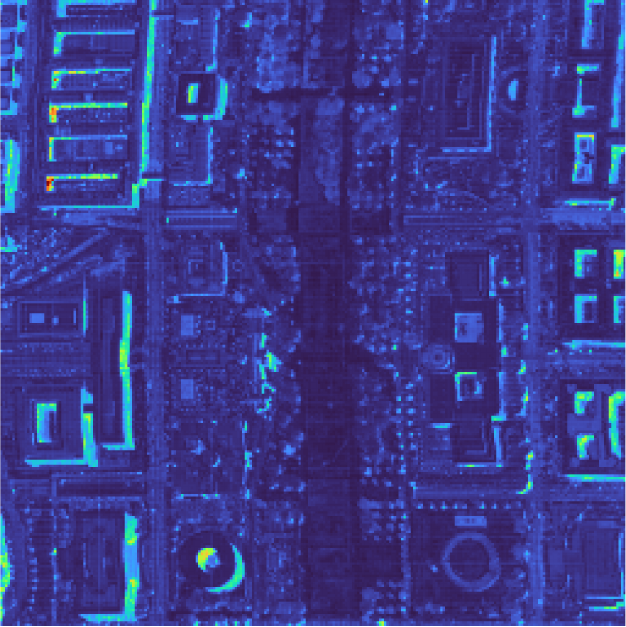}   &
		\includegraphics[width=\linewidth]{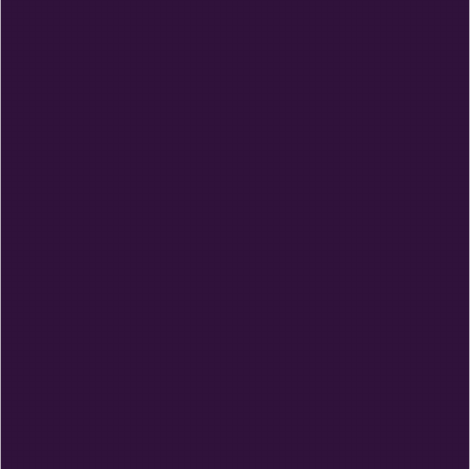} \\
		\multicolumn{1}{c}{\footnotesize{ASLA}}
		&\multicolumn{1}{c}{\footnotesize{ZSL}}
		& \multicolumn{1}{c}{\footnotesize{GTNN}}
		& \multicolumn{1}{c}{\footnotesize{CMlpTR}}
		& \multicolumn{1}{c}{\footnotesize{GT}}\\
		\multicolumn{5}{c}{\includegraphics[width=0.5\linewidth]{WDC_nblind_colorbar.pdf}} 
	\end{tabular}
	\caption{\label{fig:WDC visualization 2} Blind fusion results and error maps on the WDC dataset. {Pseudo-color is composed of bands 40, 30 and 20.} Error maps are calculated by the pixel-wise SAM.}
\end{figure}
\begin{table}[htbp]
	\centering
	\renewcommand{\arraystretch}{1}
	\tabcolsep=0.5mm
	\caption{Numerical performance on the WDC dataset. Best results are in boldface.}
	\resizebox{\linewidth}{!}{\begin{tabular}{c|cccc|cccc}
			\toprule
			\hline
			\multirow{3}{*}{\textbf{Methods}}&\multicolumn{8}{c}{\textbf{Setup}} \\\cline{2-9}\quad& \multicolumn{4}{c|}{\textbf{Non-blind}} &\multicolumn{4}{c}{\textbf{Blind}}\\\cline{2-9} \quad
			&  PSNR$\uparrow$                 & {ERGAS}$\downarrow$    & SAM$\downarrow$    & SSIM$\uparrow$  &    PSNR$\uparrow$                 & {ERGAS}$\downarrow$    & SAM$\downarrow$    & SSIM$\uparrow$   \\\hline
			{Bicubic}            &  19.6920   &  7.1549   &   12.3534  &   0.2588 &  19.6920   &  7.1549   &   12.3534  &   0.2588
			\\Hysure            &   43.6645   &  0.6142   &   3.2399  &   0.9936 &  43.4473   &  0.6048   &   3.2742  &   0.9935 \\ LTTR 
			&       42.9732          &  0.7308   &   3.5122  &  0.9886    &           42.7777          &  0.7406   &   3.6109  &  0.9883    \\
			LRTA             & 43.8438   & 0.6604    & 3.6210    &   0.9890   & 43.6905   & 0.6624    & 3.6378    &   0.9890  \\ SURE           & 42.2816   &  0.5884   &   2.9213  &  0.9931   & 42.0958   &  0.5961   &   2.8804  &  0.9928   \\       ASLA          &   \textbf{45.1584}   &  0.5735   &  2.4119   & 0.9958  &    44.3206   &  0.5838   &  2.7515   & 0.9950  \\        ZSL           &    42.3886  &  0.5939   &  2.9659   &  0.9935  &   42.9526  &  0.6144   &  3.1297   &  0.9932  \\        GTNN         &   43.2923  &  0.5920   &  2.5136   &  0.9949 & 43.2412   &   0.5917  &   2.5163  &  0.9949   \\     CMlpTR            & 45.0320   &   \textbf{0.4819}  &  \textbf{2.2969}   &  \textbf{0.9959}   & \textbf{44.9592}   &   \textbf{0.4817}  &  \textbf{2.2971}   &  \textbf{0.9959}   \\
			\hline
			\bottomrule
	\end{tabular}}
	\label{tab:WDC metrics}
\end{table}
\subsubsection{HSR Results on WDC Dataset}
The results on the WDC dataset are depicted in Figs. \ref{fig:WDC visualization 1}-\ref{fig:WDC visualization 2} and numerically reported in Tab. \ref{tab:Houston metrics}. These visual results demonstrate that the WDC dataset presents greater textural complexity compared to the other two datasets, as evidenced by the pronounced high-frequency patterns in the error maps and the extended dynamic range of the colorbar. In Tab. \ref{tab:WDC metrics}, the proposed CMlpTR performs the best while ASLA turns out to be the most competitive achieving the best PSNR value in the non-blind experiment. The visual results agree with the numerical ones, where ASLA, GTNN and CMlpTR obtain the darkest error maps.
\section{Conclusion}
\label{sec:Conclusion}
This paper raised the attention on the difficulty in the co-representation of multi-level-priors in the HSR task. Through the proposal of our NMS-t-CTV and combining it with the BTD, we achieved compact and tight representations of multi-level and multi-dimensional structural correlations inherent in HSSI while maintaining theoretical soundness. To efficiently solve the resulting high-dimensional optimization problem, a LADMM-inspired algorithm was customized, with convergence guarantees established under practical boundedness conditions. The compactness of the proposed model not only facilitated convergence analysis but also saved it from the problem of mutual-interference between multiple constraints. Extensive experiments on both non-blind and blind HSR tasks demonstrated the practical effectiveness of the proposed approach, showing improved performance in preserving spectral-spatial details while keeping computational demands manageable.
\appendix
\subsection{Proof of Prop. \ref{Prop: Mlp}}
\label{App: Proof 1}
\begin{proof}
	In analogy with the proof architecture of Prop. 2 in the Supplementary Material of \cite{10078018}, on one hand, we have 
	\begin{align*}
		 &\quad rank_{\overset{3-n}{t}}\left(\mathcal{A}\times_n\boldsymbol{D}_{I_n}\right)\\&=rank_{\overset{3-n}{t}}\left(\mathcal{Z}\times_n\boldsymbol{D}_{I_n}\times_3\boldsymbol{S}^\ast\right)\\&=\underset{i\in\left\{1,2,\cdots,I_{3-n}\right\}}{\max} rank\left(\overline{\mathtt{P}_{3-n}\left(\mathcal{Z}\times_n\boldsymbol{D}_{I_n}\times_3\boldsymbol{S}^\ast\right)}_{:,:,i}\right)\\&=\underset{i\in\left\{1,2,\cdots,I_{3-n}\right\}}{\max} rank\left(\boldsymbol{D}_{I_n}\overline{\mathtt{P}_{3-n}\left(\mathcal{Z}\right)}_{:,:,i}\boldsymbol{S}^\ast\right)\\&=\underset{i\in\left\{1,2,\cdots,I_{3-n}\right\}}{\max} rank\left(\boldsymbol{D}_{I_n}\overline{\mathtt{P}_{3-n}\left(\mathcal{Z}\right)}_{:,:,i}\right)\\&\geq\underset{i\in\left\{1,2,\cdots,I_{3-n}\right\}}{\max} \Big\{rank\left(\boldsymbol{D}_{I_n}\right)+rank\left(\overline{\mathtt{P}_{3-n}\left(\mathcal{Z}\right)}_{:,:,i}\right)-I_n\Big\}\\&=\underset{i\in\left\{1,2,\cdots,I_{3-n}\right\}}{\max} \left\{rank\left(\overline{\mathtt{P}_{3-n}\left(\mathcal{Z}\right)}_{:,:,i}\right)-1\right\}\\&=rank_{\overset{3-n}{t}}\left(\mathcal{Z}\right)-1.
	\end{align*}
	On the other hand, 
	\begin{align*}
		 rank_{\overset{3-n}{t}}\left(\mathcal{A}\times_n\boldsymbol{D}_{I_n}\right)&=\underset{i\in\left\{1,2,\cdots,I_{3-n}\right\}}{\max} rank\left(\boldsymbol{D}_{I_n}\overline{\mathtt{P}_{3-n}\left(\mathcal{Z}\right)}_{:,:,i}\right)\\&\leq\underset{i\in\left\{1,2,\cdots,I_{3-n}\right\}}{\max} rank\left(\overline{\mathtt{P}_{3-n}\left(\mathcal{Z}\right)}_{:,:,i}\right)\\&=rank_{\overset{3-n}{t}}\left(\mathcal{Z}\right).
	\end{align*}
	Thus, Eq. \eqref{eq: Mlp1} is proven.
	We proceed to prove Eq. \eqref{eq: Mlp2}. Denoting $\mathcal{G}_n=\mathcal{A}\times_n\boldsymbol{D}_{I_n}$, let $\mathcal{U}\star\mathcal{S}\star\mathcal{V}^\ast$ be the t-SVD of $\mathtt{P}_{3-n}\left(\mathcal{G}_n\right)=\mathsf{permute}\left(\mathcal{G}_n,[n,3,3-n]\right)$, then we can write
	\begin{align*}
		\left\lVert\mathcal{G}_n\right\lVert_{\overset{3-n}{\ast},\psi}=\frac{1}{I_{3-n}}\sum_{i=1}^{I_{3-n}}\sum_{j=1}^{\min\left\{I_n,R\right\}}\psi(\overline{\mathcal{S}}\left[j,j,i\right]).
	\end{align*}
	Since $\mathcal{A}$ is bounded, $\exists\,x_B>0$, such that $\overline{\mathcal{S}}\left[j,j,i\right]\leq x_B,\,\forall\,i,j.$ Let $b=\frac{\psi(x_B)}{x_B}$, then based on the properties of $\psi(\cdot)$, $\psi(\overline{\mathcal{S}}\left[j,j,i\right])\geq b\overline{\mathcal{S}}\left[j,j,i\right],\,\forall\,i,j$. Thereupon,
	\begin{align*}
		\left\lVert\mathcal{G}_n\right\lVert_{\overset{3-n}{\ast},\psi}&=\frac{1}{I_{3-n}}\sum_{i=1}^{I_{3-n}}\sum_{j=1}^{\min\left\{I_n,R\right\}}\psi(\overline{\mathcal{S}}\left[j,j,i\right])\\&\geq\frac{b}{I_{3-n}}\sum_{i=1}^{I_{3-n}}\sum_{j=1}^{\min\left\{I_n,R\right\}}\overline{\mathcal{S}}\left[j,j,i\right]\\&=\frac{b}{I_{3-n}}\left\lVert\overline{\mathtt{P}_{3-n}\left(\mathcal{G}_n\right)}\right\lVert_\ast\\&=\frac{b}{I_{3-n}}\left\lVert\overline{\mathtt{P}_{3-n}\left(\mathcal{G}_n\right)}\right\lVert_F\\&\geq\frac{b}{\sqrt{I_{3-n}}}\left\lVert{\mathtt{P}_{3-n}\left(\mathcal{G}_n\right)}\right\lVert_F\\&\geq\frac{b}{\sqrt{I^m}}\left\lVert{\mathcal{G}_n}\right\lVert_F\\&\geq\frac{b}{\sqrt{I^mI_1I_2R}}\left\lVert{\mathcal{G}_n}\right\lVert_1
	\end{align*}
	where $I^m\triangleq\max\left\{I_1,I_2\right\}$. It then follows that:
	\begin{align*}
		\left\lVert\mathcal{A}\right\lVert_{\overset{\sim}{\ast},\mathtt{\psi}}&\geq\frac{b}{2\sqrt{I^m}}\left\lVert{\mathcal{A}}\right\lVert_{TV},\\
		\left\lVert\mathcal{A}\right\lVert_{\overset{\sim}{\ast},\mathtt{\psi}}&\geq\frac{b}{2\sqrt{I^mI_1I_2R}}\left\lVert{\mathcal{A}}\right\lVert_{ATV}.
	\end{align*}
	Furthermore, let $g\triangleq\underset{x\rightarrow0^+}{\overline{\lim}}\psi(x)$, then based on the properties of $\psi(\cdot)$, we have $\psi(\overline{\mathcal{S}}\left[j,j,i\right])\leq g\overline{\mathcal{S}}\left[j,j,i\right]$. Then, it is derived that:
	\begin{align*}
		\left\lVert\mathcal{G}_n\right\lVert_{\overset{3-n}{\ast},\psi}&=\frac{1}{I_{3-n}}\sum_{i=1}^{I_{3-n}}\sum_{j=1}^{\min\left\{I_n,R\right\}}\psi(\overline{\mathcal{S}}\left[j,j,i\right])\\&\leq\frac{g}{I_{3-n}}\sum_{i=1}^{I_{3-n}}\sum_{j=1}^{\min\left\{I_n,R\right\}}\overline{\mathcal{S}}\left[j,j,i\right]\\&=\frac{g}{I_{3-n}}\left\lVert\overline{\mathtt{P}_{3-n}\left(\mathcal{G}_n\right)}\right\lVert_\ast\\&\leq g\sqrt{\frac{\min\left\{I_n,R\right\}}{I_{3-n}}}\left\lVert\overline{\mathtt{P}_{3-n}\left(\mathcal{G}_n\right)}\right\lVert_F\\&= g\sqrt{\min\left\{I_n,R\right\}}\left\lVert{\mathtt{P}_{3-n}\left(\mathcal{G}_n\right)}\right\lVert_F\\&= g\sqrt{\min\left\{I_n,R\right\}}\left\lVert{\mathcal{G}_n}\right\lVert_F\\&\leq g\sqrt{\min\left\{I_n,R\right\}}\left\lVert{\mathcal{G}_n}\right\lVert_1,
	\end{align*}
	which leads to
	\begin{align*}
		\left\lVert\mathcal{A}\right\lVert_{\overset{\sim}{\ast},\mathtt{\psi}}&\leq g\sqrt{2\min\left\{I_n,R\right\}}\left\lVert{\mathcal{A}}\right\lVert_{TV},\\
		\left\lVert\mathcal{A}\right\lVert_{\overset{\sim}{\ast},\mathtt{\psi}}&\leq g\sqrt{\min\left\{I_n,R\right\}}\left\lVert{\mathcal{A}}\right\lVert_{ATV}.
	\end{align*}
	Hence, Eq. \eqref{eq: Mlp2} is proven.
\end{proof}

\subsection{Proof of Thm. \ref{Thm: Convergence}}
\label{App: Proof 2}
\begin{proof}
	Proof of (1) and (2): 
	 According to the optimality of subproblem \eqref{eq: G-sub}, it holds in the $(t+1)$th iteration that:
	 \begin{equation}
	 	\begin{aligned}
	 		0\in\partial&\left\lVert\mathcal{G}_{n,t+1}\right\lVert_{\overset{3-n}{\ast},\psi}+2\rho_t\left(\mathcal{G}_{n,t+1}+\frac{\mathcal{M}_{n,t}}{\rho_t}-\mathcal{A}_{t+1}\times_n\boldsymbol{D}_{I_n}\right)
	 	\end{aligned}
	 	\nonumber
	 \end{equation}
	 where $\partial\left\lVert\mathcal{G}_{n,t+1}\right\lVert_{\overset{3-n}{\ast},\psi}$ denotes the subdifferential of $\left\lVert\cdot\right\lVert_{\overset{3-n}{\ast},\psi}$ at $\mathcal{G}_{n,t+1}$, whose existence is guaranteed by Lemma A1 in the Supplementary Material of \cite{9340243}, which together with \textbf{P1} and \textbf{P2} results in the boundedness of $\partial\left\lVert\mathcal{G}_{n,t+1}\right\lVert_{\overset{3-n}{\ast},\psi}$. Further exploiting the updating rule of $\mathcal{M}_n$, it is derived that:
	 \begin{align}
	 	\mathcal{M}_{n,t+1}\in-\frac{1}{2}\partial\left\lVert\mathcal{G}_{n,t+1}\right\lVert_{\overset{3-n}{\ast},\psi}
	 	\label{eq: G-opt}
	 \end{align}
	 Thus, $\mathcal{M}_{n,t+1}$ is bounded. Moreover, by the updating rule \eqref{eq: A-update}, 
	 \begin{align}
	 	&\quad\tau\left(\mathcal{A}_{t+1}-\mathcal{A}_t\right)\\&=-\bigtriangledown L_1\left(\mathcal{A}_t\right)\notag\\&=-2\Bigg(\mathcal{A}_t\times_1\boldsymbol{P}_1^*\boldsymbol{P}_1\times_2\boldsymbol{P}_2^*\boldsymbol{P}_2+\mathcal{A}_t\times_3\boldsymbol{S}^*\boldsymbol{P}_3^*\boldsymbol{P}_3\boldsymbol{S}+\notag\\&\quad\quad\quad\quad\sum_{n=1}^{2}\mathcal{A}_t\times_n\boldsymbol{D}_{I_n}^*\boldsymbol{D}_{I_n}-\left(\mathcal{X}+\frac{\mathcal{M}_{x,t}}{\rho_t}\right)\times_1\boldsymbol{P}_1^*\notag\\&\quad\quad\quad\quad\times_2\boldsymbol{P}_2^*\times_3\boldsymbol{S}^*-\left(\mathcal{Y}+\frac{\mathcal{M}_{y,t}}{\rho_t}\right)\times_3\boldsymbol{S}^*\boldsymbol{P}_3^*\notag\\&\quad\quad\quad-\sum_{n=1}^{2}\left(\mathcal{G}_{n,t}+\frac{\mathcal{M}_{n,t}}{\rho_t}\right)\times_n\boldsymbol{D}_{I_n}^*\Bigg).\label{eq: aux}
	 \end{align}
	 By the updating rule of $\mathcal{M}_{x}$,
	 $$\mathcal{X}+\frac{\mathcal{M}_{x,t}}{\rho_t}=\frac{\mathcal{M}_{x,t+1}}{\rho_t}+\mathcal{A}_{t+1}\times_1\boldsymbol{P}_1\times_2\boldsymbol{P}_2\times_3\boldsymbol{S}.$$
	 Then
	 \begin{align}
	 	&\quad\mathcal{A}_t\times_1\boldsymbol{P}_1^*\boldsymbol{P}_1\times_2\boldsymbol{P}_2^*\boldsymbol{P}_2-\left(\mathcal{X}+\frac{\mathcal{M}_{x,t}}{\rho_t}\right)\times_1\boldsymbol{P}_1^*\times_2\boldsymbol{P}_2^*\notag\\&\times_3\boldsymbol{S}^*=\left(\mathcal{A}_t-\mathcal{A}_{t+1}\right)\times_1\boldsymbol{P}_1^*\boldsymbol{P}_1\times_2\boldsymbol{P}_2^*\boldsymbol{P}_2\notag\\&\quad\quad-\frac{\mathcal{M}_{x,t+1}\times_1\boldsymbol{P}_1^*\times_2\boldsymbol{P}_2^*\times_3\boldsymbol{S}^*}{\rho_t}.
	 	\label{eq: aux1}
	 \end{align}
	 Similarly, by the updating rule of $\mathcal{M}_{y}$,
	 $$\mathcal{Y}+\frac{\mathcal{M}_{y,t}}{\rho_t}=\frac{\mathcal{M}_{y,t+1}}{\rho_t}+\mathcal{A}_{t+1}\times_3\boldsymbol{P}_3\boldsymbol{S}.$$
	 Then
	 \begin{align}
	 	&\quad\mathcal{A}_t\times_3\boldsymbol{S}^*\boldsymbol{P}_3^*\boldsymbol{P}_3\boldsymbol{S}-\left(\mathcal{Y}+\frac{\mathcal{M}_{y,t}}{\rho_t}\right)\times_3\boldsymbol{S}^*\boldsymbol{P}_3^*\notag\\&=\left(\mathcal{A}_t-\mathcal{A}_{t+1}\right)\times_3\boldsymbol{S}^*\boldsymbol{P}_3^*\boldsymbol{P}_3\boldsymbol{S}-\frac{\mathcal{M}_{y,t+1}\times_3\boldsymbol{S}^*\boldsymbol{P}_3^*}{\rho_t}.
	 	\label{eq: aux2}
	 \end{align}
	 Subsequently, we invoke the updating rule of $\mathcal{M}_n$ to obtain
	 $$\frac{\mathcal{M}_{n,t}}{\rho_t}=\frac{\mathcal{M}_{n,t+1}}{\rho_t}+\mathcal{A}_{t+1}\times_n\boldsymbol{D}_{I_n}-\mathcal{G}_{n,t+1}.$$
	 Then
	 \begin{align}
	 	&\quad\mathcal{A}_t\times_n\boldsymbol{D}_{I_n}^*\boldsymbol{D}_{I_n}-\left(\mathcal{G}_{n,t}+\frac{\mathcal{M}_{n,t}}{\rho_t}\right)\times_n\boldsymbol{D}_{I_n}^*\notag\\&=\left(\mathcal{A}_t-\mathcal{A}_{t+1}\right)\times_n\boldsymbol{D}_{I_n}^*\boldsymbol{D}_{I_n}\notag\\&\quad-\left(\mathcal{G}_{n,t}-\mathcal{G}_{n,t+1}+\frac{\mathcal{M}_{n,t+1}}{\rho_t}\right)\times_n\boldsymbol{D}_{I_n}^*.
	 	\label{eq: aux3}
	 \end{align}
	 We then use the updating rule of $\mathcal{M}_n$ again to yield
	 \begin{align*}
	 	\mathcal{G}_{n,t}&=\frac{\mathcal{M}_{n,t}-\mathcal{M}_{n,t-1}}{\rho_{t-1}}+\mathcal{A}_t\times_n\boldsymbol{D}_{I_n},\\
	 	\mathcal{G}_{n,t+1}&=\frac{\mathcal{M}_{n,t+1}-\mathcal{M}_{n,t}}{\rho_{t}}+\mathcal{A}_{t+1}\times_n\boldsymbol{D}_{I_n},
	 \end{align*}
	 from which it follows
	 \begin{equation}
	 	\begin{aligned}
	 		&\quad\mathcal{G}_{n,t}-\mathcal{G}_{n,t+1}=\frac{(\nu+1)\mathcal{M}_{n,t}+\mathcal{M}_{n,t+1}-\nu\mathcal{M}_{n,t-1}}{\rho_t}\\&+\left(\mathcal{A}_{t}-\mathcal{A}_{t+1}\right)\times_n\boldsymbol{D}_{I_n}.
	 	\end{aligned}
	 	\label{eq: aux4}
	 \end{equation}
	 Taking Eq. \eqref{eq: aux4} into Eq. \eqref{eq: aux3}, it is derived that:
	 \begin{align}
	 	&\quad\mathcal{A}_t\times_n\boldsymbol{D}_{I_n}^*\boldsymbol{D}_{I_n}-\left(\mathcal{G}_{n,t}+\frac{\mathcal{M}_{n,t}}{\rho_t}\right)\times_n\boldsymbol{D}_{I_n}^*\notag\\&=-\frac{(\nu+1)\mathcal{M}_{n,t}+2\mathcal{M}_{n,t+1}-\nu\mathcal{M}_{n,t-1}}{\rho_t}\times_n\boldsymbol{D}_{I_n}^*.
	 	\label{eq: aux5}
	 \end{align}
	 Absorbing Eq. \eqref{eq: aux1}, Eq. \eqref{eq: aux2} and Eq. \eqref{eq: aux5} into Eq. \eqref{eq: aux}, we have
	 \begin{align}
	 	&\quad\mathtt{H}\left(\mathcal{A}_{t+1}-\mathcal{A}_t\right)\notag\\&=2\Bigg(\frac{\mathcal{M}_{x,t+1}\times_1\boldsymbol{P}_1^*\times_2\boldsymbol{P}_2^*\times_3\boldsymbol{S}^*}{\rho_t}+\frac{\mathcal{M}_{y,t+1}\times_3\boldsymbol{S}^*\boldsymbol{P}_3^*}{\rho_t}\notag\\&+\sum_{n=1}^{2}\frac{(\nu+1)\mathcal{M}_{n,t}+2\mathcal{M}_{n,t+1}-\nu\mathcal{M}_{n,t-1}}{\rho_t}\times_n\boldsymbol{D}_{I_n}^*\Bigg)
	 \end{align}
	 where
	 \begin{align*}
	 	\mathtt{H}(\mathcal{A})=\tau\mathcal{A}-2\left(\mathcal{A}\times_1\boldsymbol{P}_1^*\boldsymbol{P}_1\times_2\boldsymbol{P}_2^*\boldsymbol{P}_2+\mathcal{A}\times_3\boldsymbol{S}^*\boldsymbol{P}_3^*\boldsymbol{P}_3\boldsymbol{S}\right)
	 \end{align*}
	 for any $\mathcal{A}\in\mathbb{R}^{I_1\times I_2\times R}$. According to Eq. \eqref{eq: Lipschitz Variable}, it can be proven that $\mathtt{H}(\cdot)$ is an invertible linear mapping. Hence,
	 \begin{align}
	 	&\quad\mathcal{A}_{t+1}-\mathcal{A}_t\notag\\&=2\mathtt{H}^{-1}\Bigg(\frac{\mathcal{M}_{x,t+1}\times_1\boldsymbol{P}_1^*\times_2\boldsymbol{P}_2^*\times_3\boldsymbol{S}^*}{\rho_t}+\frac{\mathcal{M}_{y,t+1}\times_3\boldsymbol{S}^*\boldsymbol{P}_3^*}{\rho_t}\notag\\&+\sum_{n=1}^{2}\frac{(\nu+1)\mathcal{M}_{n,t}+2\mathcal{M}_{n,t+1}-\nu\mathcal{M}_{n,t-1}}{\rho_t}\times_n\boldsymbol{D}_{I_n}^*\Bigg).\notag
	 \end{align}
	 Since $\left\{\mathcal{M}_{x,t},\mathcal{M}_{y,t},\left\{\mathcal{M}_{n,t}\right\}_{n=1,2}\right\}_{t\in\mathbb{N}}$ is bounded and $\rho_t$ is exponential, it is concluded that $\left\{\mathcal{A}_{t}\right\}_{t\in\mathbb{N}}$ is Cauchy sequence. Besides, from the updating rule of $\mathcal{M}_n$, we have
	 \begin{align*}
	 	&\lim_{t\rightarrow\infty}\mathcal{G}_{n,t+1}\\=&\lim_{t\rightarrow\infty}\frac{\mathcal{M}_{n,t+1}-\mathcal{M}_{n,t}}{\rho_{t}}+\mathcal{A}_{t+1}\times_n\boldsymbol{D}_{I_n}\\=&\lim_{t\rightarrow\infty}\mathcal{A}_{t+1}\times_n\boldsymbol{D}_{I_n}.
	 \end{align*}
	 Hence $\left\{\left\{\mathcal{G}_{n,t}\right\}_{n\in\left\{1,2\right\}}\right\}_{t\in\mathbb{N}}$ is also Cauchy sequence.
	 We now have (1) and (2) proven.
	 
	 Proof of (3): Recall that $\psi(x)=\frac{\log(\gamma x+1)}{\log(\gamma+1)}$, once again invoking Lemma A1 in the Supplementary Material of \cite{9340243}, it is easy to conclude from Eq. \eqref{eq: G-opt} that $\mathcal{M}_{n,*}\in-\frac{1}{2}\partial\left\lVert\mathcal{G}_{n,\ast}\right\lVert_{\overset{3-n}{\ast},\psi}.$ 
	 
	 By Eq. \eqref{eq: A-update}, $\mathcal{A}_{t+1}-\mathcal{A}_t=\frac{\bigtriangledown L_1\left(\mathcal{A}_t\right)}{\tau}$. Taking $t\rightarrow\infty$, we have $\bigtriangledown L_1\left(\mathcal{A}_\ast\right)=0$. 
	 
	 By Eq. \eqref{eq: Multipliers Update},
	 \begin{equation}
	 	\begin{aligned}
	 		\frac{\mathcal{M}_{x,t+1}-\mathcal{M}_{x,t}}{\rho_t}&=\mathcal{X}-\mathcal{A}_t\times_1\boldsymbol{P}_1\times_2\boldsymbol{P}_2\times_3\boldsymbol{S},\\
	 		\frac{\mathcal{M}_{y,t+1}-\mathcal{M}_{y,t}}{\rho_t}&=\mathcal{Y}-\mathcal{A}_t\times_3\boldsymbol{P}_3\boldsymbol{S},\\
	 		\frac{\mathcal{M}_{n,t+1}-\mathcal{M}_{n,t}}{\rho_t}&=\mathcal{G}_n-\mathcal{A}_t\times_n\boldsymbol{D}_{I_n},\,n=1,2.\\
	 	\end{aligned}
	 	\nonumber
	 \end{equation}
	 Taking $t\rightarrow\infty$, we have:
	 \begin{equation}
	 	\begin{aligned}
	 		\mathcal{X}&=\mathcal{A}_*\times_1\boldsymbol{P}_1\times_2\boldsymbol{P}_2\times_3\boldsymbol{S},\\
	 		\mathcal{Y}&=\mathcal{A}_*\times_3\boldsymbol{P}_3\boldsymbol{S},\\
	 		\mathcal{G}_{n,*}&=\mathcal{A}_*\times_n\boldsymbol{D}_{I_n}, n=1,2.
	 	\end{aligned}
	 	\nonumber
	 \end{equation}
\end{proof}

\bibliographystyle{IEEEtran}
\bibliography{main}

\end{document}